\documentclass[10pt,twocolumn,letterpaper]{article}

\usepackage{cvpr}
\usepackage{times}
\usepackage{epsfig}
\usepackage{graphicx}
\usepackage{amsmath}
\usepackage{amssymb}
\usepackage{amsthm}
\usepackage{calc}
\usepackage{comment}
\usepackage{xcolor}
\usepackage{subcaption}
\usepackage{pifont}
\usepackage{colortbl}
\usepackage{multirow}
\usepackage{setspace}
\newtheorem{theorem}{Theorem}[section]

\newtheorem{lemma}[theorem]{Lemma}
\newtheorem{definition}[theorem]{Definition}

\newcommand\trainweaksup{\texttt{train-weaksup}\xspace}
\newcommand\trainfullsup{\texttt{train-fullsup}\xspace}
\newcommand\testfullsup{\texttt{test}\xspace}

\newcommand\boxacc{\texttt{BoxAcc}\xspace}
\newcommand\maxboxacc{\texttt{MaxBoxAcc}\xspace}
\newcommand\newmaxboxacc{\texttt{MaxBoxAccV2}\xspace}
\newcommand\pxprec{\texttt{PxPrec}\xspace}
\newcommand\pxrec{\texttt{PxRec}\xspace}
\newcommand\pxap{\texttt{PxAP}\xspace}

\newcommand{\myparagraph}[1]{\vspace{2pt}\noindent{\bf #1}}

\newcommand\pxacc{\texttt{PxAcc}\xspace}

\usepackage[pagebackref=true,breaklinks=true,letterpaper=true,colorlinks,bookmarks=false]{hyperref}

\cvprfinalcopy %

\def\cvprPaperID{1741} %

\ifcvprfinal\fi
\begin{document}

\title{Evaluating Weakly Supervised Object Localization Methods Right}
\author{\vspace{0em}
\setlength\tabcolsep{2em}
\begin{tabular}{ccc} 
Junsuk Choe$^{1,3}$\thanks{Equal contribution. Work done at Clova AI Research.} & Seong Joon Oh$^2$\footnotemark[1] & Seungho Lee$^1$ \tabularnewline
\vspace{-0.8em} & & \tabularnewline
Sanghyuk Chun$^3$ & Zeynep Akata$^4$ & Hyunjung Shim$^1$\thanks{Hyunjung Shim is a corresponding author.} \tabularnewline
\end{tabular}
\\
\\
\renewcommand{\arraystretch}{0.9}
\begin{tabular}{cccc} 
    $^1$\normalsize{School of Integrated Technology,} & $^2$\normalsize{Clova AI Research,} & $^3$\normalsize{Clova AI Research,} & $^4$\normalsize{University of Tuebingen} \tabularnewline
    \normalsize{Yonsei University} & \normalsize{\enskip LINE Plus Corp.} & \normalsize{\enskip NAVER Corp.} & \normalsize{}
\end{tabular}
}%
\maketitle

\begin{abstract}
Weakly-supervised object localization (WSOL) has gained popularity over the last years for its promise to train localization models with only image-level labels. Since the seminal WSOL work of class activation mapping (CAM), the field has focused on how to expand the attention regions to cover objects more broadly and localize them better. However, these strategies rely on full localization supervision to validate hyperparameters and for model selection, which is in principle prohibited under the WSOL setup. In this paper, we argue that WSOL task is ill-posed with only image-level labels, and propose a new evaluation protocol where full supervision is limited to only a small held-out set not overlapping with the test set. We observe that, under our protocol, the five most recent WSOL methods have not made a major improvement over the CAM baseline. Moreover, we report that existing WSOL methods have not reached the few-shot learning baseline, where the full-supervision at validation time is used for model training instead. Based on our findings, we discuss some future directions for WSOL.
Source code and dataset are available at \href{https://github.com/clovaai/wsolevaluation}{https://github.com/clovaai/wsolevaluation}.
\end{abstract}

\begin{figure}
    \centering
    \includegraphics[width=\linewidth]{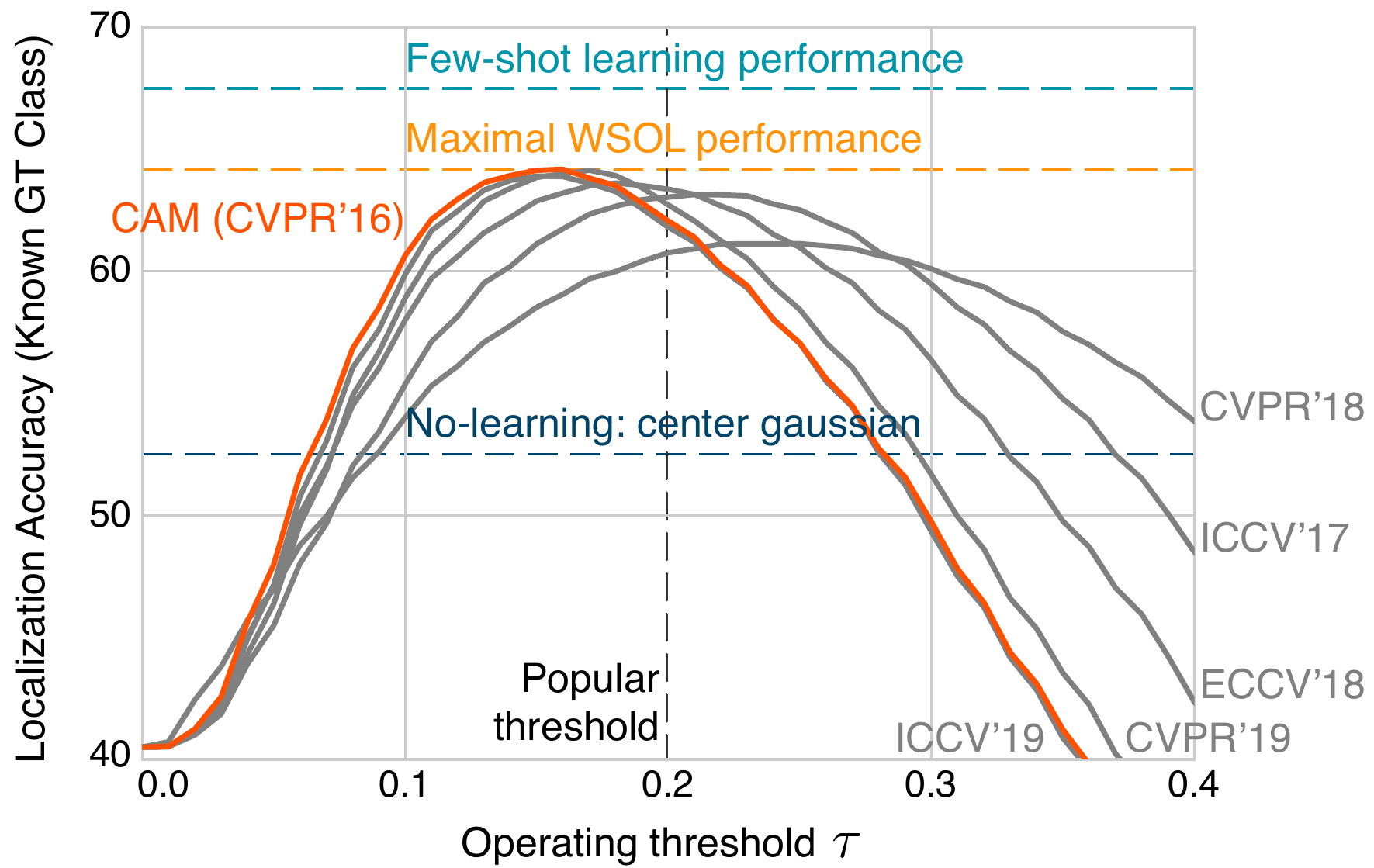}\vspace{0.3em}
    \caption{\small \textbf{WSOL 2016-2019.} Recent improvements in WSOL are illusory due to (1) different amount of implicit full supervision through validation and (2) a fixed score-map threshold (usually $\tau=0.2$) to generate object boxes. Under our evaluation protocol with the same validation set sizes and oracle $\tau$ for each method, CAM is still the best.
    In fact, our few-shot learning baseline, \ie using the validation supervision (10 samples/class) at training time, outperforms existing WSOL methods. Results on ImageNet.}
    \label{fig:teaser}
    \vspace{-1.5em}
\end{figure}

\section{Introduction}

As human labeling for every object is too costly and weakly-supervised object localization (WSOL) requires \textit{only} image-level labels, the WSOL research has gained significant momentum~\cite{CAM,ACoL,SPG,ADL,HaS,CutMix} recently.

Among these, class activation mapping (CAM)~\cite{CAM} uses the intermediate classifier activations focusing on the most discriminative parts of the objects to localize the objects of the target class. As the aim in object localization is to cover the full extent of the object, focusing only on the most discriminative parts of the objects is a limitation.
WSOL techniques since CAM have focused on this limitation and have proposed different architectural~\cite{ACoL,SPG,ADL} and data-augmentation~\cite{HaS,CutMix} solutions.
The reported state-of-the-art WSOL performances have made a significant improvement over the CAM baseline, from 49.4\% to 62.3\%~\cite{ADL} and 43.6\% to 48.7\%~\cite{ADL} top-1 localization performances on Caltech-UCSD Birds-200-2011~\cite{CUB} and ImageNet~\cite{ImageNet}, respectively.
However, these techniques have introduced a set of hyperparameters for suppressing the discriminative cues of CAM and different ways for selecting these hyperparameters.
One of such hyperparameters is the operating threshold $\tau$ for generating object bounding boxes from the score maps. Among others, the mixed policies for selecting $\tau$ has contributed to the illusory improvement of WSOL performances over the years; see Figure~\ref{fig:teaser}.

Due to the lack of a unified definition of the WSOL task, we revisit the problem formulation of WSOL and show that WSOL problem is ill-posed in general without any localization supervision. Towards a well-posed setup, we propose a new WSOL setting where a small held-out set with full supervision is available to the learners.

Our contributions are as follows. (1) Propose new experimental protocol that uses a fixed amount of full supervision for hyperparameter search and carefully analyze six WSOL methods on three architectures and three datasets. (2) Propose new evaluation metrics as well as data, annotations, and benchmarks for the WSOL task at \href{https://github.com/clovaai/wsolevaluation}{https://github.com/clovaai/wsolevaluation}. (3) Show that WSOL has not progressed significantly since CAM, when the calibration dependency and the different amounts of full supervision are factored out. Moreover, searching hyperparameters on a held-out set consisting of 5 to 10 full localization supervision per class often leads to significantly lower performance compared to the few-shot learning (FSL) baselines that use the full supervision directly for model training. Finally, we suggest a shift of focus in future WSOL research: consideration of learning paradigms utilizing both weak and full supervisions, and other options for resolving the ill-posedness of WSOL (\eg background-class images).

\section{Related Work}

\myparagraph{By model output.}
Given an input image, \textit{semantic segmentation} models generate pixel-wise class predictions~\cite{Pascal,FCN}, \textit{object detection} models~\cite{Pascal,RCNN} output a set of bounding boxes with class predictions, and \textit{instance segmentation} models~\cite{COCO,CityScapes,MaskRCNN} predict a set of disjoint masks with class \textit{and} instance labels. \textit{Object localization}~\cite{ImageNet}, on the other hand, assumes that the image contains an object of single class and produces a binary mask or a bounding box around that object coming from the class of interest. 

\myparagraph{By type of supervision.}
Since bounding box and mask labels cost significantly more than image-level labels, \eg categories~\cite{PointSup}, researchers have considered different types of localization supervision: image-level labels~\cite{papandreou2015weakly}, gaze~\cite{GazeSup}, points~\cite{PointSup}, scribbles~\cite{ScribbleSup}, boxes~\cite{BoxSup}, or a mixture of multiple types~\cite{HeterogeneousSup}. Our work is concerned with the object localization task with only image-level category labels~\cite{Oquab2015CVPR,CAM}. 

\myparagraph{By amount of supervision.}
Learning from a small amount of labeled samples per class is referred to as few-shot learning (FSL)~\cite{XCHSA19}. We recognize the relationship between our new WSOL setup and the FSL paradigm; we consider FSL methods as baselines for future WSOL methods.

\myparagraph{WSOL works.}
Class activation mapping (CAM)~\cite{CAM} turns a fully-convolutional classifier into a score map predictor by considering the activations before the global average pooling layer. Vanilla CAM has been criticized for its focus on the small discriminative part of the object.
Researchers have considered dropping regions in inputs at random~\cite{HaS, CutMix} to diversify the cues used for recognition.
Adversarial erasing techniques~\cite{ACoL,ADL} drop the most discriminative part at the current iteration.
Self-produced guidance (SPG)~\cite{SPG} is trained with auxiliary foreground-background masks generated by its own activations.
Other than object classification in static images, there exists work on localizing informative video frames for action recognition~\cite{paul2018w,liu2019completeness,xue2019danet}, but they are beyond the scope of our analysis.

\myparagraph{Relation to explainability.}
WSOL methods share similarities with the model explainability~\cite{ExplainableAI}, specifically the \textit{input attribution} task: analyzing which pixels have led to the image classification results~\cite{guidotti2019survey}. There are largely two streams of work on visual input attribution: variants of input gradients~\cite{FirstInputGradient,FirstDNNInputGradient,Wojciech16LRP,GuidedBackprop,selvaraju2017grad,IntegratedGradients,KRDCA18, park2018multimodal} and counterfactual reasoning~\cite{LIME,VedaldiMeaingfulPerturbation,zintgraf2017visualizing,ribeiro2018anchors,goyal2019counterfactual,hendricks2018grounding}. While they can be viewed as WSOL methods, we have not included them in our studies because they are seldom evaluated in WSOL benchmarks. Analyzing their possibility as WSOL methods is an interesting future study.

\myparagraph{Our scope.}
We study the WSOL task, rather than weakly-supervised detection, segmentation, or instance segmentation. The terminologies tend to be mixed in the earlier works of weakly-supervised learning~\cite{WSOLasMIL2,WSOLasMIL3,OldWSOL1,WSOLasMIL4}. Extending our analysis to other weakly-supervised learning tasks is valid and will be a good contribution to the respective communities.

\begin{figure}[t]
    \centering
    \includegraphics[width=\linewidth]{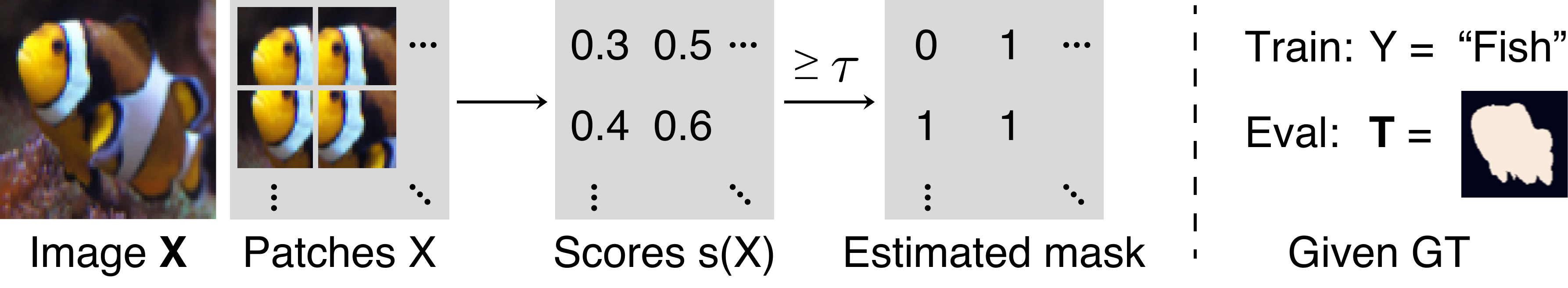}\vspace{0.3em}
    \caption{\small\textbf{WSOL as MIL.}
    WSOL is interpreted as a patch classification task trained with multiple-instance learning (MIL). The score map $s(\textbf{X})$ is thresholded at $\tau$ to estimate the mask $\mathbf{T}$.}
    \label{fig:overview}
    \vspace{-1.5em}
\end{figure}

\section{Problem Formulation of WSOL}
\label{sec:wsol_impossibility}

We define and formulate the weakly-supervised object localization (WSOL) task as an image patch classification and show the ill-posedness of the problem. We will discuss possible modifications to resolve the ill-posedness in theory.

\subsection{WSOL Task as Multiple Instance Learning}
\label{subsec:what_is_wsol}

Given an image $\mathbf{X}\in\mathbb{R}^{H\times W}$, \textbf{object localization} is the task to identify whether or not the pixel belongs to the object of interest, represented via dense binary mask $\mathbf{T}=(T_{11},\cdots,T_{HW})$ where $T_{ij}\in\{0,1\}$ and $(i,j)$ indicate the pixel indices. When the training set consists of precise image-mask pairs $(\mathbf{X},\mathbf{T})$, we refer to the task as \textbf{fully-supervised object localization (FSOL)}. In this paper, we consider the case when only an image-level label $Y\in\{0,1\}$ for global presence of the object of interest is provided per training image $\mathbf{X}$. This task is referred to as the \textbf{weakly-supervised object localization (WSOL)}.

One can treat an input image $\mathbf{X}$ as a bag of stride-1 sliding window patches of suitable side lengths, $h$ and $w$: $(X_{11},\cdots,X_{HW})$ with $X_{ij}\in\mathbb{R}^{h\times w}$. The object localization task is then the problem of predicting the object presence $T_{ij}$ at the image patch $X_{ij}$. The weak supervision imposes the requirement that each training image $\mathbf{X}$, represented as $(X_{11},\cdots,X_{HW})$, is only collectively labeled with a single label $Y\in\{0,1\}$ indicating whether at least one of the patches represents the object. This formulation is an example of the multiple-instance learning (MIL)~\cite{MIL}, as observed by many traditional WSOL works~\cite{papandreou2015weakly,WSOLasMIL2,WSOLasMIL3,WSOLasMIL4}.

Following the patch classification point of view, we formulate WSOL task as a mapping from patches $X$ to the binary labels $T$ (indices dropped). We assume that the patches $X$, image-level labels $Y$, and  the pixel-wise labeling $T$ in our data arise in an i.i.d. fashion from the joint distribution $p(X,Y,T)$. See Figure~\ref{fig:overview} for an overview. The aim of WSOL is to produce a scoring function $s(X)$ such that thresholding it at $\tau$ closely approximates binary label $T$.
Many existing approaches for WSOL, including CAM~\cite{CAM}, use the scoring rules based on the posterior $s(X)=p(Y|X)$. See Appendix~\S\ref{appendix:cam_as_posterior_approximation} for the interpretation of CAM as pixel-wise posterior approximation.

\subsection{When is WSOL ill-posed?}
\label{subsec:when_is_wsol_unsolvable}

We show that if background cues are more strongly associated with the target labels $T$ than some foreground cues, the localization task cannot be solved, even when we know the exact posterior $p(Y|X)$ for the image-level label $Y$. We will make some strong assumptions in favor of the learner, and then show that WSOL still cannot be perfectly solved. 

We assume that there exists a finite set of \textbf{cue} labels $\mathcal{M}$ containing all patch-level concepts in natural images. For example, patches from a duck image are one of \{duck's head, duck's feet, sky, water, $\cdots$\} (see Figure~\ref{fig:duck_feet_lake}). We further assume that every patch $X$ is equivalently represented by its cue label $M(X)\in\mathcal{M}$. Therefore, from now on, we write $M$ instead of $X$ in equations and examine the association arising in the joint distribution $p(M,Y,T)$. 
We write $M^{\text{fg}},M^{\text{bg}}\in\mathcal{M}$ for foreground and background cues.

\newcommand{\eqd}{\text{duck}}
\newcommand{\eqnd}{\text{duck}^c}
\newcommand{\eqf}{\text{feet}}
\newcommand{\eql}{\text{water}}

\begin{figure}[t]
    \centering
    \includegraphics[width=1\linewidth]{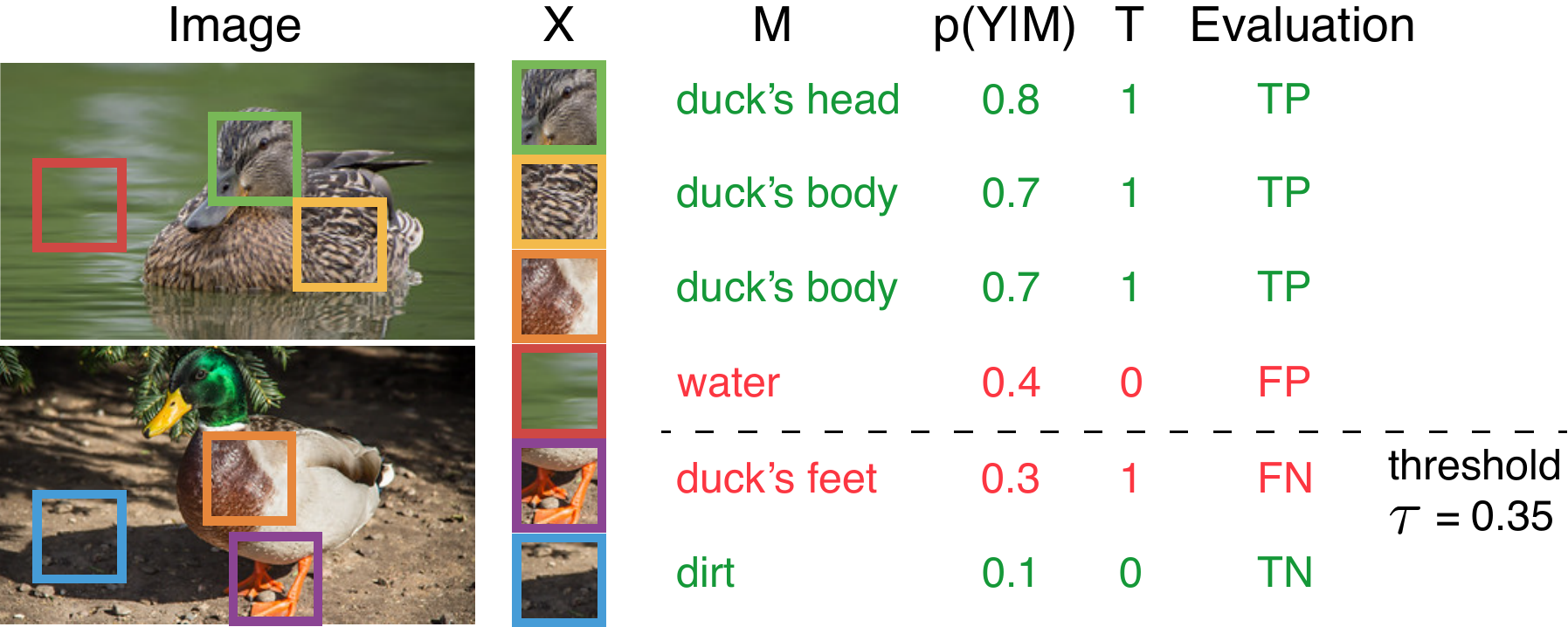}\vspace{0.3em}
    \caption{\small\textbf{Ill-posed WSOL: An example.} Even the true posterior $s(M)=p(Y|M)$ may not lead to the correct prediction of $T$ if background cues are more associated with the class than the foreground cues (\eg $p(\eqd|\eql)> p(\eqd|\eqf)$).}
    \label{fig:duck_feet_lake}
    \vspace{-1.5em}
\end{figure}

We argue that, even with access to the joint distribution $p(Y,M)$, it may not be possible to make perfect predictions for the patch-wise labels $T(M)$ (proof in Appendix~\S\ref{appendix:proof}).
\begin{lemma}
    Assume that the true posterior $p(Y|M)$ with a continuous pdf is used as the scoring rule $s(M)=p(Y|M)$. Then, there exists a scalar $\tau\in\mathbb{R}$ such that $s(M)\geq\tau$ is identical to $T$ if and only if the foreground-background posterior ratio $\frac{p(Y=1|M^{\text{fg}})}{p(Y=1| M^{\text{bf}})}\geq 1$ almost surely, conditionally on the event $\{T(M^{\text{fg}})=1\text{ and }T(M^{\text{bf}})=0\}$.
\end{lemma}%
In other words, if the posterior likelihood for the image-level label $Y$ given a foreground cue $M^{\text{fg}}$ is less than the posterior likelihood given background $M^{\text{bg}}$ for some foreground and background cues, no WSOL method can make a correct prediction. This pathological scenario is described in Figure~\ref{fig:duck_feet_lake}: Duck's feet are less seen in duck images than the water background. Such cases are abundant in user-collected data (Appendix Figure~\ref{fig:duck_feet_lake_appendix}).

This observation implies a data-centric solution towards well-posed WSOL: we can augment (1) positive samples ($Y=1$) with more less-represented foreground cues (\eg duck images with feet) and (2) negative samples ($Y=0$) with more target-correlated background cues (\eg non-duck images with water background). Such data-centric approaches are promising future directions for WSOL.

\myparagraph{How have WSOL methods addressed the ill-posedness?}
Previous solutions to the WSOL problem have sought architectural modifications~\cite{ACoL, SPG, ADL} and data augmentation~\cite{HaS,CutMix} schemes that typically require heavy hyperparameter search and model selection, which are a form of implicit localization supervision. For example, \cite{HaS} has found the operating threshold $\tau$ via ``observing a few qualitative results'', while others have evaluated their models over the test set to select reasonable hyperparameter values (Table 1 of \cite{HaS}, Table 6 of \cite{ACoL}, and Table 1 of \cite{ADL}). \cite{SPG} has performed a ``grid search'' over possible values. We argue that certain level of localization labels are inevitable for WSOL. In the next section, we propose to allow a fixed number of fully labeled samples for hyperparameter search and model selection for a more realistic evaluation.

\section{Evaluation Protocol for WSOL}
\label{sec:evaluation}

We reformulate the WSOL evaluation based on our observation of the ill-posedness. We define performance metrics, benchmarks, and the hyperparameter search procedure.

\subsection{Evaluation metrics}
\label{subsec:evaluation_metrics}

The aim of WSOL is to produce score maps, where their pixel value $s_{ij}$ is higher on foreground $T_{ij}=1$ and lower on background $T_{ij}=0$ (\S\ref{subsec:what_is_wsol}). We discuss how to quantify the above conditions and how prior evaluation metrics have failed to clearly measure the relevant performance. We then propose the \maxboxacc and \pxap metrics for bounding box and mask ground truths, respectively.

The \textit{localization accuracy}~\cite{ImageNet} metric entangles classification and localization performances by counting the number of images where both tasks are performed correctly. We advocate the measurement of localization performance alone, as the goal of WSOL is to localize objects (\S\ref{subsec:what_is_wsol}) and not to classify images correctly. To this end, we only consider the score maps $s_{ij}$ corresponding to the ground-truth classes in our analysis. Metrics based on such are commonly referred to as the \textit{GT-known} metrics~\cite{HaS,ACoL,SPG,ADL}.

A common practice in WSOL is to normalize the score maps per image because the score statistics differ vastly across images. Either max normalization (divide through by $\max_{ij}s_{ij}$) or min-max normalization (additionally map $\min_{ij}s_{ij}$ to zero) has been used; see Appendix~\S\ref{appendix:score_map_normalization} for the full summary. We always use the min-max normalization.

After normalization, WSOL methods threshold the score map at $\tau$ to generate a tight box around the binary mask $\{(i,j)\mid s_{ij}\geq \tau\}$. WSOL metrics then measure the quality of the boxes. $\tau$ is typically treated as a fixed value~\cite{CAM,ACoL,CutMix} or a hyperparameter to be tuned~\cite{HaS,SPG,ADL}.
We argue that the former is misleading because the ideal threshold $\tau$ depends heavily on the data and model architecture and fixing its value may be disadvantageous for certain methods. To fix the issue, we propose new evaluation metrics that are independent of the threshold $\tau$.

\myparagraph{Masks: \pxap.}
When masks are available for evaluation, we measure the pixel-wise precision and recall~\cite{PixelPrecisionRecall}. Unlike single-number measures like mask-wise IoU, those metrics allow users to choose the preferred operating threshold $\tau$ that provides the best precision-recall trade-off for their downstream applications. We define the \textbf{pixel precision and recall at threshold} $\tau$ as:
\vspace{-.1em}
{\small
\begin{align}
    \text{\pxprec}(\tau)=
    \frac{|\{s^{(n)}_{ij}\geq\tau\}\cap\{T^{(n)}_{ij}=1\}|}%
    {|\{s^{(n)}_{ij}\geq\tau\}|}\\
    \text{\pxrec}(\tau)=
    \frac{|\{s^{(n)}_{ij}\geq\tau\}\cap\{T^{(n)}_{ij}=1\}|}%
    {|\{T^{(n)}_{ij}=1\}|}
\end{align}
\vspace{-.1em}
}%
For threshold independence, we define and use the \textbf{pixel average precision}, $\text{\pxap}:=\sum_l \text{\pxprec}(\tau_l)(\text{\pxrec}(\tau_l)-\text{\pxrec}(\tau_{l-1}))$, the area under curve of the pixel precision-recall curve. We use \pxap as the final metric in this paper.

\myparagraph{Bounding boxes: \maxboxacc.}
Pixel-wise masks are expensive to collect; many datasets only provide box annotations. Since it is not possible to measure exact pixel-wise precision and recall with bounding boxes, we suggest a surrogate in this case. Given the ground truth box $B$, we define the \textbf{box accuracy at score map threshold} $\tau$ \textbf{and IoU threshold} $\delta$, \boxacc$(\tau, \delta)$~\cite{CAM,ImageNet}, as:
{\small
\vspace{-0.5em}
\begin{align}
    \text{\boxacc}(\tau, \delta)=
    \frac{1}{N}
    \sum_n
    1_{
    \text{IoU}\left(
    \text{box}(s(\mathbf{X}^{(n)}),\tau),B^{(n)}
    \right)\geq \delta
    }
\end{align}
}%
where $\text{box}(s(\mathbf{X}^{(n)}),\tau)$ is the tightest box around the largest-area connected component of the mask $\{(i,j)\mid s(X^{(n)}_{ij})\geq\tau\}$. In datasets where more than one bounding box are provided (\eg ImageNet), we count the number of images where the box prediction overlaps with \textit{at least one} of the ground truth boxes with $\text{IoU}\geq \delta$. When $\delta$ is $0.5$, the metric is identical to the commonly-called \textit{GT-known localization accuracy}~\cite{HaS} or \textit{CorLoc}~\cite{CorLoc}, but we suggest a new naming to more precisely represent what is being measured.
For score map threshold independence, we report the box accuracy at the optimal threshold $\tau$, the \textbf{maximal box accuracy} $\text{\maxboxacc}(\delta):=\max_\tau\text{\boxacc}(\tau, \delta)$, as the final performance metric. In this paper, we set $\delta$ to $0.5$, following the prior works~\cite{CAM, HaS, ACoL, SPG, ADL, CutMix}. 

\myparagraph{Better box evaluation: \newmaxboxacc.} After the acceptance at CVPR 2020, we have developed an improved version of \maxboxacc. It is better in two aspects. 
(1) \maxboxacc measures the performance at a fixed IoU threshold ($\delta=0.5$), only considering a specific level of fineness of localization outputs. We suggest averaging the performance across $\delta\in\{0.3, 0.5, 0.7\}$ to address diverse demands for localization fineness.
(2) \maxboxacc takes the \textit{largest} connected component for estimating the box, assuming that the object of interest is usually large. We remove this assumption by considering the best match between the set of all estimated boxes and the set of all ground truth boxes. 
We call this new metric as \newmaxboxacc. 
For future WSOL researches, we encourage using the \newmaxboxacc metric. The code is already available in our repository.
We show the evaluation results under the new metric in the in Appendix~\S\ref{appendix:newmaxboxacc}.

{
\setlength{\tabcolsep}{5pt}
\renewcommand{\arraystretch}{0.9}
\begin{table}[t]
    \small
    \centering
    \begin{tabular}{*{2}{l}*{3}{r}}
         Statistics && ImageNet & \hspace{1em} CUB & OpenImages  \\
         \cline{1-1} \cline{3-5}
         \vspace{-1em} & \\
         \#\ignorespaces Classes && $1000$ & $200$ & $100$ \\
         \vspace{-1em} & \\
         \#\ignorespaces images/class && \\
         \trainweaksup && $\sim\!1.2$K & $\sim\!30$ & $\sim\!300$ \\
         \trainfullsup && $10$ & $\sim\!5$ & $25$ \\
         \testfullsup && $10$ & $\sim\!29$ & $50$ \\
    \end{tabular}
    \caption{\small \textbf{Dataset statistics.} ``$\sim$'' indicates that the number of images per class varies across classes and the average value is shown.}
    \label{tab:dataset}
    \vspace{-1.5em}
\end{table}
}

\subsection{Data splits and hyperparameter search}
\label{subsec:evaluation_benchmarks}

For a fair comparison of the WSOL methods, we fix the amount of full supervision for hyperparameter search. As shown in Table~\ref{tab:dataset} we propose three disjoint splits for every dataset: \trainweaksup, \trainfullsup, and \testfullsup. The \trainweaksup contains images with weak supervision (the image-level labels). The \trainfullsup contains images with full supervision (either bounding box or binary mask). It is left as freedom for the user to utilize it for hyperparameter search, model selection, ablative studies, or even model fitting. The \testfullsup split contains images with full supervision; it must be used only for the final performance report. For example, checking the \testfullsup results multiple times with different model configurations violates the protocol as the learner implicitly uses more full supervision than allowed.

As WSOL benchmark datasets, ImageNet~\cite{ImageNet} and Caltech-UCSD Birds-200-2011 (CUB)~\cite{CUB} have been extensively used. 
For ImageNet, the $1.2$M ``train'' and $10$K ``validation'' images for $1\,000$ classes are treated as our \trainweaksup and \testfullsup, respectively. For \trainfullsup, we use the ImageNetV2~\cite{ImageNetV2}. We have annotated bounding boxes on those images.
CUB has $5\,994$ ``train'' and $5\,794$ ``test'' images for 200 classes. We treat them as our \trainweaksup and \testfullsup, respectively. For \trainfullsup, we have collected $1\,000$ extra images ($\sim\!5$ images per class) from Flickr, on which we have annotated bounding boxes. For ImageNet and CUB we use the oracle box accuracy \boxacc.

We contribute a new WSOL benchmark based on the OpenImages instance segmentation subset~\cite{OpenImagesV5}. It provides a fresh WSOL benchmark to which the models have not yet overfitted. 
To balance the original OpenImages dataset, we have sub-sampled 100 classes and have randomly selected $29\,819$, $2\,500$, and $5\,000$ images from the original ``train'', ``validation'', and ``test'' splits as our \trainweaksup, \trainfullsup, and \testfullsup splits, respectively.
We use the pixel average precision \pxap. 
A summary of dataset statistics is in Table~\ref{tab:dataset}. Details on data collection and preparation are in Appendix~\S\ref{appendix:data}.

\myparagraph{Hyperparameter search.}
To make sure that the same amount of localization supervision is provided for each WSOL method, we refrain from employing any source of human prior outside the \trainfullsup split. 
If the optimal hyperparameter for an arbitrary dataset and architecture is not available by default, we subject it to the hyperparameter search algorithm.
For each hyperparameter, its \textit{feasible range}, as opposed to \textit{sensible range}, is used as the search space, to minimize the impact of human bias.

We employ the random search hyperparameter optimization~\cite{RandomSearch}; it is simple, effective, and parallelizable. For each WSOL method, we sample 30 hyperparameters to train models on \trainweaksup and validate on \trainfullsup. The best hyperparameter combination is then selected.
Since running 30 training sessions is costly for ImageNet ($1.2$M training images), we use 10\% of images in each class for fitting models during the search. We verify in Appendix~\S\ref{appendix:okay_to_use_proxy_imagenet} that the ranking of hyperparameters is preserved even if the training set is sub-sampled.

\definecolor{darkergreen}{RGB}{21, 152, 56}
\definecolor{red2}{RGB}{252, 54, 65}
\newcommand\tableminus[1]{\textcolor{red2}{#1}}
\newcommand\tableplus[1]{\textcolor{darkergreen}{#1}}

\definecolor{Gray}{gray}{0.85}
\newcolumntype{g}{>{\columncolor{Gray}}c}

{
\setlength{\tabcolsep}{3pt}
\renewcommand{\arraystretch}{1.1}
\begin{table*}[ht!]
\resizebox{\textwidth}{!}{%
\centering
\small
\begin{tabular}{lc*{3}{c}gc*{3}{c}gc*{3}{c}gcg}
& & \multicolumn{4}{c}{ImageNet (\maxboxacc)} & & \multicolumn{4}{c}{CUB (\maxboxacc)}  & & \multicolumn{4}{c}{OpenImages (\pxap)} && \multicolumn{1}{c}{Total}\\
Methods  &  & VGG & Inception & ResNet & Mean &  & VGG & Inception & ResNet & Mean &  & VGG & Inception & ResNet & Mean &  & Mean \\
\cline{1-1}\cline{3-6}\cline{8-11}\cline{13-16}\cline{18-18} & \vspace{-1em} \\
CAM~\cite{CAM} &  & 61.1 & 65.3 & 64.2 & 63.5 &  & 71.1 & 62.1 & 73.2 & 68.8 &  & 58.1 & 61.4 & 58.0 & 59.1 &  & 63.8\\
HaS~\cite{HaS} &  & \tableplus{+0.7} & \tableplus{+0.1} & \tableminus{-1.0} & \tableminus{-0.1} &  & \tableplus{+5.2} & \tableminus{-4.4} & \tableplus{+4.9} & \tableplus{+1.9} &  & \tableminus{-1.2} & \tableminus{-2.9} & \tableplus{+0.2} & \tableminus{-1.3} &  & \tableplus{+0.2}\\
ACoL~\cite{ACoL} &  & \tableminus{-0.8} & \tableminus{-0.7} & \tableminus{-2.5} & \tableminus{-1.4} &  & \tableplus{+1.2} & \tableminus{-2.5} & \tableminus{-0.5} & \tableminus{-0.6} &  & \tableminus{-3.4} & \tableplus{+1.6} & \tableminus{-0.2} & \tableminus{-0.7} &  & \tableminus{-0.9}\\
SPG~\cite{SPG} &  & \tableplus{+0.5} & \tableplus{+0.1} & \tableminus{-0.7} & \tableplus{+0.0} &  & \tableminus{-7.4} & \tableplus{+0.7} & \tableminus{-1.8} & \tableminus{-2.8} &  & \tableminus{-2.2} & \tableplus{+1.0} & \tableminus{-0.3} & \tableminus{-0.5} &  & \tableminus{-1.1}\\
ADL~\cite{ADL} &  & \tableminus{-0.3} & \tableminus{-3.8} & \tableplus{+0.0} & \tableminus{-1.4} &  & \tableplus{+4.6} & \tableplus{+1.3} & \tableplus{+0.3} & \tableplus{+2.0} &  & \tableplus{+0.2} & \tableplus{+0.7} & \tableminus{-3.7} & \tableminus{-0.9} &  & \tableminus{-0.1}\\
CutMix~\cite{CutMix} &  & \tableplus{+1.0} & \tableplus{+0.1} & \tableminus{-0.3} & \tableplus{+0.3} &  & \tableplus{+0.8} & \tableplus{+3.4} & \tableminus{-5.4} & \tableminus{-0.4} &  & \tableplus{+0.1} & \tableplus{+0.3} & \tableplus{+0.7} & \tableplus{+0.4} &  & \tableplus{+0.1}\\
\cline{1-1}\cline{3-6}\cline{8-11}\cline{13-16}\cline{18-18} & \vspace{-1em} \\
Best WSOL &  & 62.2 & 65.5 & 64.2 & 63.8 &  & 76.2 & 65.5 & 78.1 & 70.8 &  & 58.3 & 63.0 & 58.6 & 59.5 &  & 64.0\\
FSL baseline &  & 62.8 & 68.7 & 67.5 & 66.3 &  & 86.3 & 94.0 & 95.8 & 92.0 &  & 61.5 & 70.3 & 74.4 & 68.7 &  & 75.7\\
Center baseline &  & 52.5 & 52.5 & 52.5 & 52.5 &  & 59.7 & 59.7 & 59.7 & 59.7 &  & 45.8 & 45.8 & 45.8 & 45.8 &  & 52.3\\
\cline{1-1}\cline{3-6}\cline{8-11}\cline{13-16}\cline{18-18} & \vspace{-1em} \\
\end{tabular}
}
\caption{\small \textbf{Re-evaluating WSOL.} How much have WSOL methods improved upon the vanilla CAM model? \testfullsup split results are shown, relative to the vanilla CAM performance (\tableplus{increase} or \tableminus{decrease}). Hyperparameters have been optimized over the identical \trainfullsup split for all WSOL methods and the FSL baseline: (10,5,5) full supervision/class for (ImageNet,CUB,OpenImages). Reported results are in the Appendix Table {\color{red}5};
classification accuracies are in Appendix Table {\color{red}4}.
}
\label{tab:main}
\vspace{-1em}
\end{table*}
}

\section{Experiments}
\label{sec:results}

\subsection{Evaluated Methods}
\label{subsec:prior_wsol_methods}

We evaluate six widely used WSOL methods published in peer-reviewed venues. We describe each method in chronological order and discuss the set of hyperparameters. The full list of hyperparameters is in Appendix~\S\ref{appendix:wsol_methods}.

\noindent
\textbf{Class activation mapping (CAM)}~\cite{CAM} trains a classifier of fully-convolutional backbone with the global average pooling (GAP) structure. At test time, CAM uses the logit outputs before GAP as the score map $s_{ij}$. CAM has the learning rate and the score-map resolution as hyperparameters and all five methods below use CAM in the background.

\noindent
\textbf{Hide-and-seek (HaS)}~\cite{HaS} is a data augmentation technique that randomly selects grid patches to be dropped. The hyperparameters are the drop rate and grid size.

\noindent
\textbf{Adversarial complementary learning (ACoL)}~\cite{ACoL} proposes an architectural solution: a two-head architecture where one adversarially erases the high-scoring activations in the other. The erasing threshold is a hyperparameter.

\noindent
\textbf{Self-produced guidance (SPG)}~\cite{SPG} is another architectural solution where internal pseudo-pixel-wise supervision is synthesized on the fly. Three tertiary pixel-wise masks (foreground, unsure, background) are generated from three different layers using two thresholding hyperparameters for each mask and are used as auxiliary supervisions.

\noindent
\textbf{Attention-based dropout layer (ADL)}~\cite{ADL} has proposed a module that, like ACoL, adversarially produces drop masks at high-scoring regions, while not requiring an additional head. Drop rate and threshold are the hyperparameters.

\noindent
\textbf{CutMix}~\cite{CutMix} is a data augmentation technique, where patches in training images are cut and pasted to other images during training. The target labels are also mixed. The hyperparameters are the size prior $\alpha$ and the mix rate $r$.

\myparagraph{Few-shot learning (FSL) baseline.}
The full supervision in \trainfullsup used for validating WSOL hyperparameters can be used for training a model itself. Since only a few fully labeled samples per class are available, we refer to this setting as the few-shot learning (FSL) baseline.

As a simple baseline, we consider a foreground saliency mask predictor~\cite{FirstSaliency}. We alter the last layer of a fully convolutional network (FCN) into a $1\times 1$ convolutional layer with $H\times W$ score map output. Each pixel is trained with the binary cross-entropy loss against the target mask, as done in~\cite{Deeplab,FCN,oh2017exploiting}. For OpenImages, the pixel-wise masks are used as targets; for ImageNet and CUB, we build the mask targets by labeling pixels inside the ground truth boxes as foreground~\cite{BoxWSSS}. At inference phase, the $H\times W$ score maps are evaluated with the box or mask metrics.

\myparagraph{Center-gaussian baseline.}
The Center-gaussian baseline generates isotropic Gaussian score maps centered at the images. We set the standard deviation to 1, but note that it does not affect the \maxboxacc and \pxap measures. This provides a no-learning baseline for every localization method.

\subsection{Comparison of WSOL methods}
\label{subsec:main_comparison_wsol}
We evaluate the six WSOL methods over three backbone architectures, \ie VGG-GAP~\cite{VGG,CAM}, InceptionV3~\cite{InceptionV3}, and ResNet50~\cite{ResNet}, and three datasets, \ie CUB, ImageNet and OpenImages. For each (method, backbone, dataset) tuple, we have randomly searched the optimal hyperparameters over the \trainfullsup with 30 trials, totalling about $9\,000$ GPU hours. Since the sessions are parallelizable, it has taken only about 200 hours over 50 P40 GPUs to obtain the results. The results are shown in Table~\ref{tab:main}. We use the same batch sizes and training epochs to enforce the same computational budget. The checkpoints that achieves the best localization performance on \trainfullsup are used for evaluation.

Contrary to the improvements reported in prior work (Appendix Table~\ref{tab:supp-val-test-transfer}), recent WSOL methods have not led to major improvements compared to CAM, when validated in the same data splits and same evaluation metrics. On ImageNet, methods after CAM are generally struggling: only CutMix has seen a boost of +0.3pp on average. On CUB, ADL has attained a +2.0pp gain on average, but ADL fails to work well on other benchmarks. On the new WSOL benchmark, OpenImages, no method has improved over CAM, except for CutMix (+0.4pp on average). The best overall improvements over CAM (63.8\% total mean) is a mere +0.2pp boost by HaS. In general, we observe a random mixture of increases and decreases in performance over the baseline CAM, depending on the architecture and dataset. An important result in the table to be discussed later is the comparison against the few-shot learning baseline (\S\ref{subsec:few_shot_learning_results}).

Some reasons for the discrepancy between our results and the reported results include (1) the confounding of the actual score map improvement and the calibration scheme, (2) different types and amounts of full supervision employed under the hood, and (3) the use of different training settings (\eg batch size, learning rates, epochs). More details about the training settings are in Appendix~\S\ref{appendix:reproducing}.

\myparagraph{Which checkpoint is suitable for evaluation?} After the acceptance by CVPR 2020, we believe that it is inappropriate to use the best checkpoint for WSOL evaluation. This is because the best localization performances are achieved before convergence in many cases (Appendix~\S\ref{appendix:cls_performance}. At early epochs, the localization performance fluctuates a lot, so the peak performance is noise rather than the real performance. Hence, we recommend using the final checkpoint for future WSOL researchers. The evaluation results are shown in Appendix~Table~\ref{tab:main_cls}. 

\begin{figure}
    \centering
    \includegraphics[width=0.93\linewidth]{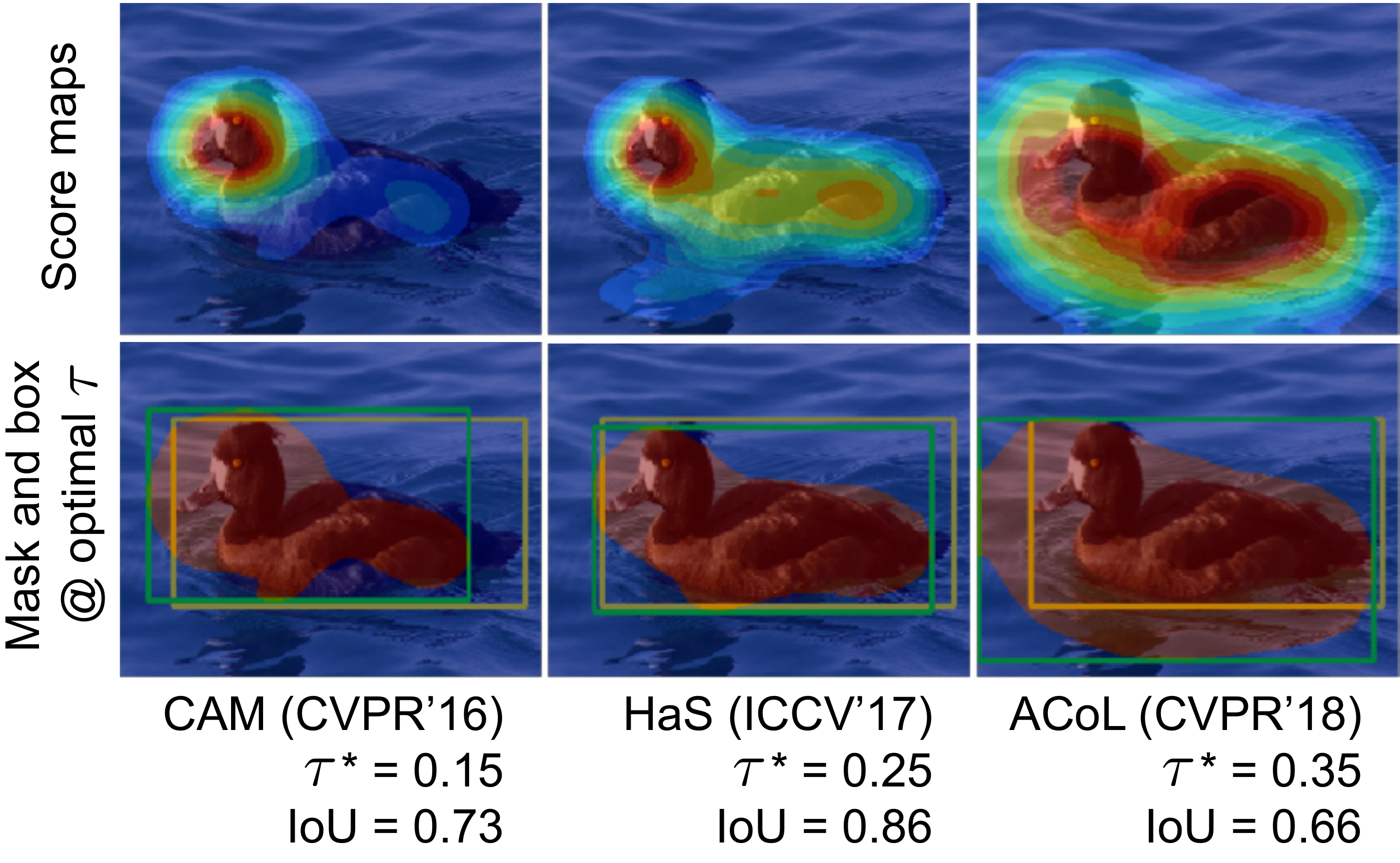}
    \caption{\small \textbf{Selecting $\tau$.} Measuring performance at a fixed threshold $\tau$ can lead to a false sense of improvement. Compared to CAM, HaS and ACoL expand the score maps, but they do not necessarily improve the box qualities (IoU) at the optimal $\tau^\star$. Predicted and ground-truth boxes are shown as green and yellow boxes.}
    \label{fig:qualitative_cam_threshold}
    \vspace{0.6em}
    \centering
    \includegraphics[width=\linewidth]{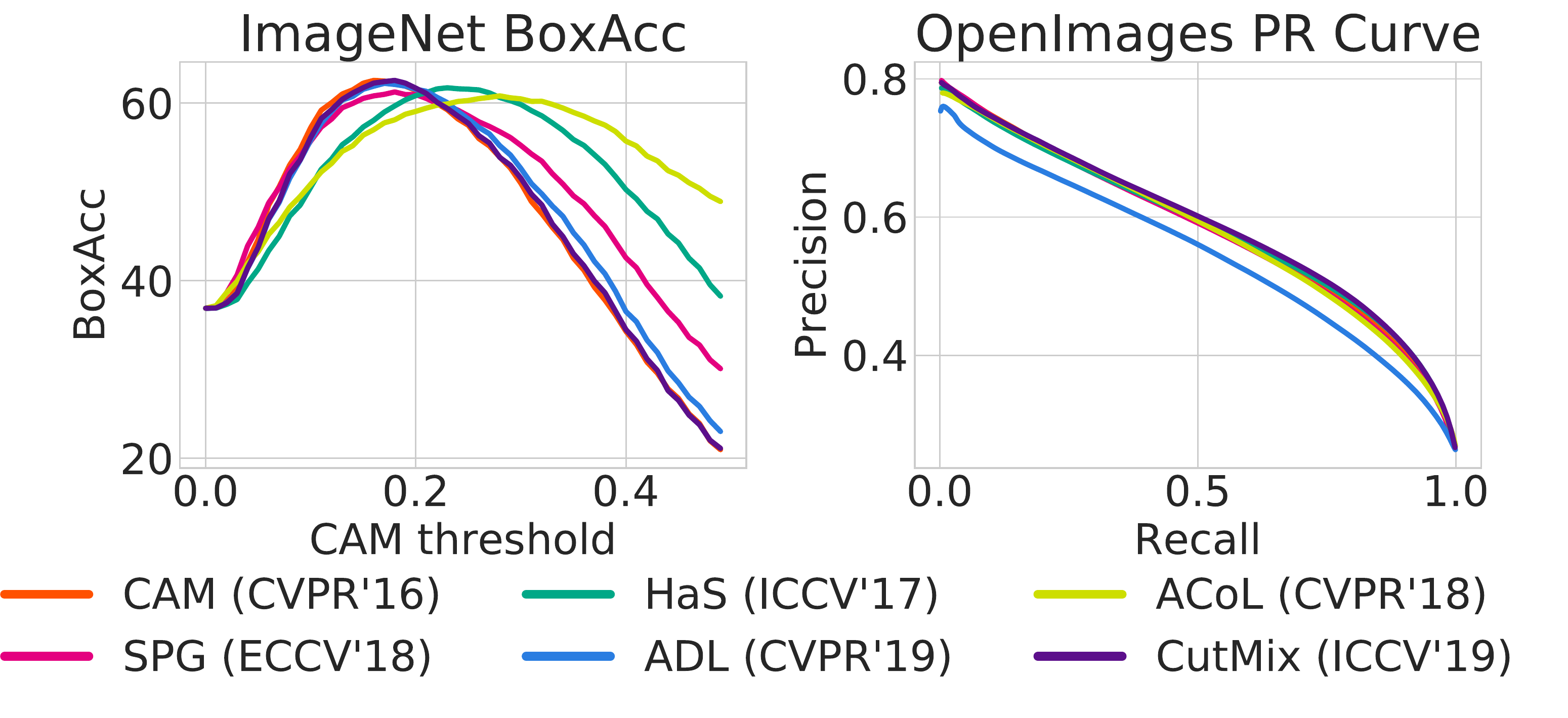}
    \caption{\small \textbf{Performance at varying operating thresholds.} ImageNet: $\text{\boxacc}(\tau)$ versus $\tau$. OpenImages: $\text{\pxprec}(\tau)$ versus $\text{\pxrec}(\tau)$. Both use ResNet.}
    \label{fig:cam_threshold_and_pr_curves}
    \vspace{-1em}
\end{figure}

\subsection{Score calibration and thresholding}
\label{subsec:score_calibration_thresholding}

\begin{figure}[t]
    \centering
    
    \begin{subfigure}[b]{\linewidth}
        \includegraphics[width=.99\linewidth]{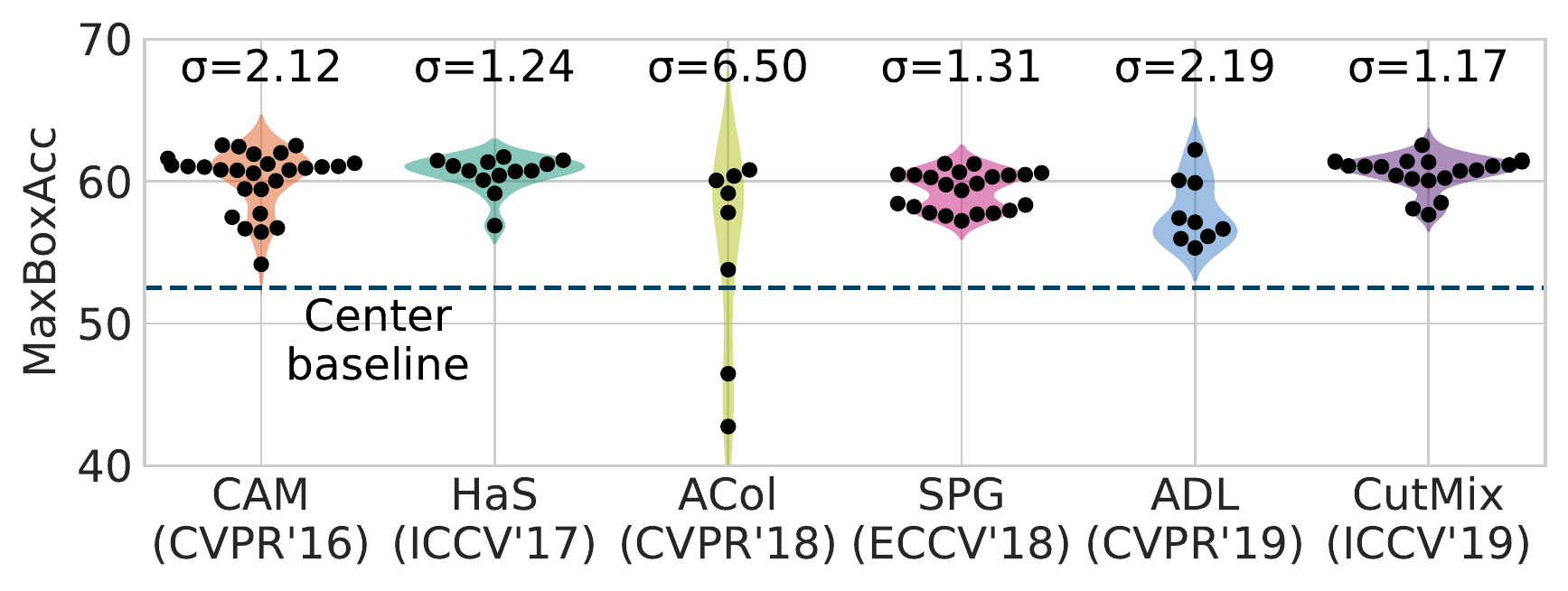}%
        \vspace{-0.5em}
        \caption{ResNet50 architecture, ImageNet dataset.}
        \vspace{0.5em}
    \end{subfigure}
    \begin{subfigure}[b]{\linewidth}
        \includegraphics[width=.99\linewidth]{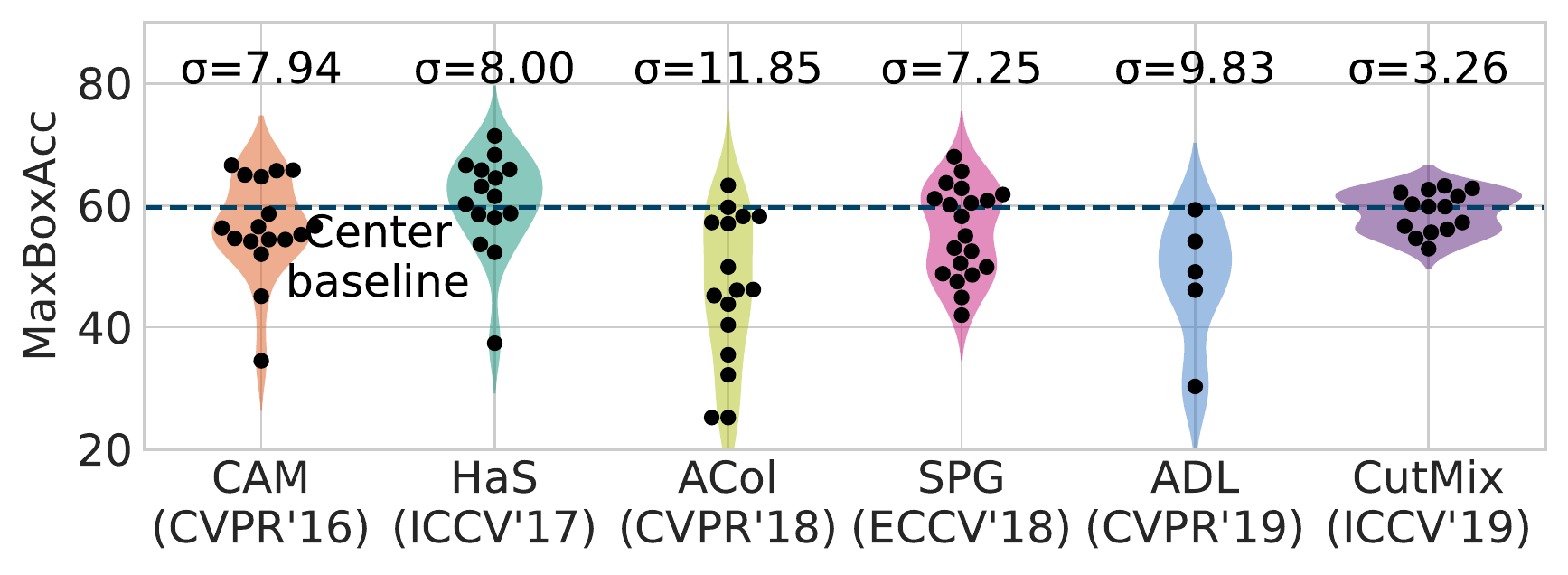}%
        \vspace{-0.5em}
        \caption{ResNet50 architecture, CUB dataset.}
    \vspace{0.5em}
    \end{subfigure}
    \begin{subfigure}[b]{\linewidth}
        \includegraphics[width=.99\linewidth]{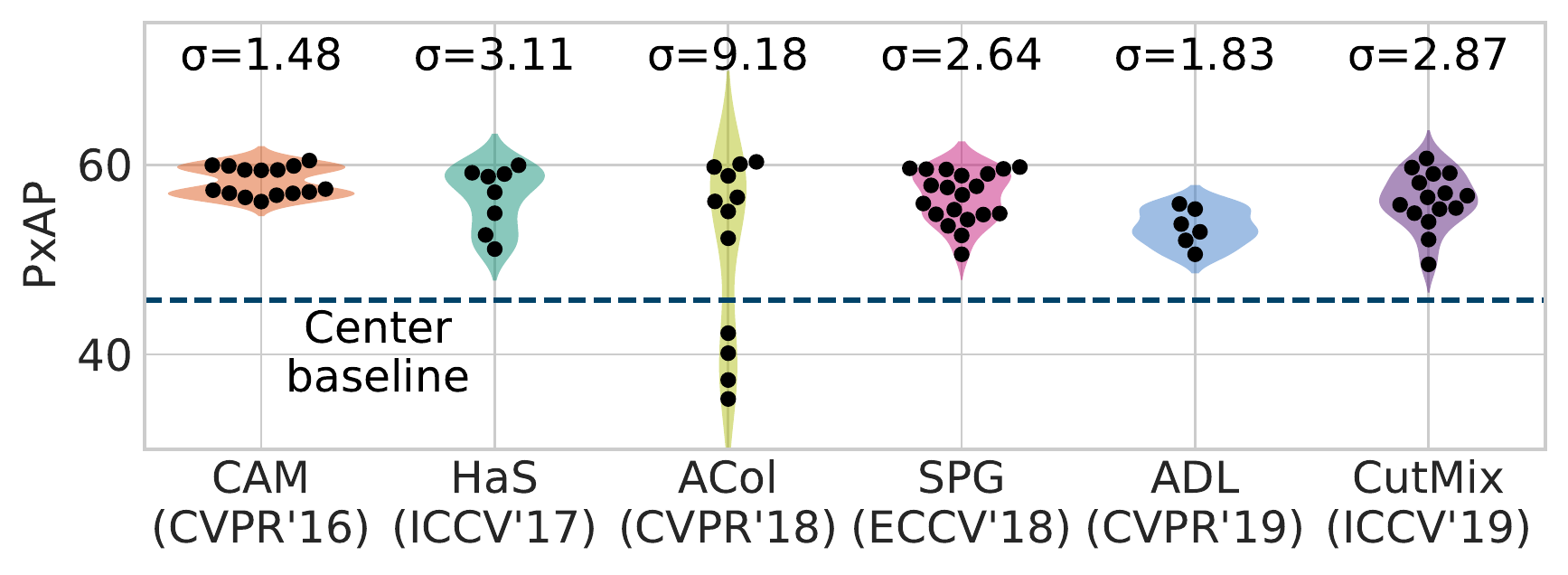}%
        \vspace{-0.5em}
        \caption{ResNet50 architecture, OpenImages dataset.}
    \vspace{0.5em}
    \end{subfigure}
    \begin{subfigure}[b]{\linewidth}
        \includegraphics[width=.99\linewidth]{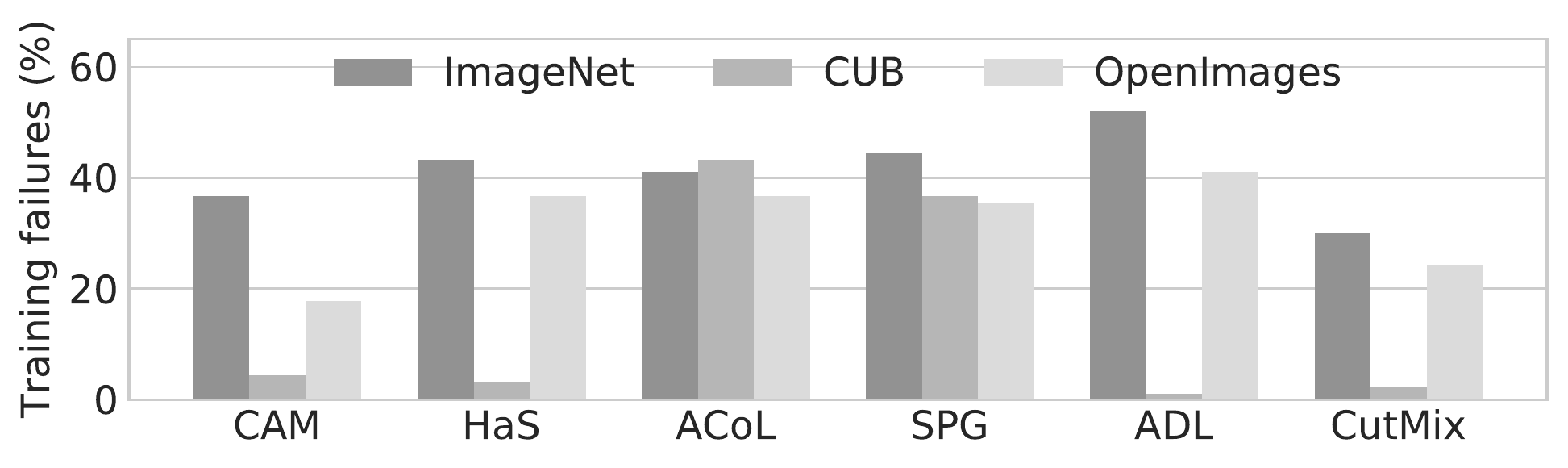}%
        \vspace{-0.5em}
        \caption{Ratio of training failures in different WSOL methods.}
        \label{fig:training_failures}
    \vspace{0.5em}
    \end{subfigure}
    \caption{\small \textbf{Results of the 30 hyperparameter trials.} ImageNet performances of all 30 randomly chosen hyperparameter combinations for each method, with ResNet50 backbone. The violin plots show the estimated distributions (kernel density estimation) of performances. $\sigma$ are the sample standard deviations.}
    \label{fig:wsol_robustness}
    \vspace{0.8em}
    \definecolor{greencross}{RGB}{60,140,50}
    \centering
    \includegraphics[width=0.36\linewidth]{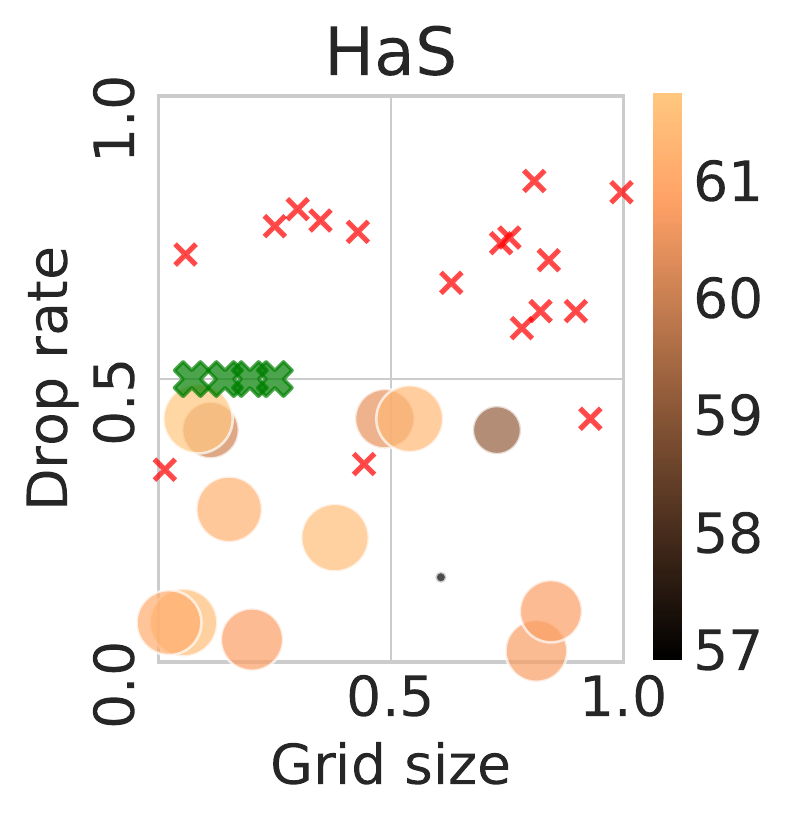}\hspace{0.03\linewidth}%
    \includegraphics[width=0.201\linewidth]{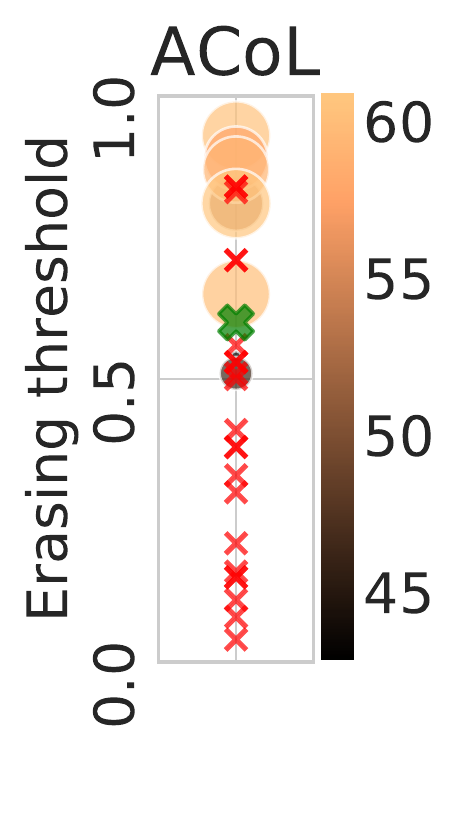}\hspace{0.03\linewidth}%
    \includegraphics[width=0.36\linewidth]{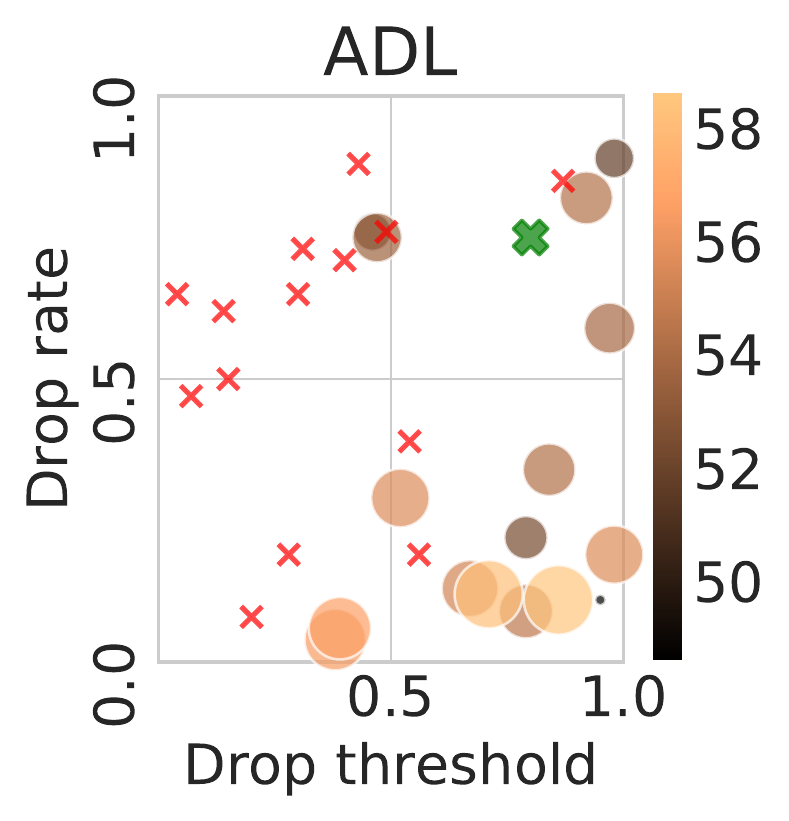}
    \caption{\small \textbf{Impact of hyperparameters for feature erasing.} Color and size of the circles indicate the performance at the corresponding hyperparameters. {\color{red}\ding{53}}: non-convergent training sessions. {\color{greencross}\ding{54}}: hyperparameters suggested by the original papers.
    }
    \label{fig:adl_acol_scatter}
    \vspace{-1.1em}
\end{figure}

WSOL evaluation must focus more on score map evaluation, independent of the calibration. As shown in Figure~\ref{fig:qualitative_cam_threshold} the min-max normalized score map for CAM predicts a peaky foreground score on the duck face, While HaS and ACoL score maps show more distributed scores in body areas, demonstrating the effects of adversarial erasing during training. However, the maximal IoU performances do not differ as much. This is because WSOL methods exhibit different score distributions (Figure~\ref{fig:cam_threshold_and_pr_curves} and Appendix~\S\ref{appendix:score_calibration}). Fixing the operating threshold $\tau$ at a pre-defined value, therefore, can lead to an apparent increase in performance without improving the score maps.

Under our threshold-independent performance measures (\maxboxacc and \pxap) shown in Figure~\ref{fig:cam_threshold_and_pr_curves}, we observe that (1) the methods have different optimal $\tau^\star$ on ImageNet and (2) the methods do not exhibit significantly different \maxboxacc or \pxap performances. This provides an explanation of the lack of improvement observed in Table~\ref{tab:main}.
We advise future WSOL researchers to report the threshold-independent metrics.

\subsection{Hyperparameter analysis}
\label{subsec:hyperparameter_analysis}

Different types and amounts of full supervision used in WSOL methods manifest in the form of model hyperparameter selection (\S\ref{sec:wsol_impossibility}). Here, we measure the impact of the validation on \trainfullsup by observing the performance distribution among 30 trials of random hyperparameters. We then study the effects of feature-erasing hyperparameters, a common hyperparameter type in WSOL methods.

\begin{figure*}
    \centering
    \begin{subfigure}[b]{.30\linewidth}
        \includegraphics[width=\linewidth]{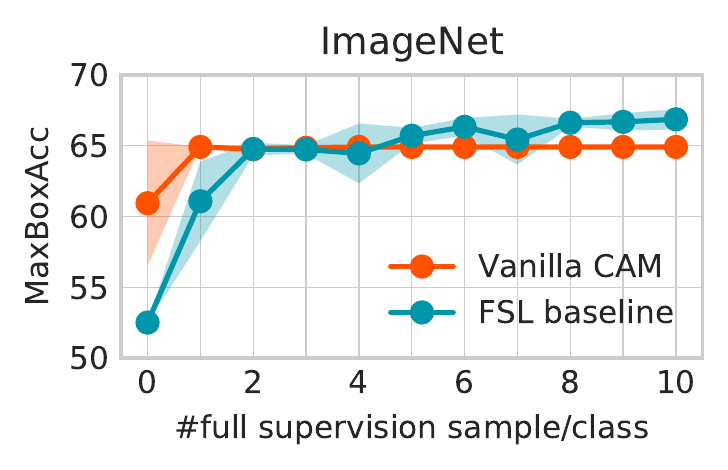}
    \end{subfigure}
    \begin{subfigure}[b]{.30\linewidth}
        \includegraphics[width=\linewidth]{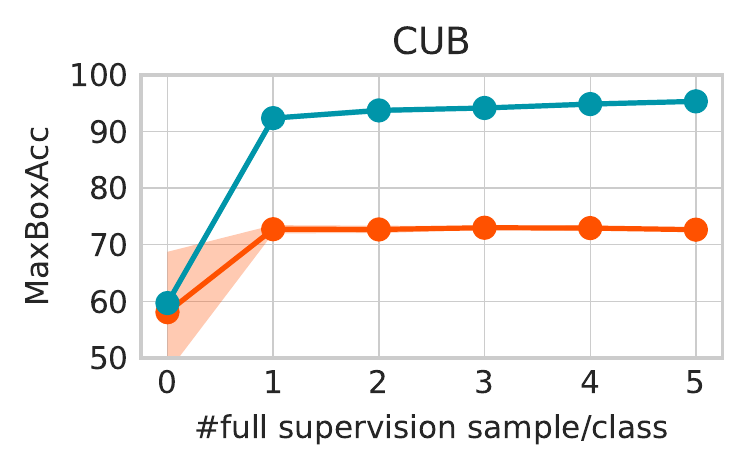}
    \end{subfigure}
    \begin{subfigure}[b]{.30\linewidth}
        \includegraphics[width=\linewidth]{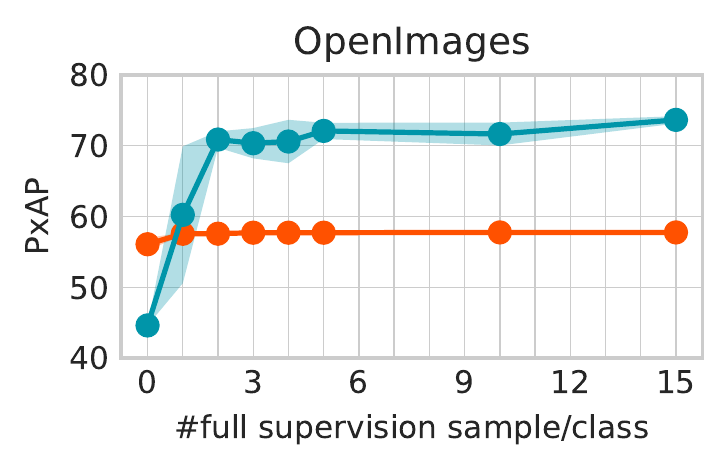}
    \end{subfigure}
    \caption{\small \textbf{WSOL versus few-shot learning.} The mean and standard error of models trained on three samples of full-supervision subsets are reported. ResNet50 is used throughout. At 0 full supervision, Vanilla CAM$=$random-hyperparameter and FSL$=$center-gaussian baseline.
    }
    \label{fig:wsol_vs_fsl}
    \vspace{-1.5em}
\end{figure*}

\myparagraph{Performance with 30 hyperparameter trials.}
To measure the sensitivity of each method to hyperparameter choices, we plot the performance distribution of the intermediate models in the 30 random search trials. We say that a training session is \textit{non-convergent} if the training loss is larger than 2.0 at the last epoch. We show the performance distributions of the converged sessions, and report the ratio of non-convergent sessions separately.

Our results in Figure~\ref{fig:wsol_robustness} indicate the diverse range of performances depending on the hyperparameter choice. Vanilla CAM is among the less sensitive, with the smallest standard deviation $\sigma=1.5$ on OpenImages. This is the natural consequence of its minimal use of hyperparameters. We thus suggest to use the vanilla CAM when absolutely no full supervision is available. ACoL and ADL tend to have greater variances across benchmarks ($\sigma=11.9$ and $9.8$ on CUB). We conjecture that the drop threshold for adversarial erasing is a sensitive hyperparameter. 

WSOL on CUB are generally struggling:
random hyperparameters often show worse performance than the center baseline (66\% cases).
We conjecture that CUB is a disadvantageous setup for WSOL: as all images contain birds, the models only attend on bird parts for making predictions. We believe adding more non-bird images can improve the overall performances~(\S\ref{subsec:when_is_wsol_unsolvable}). 

We show the non-convergence statistics in Figure~\ref{fig:training_failures}. Vanilla CAM exhibit a stable training: non-convergence rates are lowest on OpenImages and second lowest on ImageNet. ACoL and SPG suffer from many training failures, especially on CUB (43\% and 37\%, respectively).

In conclusion, vanilla CAM is stable and robust to hyperparameters.
Complicated design choices introduced by later methods only seem to lower the overall performances rather than providing new avenues for performance boost.

\myparagraph{Effects of erasing hyperparameters.}
Many WSOL methods since CAM have introduced different forms of erasing to encourage models to extract cues from broader regions (\S\ref{subsec:prior_wsol_methods}). We study the contribution of such hyperparameters in ADL, HaS, and ACoL in Figure~\ref{fig:adl_acol_scatter}. We observe that the performance improves with higher erasing thresholds (ADL drop threshold and ACoL erasing threshold). We also observe that lower drop rates leads to better performances (ADL and HaS). The erasing hyperparameters introduced since CAM only negatively impact the performance.

\subsection{Few-shot learning baselines}
\label{subsec:few_shot_learning_results}
Given that WSOL methods inevitably utilize some form of full localization supervision (\S\ref{sec:wsol_impossibility}), it is important to compare them against the few-shot learning (FSL) baselines that use it for model tuning itself.

Performances of the FSL baselines (\S\ref{subsec:evaluation_benchmarks}) are presented in Table~\ref{tab:main}. Our simple FSL method performs better than the vanilla CAM at 10, 5, and 5 fully labeled samples per class for ImageNet, CUB, and OpenImages, respectively. The mean FSL accuracy on CUB is 92.0\%, which is far better than that of the maximal WSOL performance of 70.8\%.

We compare FSL against CAM at different sizes of \trainfullsup in Figure~\ref{fig:wsol_vs_fsl}. 
We simulate the zero-fully-labeled WSOL performance with a set of randomly chosen hyperparameters (\S\ref{subsec:hyperparameter_analysis}); for FSL, we simulate the no-learning performance through the center-gaussian baseline.

FSL baselines surpass the CAM results already at 1-2 full supervision per class for CUB and OpenImages (92.4 and 70.9\% \maxboxacc and \pxap). We attribute the high FSL performance on CUB to the fact that all images are birds; with 1 sample/class, there are effectively 200 birds as training samples. For OpenImages, the high FSL performance is due to the rich supervision provided by pixel-wise masks. On ImageNet, FSL results are not as great: they surpass the CAM result at 8-10 samples per class. Overall, however, FSL performances are strikingly good, even at a low data regime. 
Thus, given a few fully labeled samples, it is perhaps better to train a model with it than to search hyperparameters. 
Only when there is absolutely no full supervision (0 fully labeled sample), CAM is meaningful (better than the no-learning center-gaussian baseline).

\section{Discussion and Conclusion}
\label{sec:conclusion}

After years of weakly-supervised object localization (WSOL) research, we look back on the common practice and make a critical appraisal. Based on a precise definition of the task, we have argued that WSOL is ill-posed and have discussed how previous methods have used different types of implicit full supervision (\eg tuning hyperparameters with pixel-level annotations) to bypass this issue (\S\ref{sec:wsol_impossibility}). We have then proposed an improved evaluation protocol that allows the hyperparameter search over a few labeled samples (\S\ref{sec:evaluation}). Our empirical studies lead to some striking conclusions: CAM is still not worse than the follow-up methods (\S\ref{subsec:main_comparison_wsol}) and it is perhaps better to use the full supervision directly for model fitting, rather than for hyperparameter search (\S\ref{subsec:few_shot_learning_results}).

We propose the following future research directions for the field. (1) Resolve the ill-posedness via \eg adding more background-class images (\S\ref{subsec:when_is_wsol_unsolvable}). (2) Define the new task, \textit{semi-weakly-supervised object localization}, where methods incorporating both weak and full supervision are studied. 

Our work has implications in other tasks where learners are not supposed to be given full supervision, but are supervised implicitly via model selection and hyperparameter fitting. Examples include weakly-supervised vision tasks (\eg detection and segmentation), zero-shot learning, and unsupervised tasks (\eg disentanglement~\cite{Disentangle}).

{
\myparagraph{Acknowledgements.}
We thank Dongyoon Han, Hyojin Park, Jaejun Yoo, Jung-Woo Ha, Junho Cho, Kyungjune Baek, Muhammad Ferjad Naeem, Rodrigo Benenson, Youngjoon Yoo, and Youngjung Uh for the feedback. NAVER Smart Machine Learning (NSML)~\cite{NSML} has been used. Graphic design by Kay Choi. 
The work is supported by the Basic Science Research Program through the National Research Foundation of Korea (NRF) funded by the MSIP (NRF-2019R1A2C2006123) and ICT R\&D program of MSIP/IITP [R7124-16-0004, Development of Intelligent Interaction Technology Based on Context Awareness and Human Intention Understanding]. This work was also funded by DFG-EXC-Nummer 2064/1-Projektnummer 390727645 and the ERC under the Horizon 2020 program (grant agreement No. 853489).
}%

\clearpage
\appendix
\addcontentsline{toc}{section}{Appendices}
\part*{Appendix}

We include additional materials in this document. Each section matches with the main paper sections: \S\ref{appendix:problem_formulation} to main paper \S\ref{sec:wsol_impossibility}, \S\ref{appendix:evaluation_protocol} to main paper \S\ref{sec:evaluation}, and \S\ref{appendix:experiments} to main paper \S\ref{sec:results}.

\section{Problem Formulation of WSOL}
\label{appendix:problem_formulation}

\subsection{CAM as patch-wise posterior approximation}
\label{appendix:cam_as_posterior_approximation}

In the main paper \S\ref{subsec:what_is_wsol}, we have described class activation mapping (CAM)~\cite{CAM} as a patch-wise posterior approximator trained with image-level labels. We describe in detail why this is so. 

\myparagraph{Equivalent re-formulation of CAM.}
Originally, CAM is a technique applied on a convolutional neural network classifier $h:\mathbb{R}^{3\times H \times W}\rightarrow\mathbb{R}^{C}$, where $C$ is the number of classes, of the following form:
\begin{equation}
    h_{c}(X)=\sum_{d}W_{cd}\left(\frac{1}{HW}\sum_{ij}g_{dij}(\mathbf{X})\right)
\end{equation}
where $c,d$ are the channel-dimension indices and $i,j$ are spatial-dimension indices. In other words, $h$ is a fully convolutional neural network, followed by a global average pooling (GAP) and a linear (fully-connected) layer into a $C$-dimensional vector. We may swap the GAP and linear layers without changing the representation:
\begin{align}
    h_{c}(X)
    &=\frac{1}{HW}\sum_{ij}\left(\sum_{d}W_{cd}g_{dij}(\mathbf{X})\right) \\
    &=:\frac{1}{HW}\sum_{ij}f_{cij}(\mathbf{X}) 
\end{align}
where $f$ is now a fully-convolutional network. Each pixel $(i,j)$ in the feature map, $\left(f_{1ij}(\mathbf{X}),\cdots,f_{Cij}(\mathbf{X})\right)$, corresponds to the classification result of the corresponding field of view in the input $\mathbf{X}$, written as $X_{ij}$. Thus, we equivalently write \begin{align}
    h_{c}(X)=\frac{1}{HW}\sum_{ij}f_{c}(X_{ij}) 
\end{align}
where $f$ is now re-defined as a image patch classifier with 1-dimensional feature output (not fully convolutional).

\myparagraph{CAM as patch-wise posterior approximator.}
$h$ is now the average-pooled value of the patch-wise classification scores $f_{c}(X_{ij})$. CAM trains $f$ by maximizing the image-wide posterior of the ground truth class, where the posterior is defined as the softmax over $f(X_{ij})$:
\begin{equation}
    \log p(Y|\mathbf{X}):=
    \log \text{softmax}^Y\left(\frac{1}{HW}\sum_{ij}f(X_{ij})\right).
\end{equation}
In other words, CAM trains the network for patch-wise scores $f_{c}(X_{ij})$ to estimate the image-wide posterior $p(Y|\mathbf{X})$. 

At inference time, CAM estimates the pixel-wise posterior $p(Y|X_{ij})$ approximately by performing $p(Y|X_{ij})\approx f_Y(X_{ij})/\max_{\alpha\beta}f_Y(X_{\alpha\beta})$ (modulo some calibration to make sure $p(Y|X_{ij})\in [0,1]$).

\subsection{Proof for the ill-posedness lemma}
\label{appendix:proof}

We first define an evaluation metric for our score map for an easier argumentation. 
\begin{definition}
For a scoring rule $s$ and a threshold $\tau$, we define the \textbf{pixel-wise localization accuracy} $\text{\pxacc}(s,\tau)$ as the probability of correctly predicting the pixel-wise labels:
\begin{align}
    \text{\pxacc}(s,\tau)=
    {P}_{X,T}(s(X)\geq\tau\mid T=1)\cdot{P}_{X,T}(T=1)
    \nonumber\\
    +{P}_{X,T}(s(X)<\tau\mid T=0)\cdot{P}_{X,T}(T=0)
    \nonumber
\end{align}
\end{definition}
We prove the following lemma.
\begin{lemma}
    Assume that the true posterior $p(Y|M)$ with a continuous pdf is used as the scoring rule $s(M)=p(Y|M)$. Then, there exists a scalar $\tau\in\mathbb{R}$ such that $\text{\pxacc}(p(Y|M),\tau)=1$ if and only if the foreground-background posterior ratio $\frac{p(Y=1|M^{\text{fg}})}{p(Y=1| M^{\text{bf}})}\geq 1$ almost surely, conditionally on the event $\{T(M^{\text{fg}})=1\text{ and }T(M^{\text{bf}})=0\}$.
\end{lemma}
\begin{proof}

    We write $E:=\{T(M^{\text{fg}})=1\text{ and }T(M^{\text{bf}})=0\}$.
    
    \noindent
    (\textbf{Proof for ``if''}) Assume $\alpha\geq 1$ almost surely, given $E$. Let
    \begin{align}
        \tau:=\min_{G:P(G\Delta \{T(m)=0\})=0}\,\,
        \max_{m\in G}\,\,
        p(Y=1| M=m)
    \end{align} 
    where $\Delta$ is the set XOR operation: $A\Delta B:=(A\cup B)\setminus (A\cap B)$. Then, for almost all $M^{\text{fg}},M^{\text{bg}}$ following $E$, 
    \begin{align}
        p(Y=1| M^{\text{fg}})\geq \tau \geq p(Y=1| M^{\text{bg}}).
        \label{eq:co_occurrence_ordered}
    \end{align}
    Therefore, 
    \begin{align}
        &P(p(Y=1| M^{\text{fg}})\geq \tau\mid T(M^{\text{fg}})=1)\nonumber\\
        &=P(p(Y=1| M^{\text{bg}})\leq \tau\mid T(M^{\text{bg}})=0) =1
        \label{eq:loc_conditionals_prob_1}
    \end{align}
    and so $\text{\pxacc}(p(Y|M),\tau)=1$.
    
    \noindent
    (\textbf{Proof for ``only if''}) Assume $\text{\pxacc}(p(Y| M),\tau)=1$ for some $\tau$. W.L.O.G., we assume that ${P}(T(M)=1)\neq 0$ and ${P}(T(M)=0)\neq 0$ (otherwise, ${P}(E)=0$ and the statement is vacuously true). Then, Equation~\ref{eq:loc_conditionals_prob_1} must hold to ensure $\text{\pxacc}(p(Y| M),\tau)=1$. Equation~\ref{eq:co_occurrence_ordered} then also holds almost surely, implying $\alpha\geq 1$ almost surely.
\end{proof}

\subsection{Foreground-background posterior ratio}
\label{appendix:fg_bg_posterior_ratio}

We have described the the pathological scenario for WSOL as when the foreground-background posterior ratio $\alpha$ is small (\S3.2). We discuss in greater detail what it means and whether there are data-centric approaches to resolve the issue. For quick understanding, assume the task is the localization of duck pixels in images. The foreground cue of interest $M^{\text{fg}}$ is ``feet'' of a duck and background cue of interest $M^{\text{bg}}$ is ``water''. Then, we can write the posterior ratio as
{
\begin{equation}
    \alpha:=
    \frac{p(\eqd|\eqf)}{p(\eqd|\eql)}= 
    \frac{p(\eqf|\eqd)}{p(\eql|\eqd)} \cdot \left(\frac{p(\eqf)}{p(\eql)}\right)^{-1}
    \nonumber
\end{equation}
$\alpha<1$ implies that lake patches are more abundant in duck images than are duck's feet (see Figure~\ref{fig:duck_feet_lake_appendix}) for an illustration. 

\begin{figure}
    \centering
    \includegraphics[width=\linewidth]{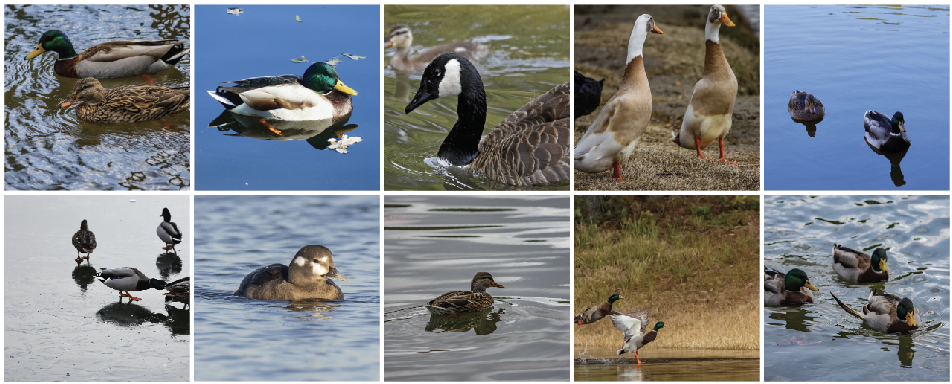}
    \caption{\small\textbf{Ducks.} Random duck images on Flickr. They contain more lake than feet pixels: $p(\eql|\eqd)\gg p(\eqf|\eqd)$.}
    \label{fig:duck_feet_lake_appendix}
\end{figure}

To increase $\alpha$, two approaches can be taken. (1) Increase the likelihood ratio $\frac{p(\eqf|\eqd)}{p(\eql|\eqd)}$. This can be done by collecting more images where duck's feet have more pixels than lake does. (2) Decrease the prior ratio $\frac{p(\eqf)}{p(\eql)}$. Note that the prior ratio can be written
\begin{equation}
\footnotesize
    \frac{p(\eqf)}{p(\eql)}=
    \frac{p(\eqf|\eqd)p(\eqd)+p(\eqf|\eqnd)p(\eqnd)}%
    {p(\eql|\eqd)p(\eqd)+p(\eql|\eqnd)p(\eqnd)}
    \nonumber
\end{equation}
With fixed likelihoods $p(\eqf|\eqd)$ and $p(\eql|\eqd)$, one can decrease the prior ratio by increasing the likelihood of lake cues in non-duck images $p(\eql|\eqnd)$. We can alter WSOL into a more well-posed task also by including many background images containing confusing background cues.

Such data-centric approaches are promising future research directions for turning WSOL into a well-posed task. 
}

\section{Evaluation Protocol for WSOL}
\label{appendix:evaluation_protocol}

\subsection{Score map normalization}
\label{appendix:score_map_normalization}

A common practice in WSOL is to normalize the score maps per image because the maximal (and minimal) scores differ vastly across images. Prior WSOL papers have introduced either max normalization (dividing through by $\max_{ij}s_{ij}$) or min-max normalization (additionally mapping $\min_{ij}s_{ij}$ to zero). After normalization, WSOL methods threshold the score map at $\tau$ to generate a tight box around the binary mask $\{(i,j)\mid s_{ij}\geq \tau\}$. $\tau$ is typically treated as a fixed value~\cite{CAM,ACoL,CutMix} or a hyperparameter to be tuned~\cite{HaS,SPG,ADL}. We summarize that how prior works calibrate and threshold score maps in Table \ref{tab:prior_wsol_operating_thresholds}. As discussed in \S\ref{subsec:evaluation_metrics} of main paper, we use the min-max normalization.

\begin{table}
    \newcommand\maxnormf{\overline{s}_{ij}}
    \newcommand\minmaxnormf{\widehat{s}_{ij}}
    \small
    \centering
    \begin{tabular}{*{2}{l}*{2}{l}}
         Method && Paper & Code  \\
         \cline{1-1} \cline{3-4}
         \vspace{-1em} & \\
         \cline{1-1} \cline{3-4}
         \vspace{-1em} & \\
         CAM~\cite{CAM} && $\maxnormf\geq 0.2$ & $\maxnormf\geq 0.2$ \\
         HaS~\cite{HaS} && Follow CAM$^\dagger$ & Follow CAM \\
         ACoL~\cite{ACoL} && Follow CAM & $\minmaxnormf\geq \text{unknown}$ \\
         SPG~\cite{SPG} && Grid search threshold & $\minmaxnormf\geq \text{unknown}$ \\
         ADL~\cite{ADL} && Not discussed & $\minmaxnormf\geq 0.2^\dagger$ \\
         CutMix~\cite{CutMix} && $\maxnormf\geq 0.15$ & $\minmaxnormf\geq 0.15$ \\
         \cline{1-1} \cline{3-4}
         \vspace{-1em} & \\
         Our protocol && $\minmaxnormf\geq \tau^\star$ & $\minmaxnormf\geq \tau^\star$ \\
         \cline{1-1} \cline{3-4}
    \end{tabular}
\[
\overline{s}_{ij}:=\frac{s_{ij}}{\max_{kl}s_{kl}}\quad \quad \widehat{s}_{ij}:=\frac{s_{ij}-\min_{kl}s_{kl}}{\max_{kl}s_{kl}-\min_{kl}s_{kl}}
\]
    \caption{\small \textbf{Calibration and thresholding in WSOL.} Score calibration is done per image: max ($\overline{s}_{ij}$) or min-max ($\widehat{s}_{ij}$) normalization. Thresholding is required only for the box evaluation. $\tau^\star$ is the optimal threshold (\S\ref{subsec:evaluation_metrics} in main paper). Daggers ($^\dagger$) imply that the threshold depends on the backbone architecture.}
    \label{tab:prior_wsol_operating_thresholds}
\end{table}

\subsection{Data preparation}
\label{appendix:data}

We present the following data-wise contributions in this paper (contributions \textbf{bolded}):
\begin{itemize}
    \item CUB: \textbf{New data} (5 images per class) with \textbf{bounding box annotations}.
    \item ImageNet: ImageNetV2~\cite{ImageNetV2} with new \textbf{bounding box annotations}.
    \item OpenImages: \textbf{Processed} a split for the WSOL training, hyperparameter search, and evaluation with ground truth object masks.
\end{itemize}

\subsubsection{ImageNet}
\label{appendixsub:imagenet}

The \textit{test set} of ImageNet-1k dataset~\cite{ImageNet} is not available. Therefore, many researchers report the accuracies on the \textit{validation set} for their final results~\cite{CutMix}. Since this practice may let models overfit to the evaluation split over time, ImageNetV2~\cite{ImageNetV2} has been proposed as the new test sets for ImageNet-1k trained models. ImageNetV2 includes three subsets according to the sampling strategies, \texttt{MatchedFrequency}, \texttt{Threshold0.7}, and \texttt{TopImages}, each with $10\,000$ images ($10$ images per class). We use the \texttt{Threshold0.7} split as our \trainfullsup. Since ImageNetV2 does not contain localization supervision, we have annotated bounding boxes on those images, following the annotation protocol of ImageNet. The total number of annotated bounding boxes are $18\,532$.

\subsubsection{CUB}
\label{appendixsub:cub}

We have collected 5 images for each of the 200 CUB fine-grained bird classes from Flickr. The overall procedure is summarized as:

\begin{enumerate}
    \item Crawl images from Flickr.
    \item De-duplicate images.
    \item Manually prune irrelevant images (three people).
    \item Prune with model classification scores.
    \item Resize images.
    \item Annotate bounding boxes.
\end{enumerate}

\myparagraph{Crawl images from Flickr.}
Since original CUB itself is collected from Flickr, we use Flickr as the source of bird images. We have used the \texttt{Flickr API}\footnote{\url{https://www.flickr.com/services/api/}} to crawl up to 400 images per class with class name as search terms. We have only crawled images under the following licenses: 
\begin{itemize}
    \item \texttt{Attribution}
    \item \texttt{Attribution-NonCommercial}
    \item \texttt{Public Domain Dedication}
    \item \texttt{Public Domain Mark}
\end{itemize}

\myparagraph{De-duplicate images.}
We de-duplicate the images using the \texttt{ImageHash} library\footnote{\url{https://pypi.org/project/ImageHash/}} first among the crawled images themselves and then against the training and test splits of CUB.

\myparagraph{Manually prune irrelevant images (three people).}
Since crawled images of each class contain negative-class images (non birds and birds of wrong categories), three humans have participated in the prune-out process. We have only kept the images that all three have voted for ``positive''.

\myparagraph{Prune with model classification scores.}
To match with the original CUB data distribution, we have trained a fine-grained bird classifier using ResNet50~\cite{ResNet} on the CUB training split. Among crawled images, those with ground truth class confidence scores lower than 0.5 have been pruned out.

\myparagraph{Achieving 5 images per class.}
At this point, most classes have more than 5 images per class and a few have less than 5. In the former case, we have randomly sampled 5 images per class. 

\myparagraph{Resize images.}
Photographic technologies have made significant advances over the decade since the time original CUB was collected. As a result, image size statistics differ. To fix this, we have resized the images appropriately.

\myparagraph{Annotate bounding boxes.}
The evaluation on CUB is based on bounding boxes. Thus, the held-out set images are supplied with tight bounding boxes around birds (one per image).

\begin{figure}
    \centering
    \includegraphics[width=\linewidth]{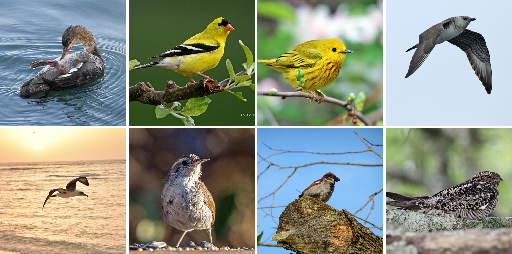}
    \caption{\small \textbf{CUB version 2.} Sample images.}
    \label{fig:cub_v2_samples}
\end{figure}

\subsubsection{OpenImages}
\label{appendixsub:openimages}
There are three significant differences between OpenImagesV5~\cite{OpenImagesV5} and CUB or ImageNet that make the OpenImages not suitable as a WSOL benchmark in its original form. (1) Images are multi-labeled; it is not sensible to train classifiers with the standard softmax cross-entropy loss assuming single label per image. (2) OpenImages has less balanced label distributions. (3) There are nice instance segmentation masks, but they have many missing instances.

We have therefore processed a subset of OpenImages into a WSOL-friendly dataset where the above three issues are resolved. The procedure is as follows:
\begin{enumerate}
    \item Prune multi-labeled samples.
    \item Exclude classes with not enough samples.
    \item Randomly sample images for each class.
    \item Prepare binary masks.
    \item Introduce ignore regions.
\end{enumerate}

\myparagraph{Prune multi-labeled samples.}
We prune the multi-labeled samples in the segmentation subset of OpenImages.
This process rejects 34.5\% samples in OpenImages.

\myparagraph{Exclude classes with not enough samples.}
After pruning multi-label samples, some classes are not available in validation and test sets. We have first defined the minimum number of samples per classes as (300, 25, 50) for (train, validation, test), and have excluded classes that do not meet the minimum requirement, resulting in 100 classes (out of 350 original classes).

\myparagraph{Randomly sample images for each class.}
We have randomly sampled (300, 25, 50) samples per class for (train, validation, test). This eventually results in $29\,819$ \trainweaksup (train), $2\,500$ \trainfullsup (validation), and $5\,000$ \testfullsup (test) samples.

\myparagraph{Prepare binary masks.}
WSOL methods are evaluated against binary masks of the foreground category in each image. Since OpenImages comes with instance-level masks of various categories, we have taken the union of instance masks of the class of interest in every image to build the binary masks.

\begin{figure}
    \centering
    \includegraphics[width=\linewidth]{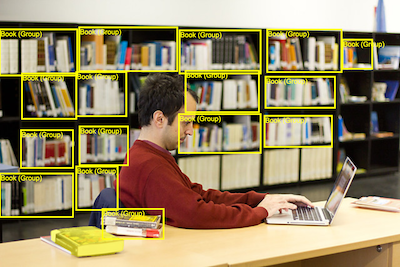}
    \caption{\small \textbf{OpenImages V5 sample.} Example of ``group-of'' box annotations. There is only one ``Book'' mask in the left bottom corner.}
    \label{fig:openimage_group_of_sample}
\end{figure}

\myparagraph{Introduce ignore regions.}
OpenImages has instance-wise masks, but not all of them are annotated as masks. When the objects are difficult to be annotated instance-wise or there are simply too many, the ``group-of'' box annotations are provided over the region containing multiple instances (\eg hundreds of books on bookshelves in an image taken at a library -- See Figure~\ref{fig:openimage_group_of_sample}). We have indicated the regions with the ``group-of'' boxes as ``ignore'' regions, and have excluded them from the evaluation during the computation of \pxprec, \pxrec, and \pxap.

\subsection{Transferability of rankings between \trainfullsup and \testfullsup}
\label{appendix:trainfullsup_testfullsup_transfer}

As detailed in main paper \S\ref{subsec:evaluation_benchmarks}, we search hyperparameters on \trainfullsup and test them on \testfullsup. To validate if the found hyperparameter rankings do transfer well between the splits, we show the preservation of ranking statistics in Table~\ref{tab:supp-val-test-transfer} and visualize the actual rankings in Figure~\ref{fig:val-test-ranking_preservation}. We observe that the rankings are relatively well-preserved (with Kendall's tau values $>0.7$).

\definecolor{Gray}{gray}{0.85}
\newcolumntype{g}{>{\columncolor{Gray}}c}

{
\setlength{\tabcolsep}{3pt}
\begin{table*}[ht!]
\centering
\small
\begin{tabular}{lc*{3}{c}gc*{3}{c}gc*{3}{c}gcg}
& & \multicolumn{4}{c}{ImageNet} & & \multicolumn{4}{c}{CUB}  & & \multicolumn{4}{c}{OpenImages} && \multicolumn{1}{c}{Total}\\
Methods  &  & VGG & Inception & ResNet & Mean &  & VGG & Inception & ResNet & Mean &  & VGG & Inception & ResNet & Mean &  & Mean \\
\cline{1-1}\cline{3-6}\cline{8-11}\cline{13-16}\cline{18-18} & \vspace{-1em} \\
CAM~\cite{CAM} &  & 0.887 & 0.795 & 0.933 & 0.872 &  & 0.731 & 0.706 & 0.758 & 0.732 &  & 0.869 & 0.714 & 0.934 & 0.839 &  & 0.814\\
HaS~\cite{HaS} &  & 0.855 & 0.795 & 0.867 & 0.839 &  & 0.422 & 0.714 & 0.867 & 0.668 &  & 0.786 & 1.000 & 0.667 & 0.818 &  & 0.775\\
ACoL~\cite{ACoL} & & 0.981 & 1.000 & 0.619 & 0.867 &  & 0.895 & 0.850 & 0.850 & 0.854 &  & 0.983 & 0.970 & 0.956 & 0.970 &  & 0.897\\
SPG~\cite{SPG} &  & 0.944 & 0.766 & 0.900 & 0.870 &  & 0.697 & 0.779 & 0.889 & 0.788 &  & 0.758 & 0.874 & 0.929 & 0.854 &  & 0.837\\
ADL~\cite{ADL} &  & 0.917 & 0.944 & 0.857 & 0.906 &  & 0.891 & 1.000 & 0.810 & 0.900 &  & 0.891 & 1.000 & 1.000 & 0.964 &  & 0.923\\
CutMix~\cite{CutMix} && 0.897 & 0.571 & 1.000 & 0.823 &  & 0.544 & 0.407 & 0.879 & 0.610 &  & 0.838 & 0.867 & 0.714 & 0.806 &  & 0.746\\
\cline{1-1}\cline{3-6}\cline{8-11}\cline{13-16}\cline{18-18} & \vspace{-1em} \\
\end{tabular}
\caption{\small \textbf{In-distribution ranking preservation.} Kendall's tau values for the hyperparameter ranking between \trainfullsup and \testfullsup are shown.}
\label{tab:supp-val-test-transfer}
\end{table*}
}

\subsection{Hyperparameter search with proxy \trainweaksup}
\label{appendix:okay_to_use_proxy_imagenet}

We have reduced the training set size (\trainweaksup) to 10\% for ImageNet hyperparameter search for the interest of computational efficiency (\S4.2). We examine how much this reduction affects the rankings of hyperparameters. We show a similar analysis as in \S\ref{appendix:trainfullsup_testfullsup_transfer}: see Figure~\ref{fig:proxy-imagenet}. We observe again that with Kendall's tau value $0.743$, the two rankings are largely preserved.

\begin{figure}
    \centering
    \includegraphics[width=\linewidth]{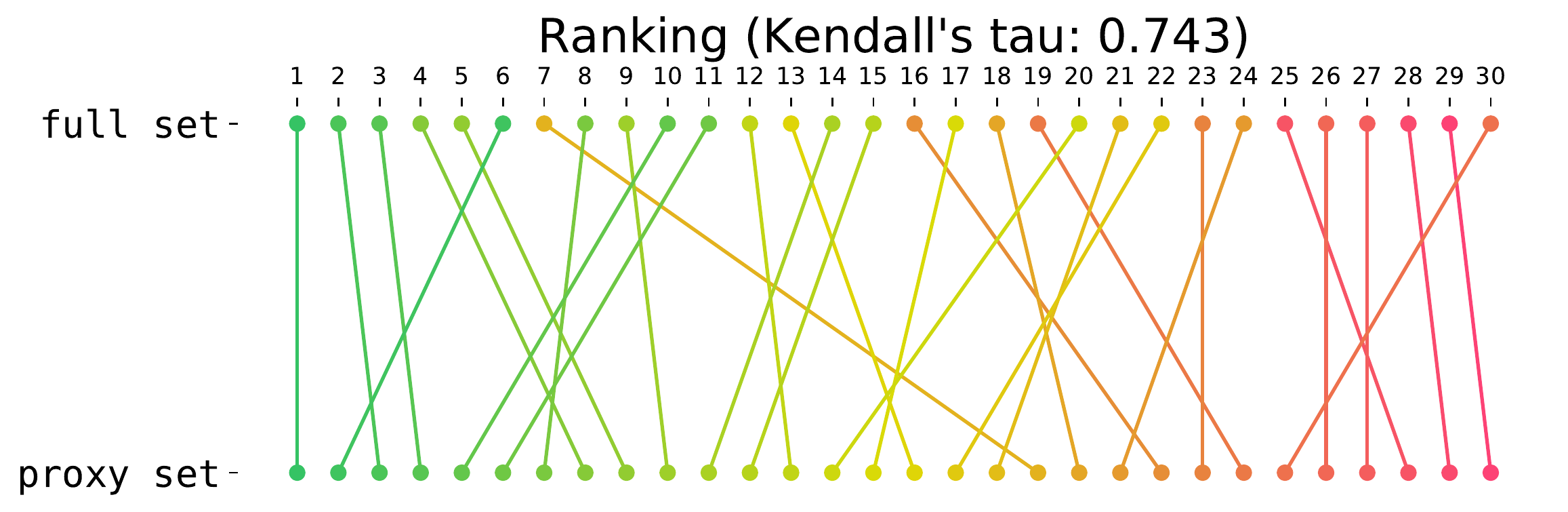}
    \caption{\small \textbf{Proxy ImageNet ranking}. Ranking of hyperparameters is largely preserved between the models trained on the full \trainweaksup and its 10\% proxy. Kendall's tau is $0.743$.}
    \label{fig:proxy-imagenet}
\end{figure}

\section{Experiments}
\label{appendix:experiments}

\subsection{Prior WSOL methods and hyperparameters}
\label{appendix:wsol_methods}

\begin{table}
    \footnotesize
    \centering
    \begin{tabular}{cccc}
        Methods && Hyperparameter & Distribution  \\
        \cline{1-1} \cline{3-4}
        \vspace{-1em} & \\
        \cline{1-1} \cline{3-4}
        \vspace{-1em} & \\
        Common && Learning rate & $\text{LogUniform}[10^{-5},10^0]$ \\
        && Score-map resolution & $\text{Categorical}\{14,28\}$ \\
        \cline{1-1} \cline{3-4}
        \vspace{-1em} & \\
        HaS~\cite{HaS} && Drop rate & $\text{Uniform}[0,1]$ \\
        && Drop area & $\text{Uniform}[0,1]$  \\
        \cline{1-1} \cline{3-4}
        \vspace{-1em} & \\  
        ACoL~\cite{ACoL} && Erasing threshold & $\text{Uniform}[0,1]$ \\
        \cline{1-1} \cline{3-4}
        \vspace{-1em} & \\
        SPG~\cite{SPG} && Threshold $\delta_{l}^{B1}$ & $\text{Uniform}[0,1]$ \\
            && Threshold $\delta_{h}^{B1}$ & $\text{Uniform}[\delta_{l}^{B1},1]$ \\
            && Threshold $\delta_{l}^{B2}$ & $\text{Uniform}[0,1]$ \\
            && Threshold $\delta_{h}^{B2}$ & $\text{Uniform}[\delta_{l}^{B2},1]$ \\
            && Threshold $\delta_{l}^{C}$ & $\text{Uniform}[0,1]$ \\
            && Threshold $\delta_{h}^{C}$ & $\text{Uniform}[\delta_{l}^{C},1]$ \\
        \cline{1-1} \cline{3-4}
        \vspace{-1em} & \\
        ADL~\cite{ADL} && Drop rate & $\text{Uniform}[0,1]$ \\
        && Erasing threshold & $\text{Uniform}[0,1]$ \\
        \cline{1-1} \cline{3-4}
        \vspace{-1em} & \\
        CutMix~\cite{CutMix} && Size prior & $\frac{1}{\text{Uniform}(0,2]}-\frac{1}{2}$ \\
        && Mix rate & $\text{Uniform}[0,1]$ \\
        \cline{1-1} \cline{3-4}
    \end{tabular}
    \vspace{1em}
    \caption{\small \textbf{Hyperparameter search spaces.} }
    \label{tab:hyperparameter_list}
\end{table}

We describe each method in greater detail here. The list of hyperparameters for each method is in Table~\ref{tab:hyperparameter_list}.

\definecolor{darkergreen}{RGB}{21, 152, 56}
\definecolor{red2}{RGB}{252, 54, 65}

\definecolor{Gray}{gray}{0.85}
\newcolumntype{g}{>{\columncolor{Gray}}c}

{
\setlength{\tabcolsep}{3pt}
\renewcommand{\arraystretch}{1.1}
\begin{table*}[ht!]
\resizebox{\textwidth}{!}{
\centering
\small
\begin{tabular}{lc*{3}{c}gc*{3}{c}gc*{3}{c}gcg}
& & \multicolumn{4}{c}{ImageNet} & & \multicolumn{4}{c}{CUB}  & & \multicolumn{4}{c}{OpenImages} && \multicolumn{1}{c}{Total}\\
Methods  &  & VGG & Inception & ResNet & Mean &  & VGG & Inception & ResNet & Mean &  & VGG & Inception & ResNet & Mean &  & Mean \\
\cline{1-1}\cline{3-6}\cline{8-11}\cline{13-16}\cline{18-18} & \vspace{-1em} \\
CAM~\cite{CAM} &  & 63.8 & 68.7 & 75.9 & 69.5 &  & 26.8 & 61.8 & 58.4 & 49.0 &  & 67.3 & 36.6 & 72.6 & 58.8 &  & 59.1\\
HaS~\cite{HaS} &  & 61.9 & 65.5 & 63.1 & 63.5 &  & 70.9 & 69.9 & 74.5 & 71.8 &  & 60.0 & 68.4 & 74.0 & 67.5 &  & 67.6\\
ACoL~\cite{ACoL} &  & 60.3 & 64.6 & 61.6 & 62.2 &  & 56.1 & 71.6 & 64.0 & 63.9 &  & 68.2 & 40.7 & 70.7 & 59.9 &  & 62.0\\
SPG~\cite{SPG} &  & 61.6 & 65.5 & 63.4 & 63.5 &  & 63.1 & 58.8 & 37.8 & 53.2 &  & 71.7 & 43.5 & 65.4 & 60.2 &  & 59.0\\
ADL~\cite{ADL} &  & 60.8 & 61.6 & 64.1 & 62.2 &  & 31.1 & 45.5 & 32.7 & 36.4 &  & 66.1 & 46.6 & 56.1 & 56.3 &  & 51.6\\
CutMix~\cite{CutMix} &  & 62.2 & 65.5 & 63.9 & 63.8 &  & 29.2 & 70.2 & 55.9 & 51.8 &  & 68.1 & 53.1 & 73.7 & 65.0 &  & 60.2\\
\cline{1-1}\cline{3-6}\cline{8-11}\cline{13-16}\cline{18-18} & \vspace{-1em} \\
\end{tabular}
}
\caption{\small \textbf{Classification performance of WSOL methods.} Classification accuracies of the models in Table 2 of the main paper are shown. Hyperparameters for each model are optimally chosen for the localization performances on \trainfullsup split. Classification performances may be sub-optimal when the best localization performance is achieved at early epochs.}
\label{tab:main_cls}
\end{table*}
}

\begin{table*}[ht!]
\centering
{
\small
\setlength{\tabcolsep}{0.15em}
\begin{tabular}{l*{22}{c}}
 &&  & \multicolumn{7}{c}{Top-1 localization accuracy}  & & \multicolumn{11}{c}{GT-known localization; \maxboxacc and \pxap} \\
\cline{4-10} \cline{12-22}
\vspace{-1em} & \\
&&\hspace{0.5em} & \multicolumn{3}{c}{ImageNet} & \hspace{0.3em} & \multicolumn{3}{c}{CUB} & \hspace{0.5em} & \multicolumn{3}{c}{ImageNet} & \hspace{0.3em} & \multicolumn{3}{c}{CUB} & \hspace{0.3em} & \multicolumn{3}{c}{OpenImages}  \\
\cline{4-6} \cline{8-10} \cline{12-14} \cline{16-18} \cline{20-22}
\vspace{-1em} & \\
&Methods && {V} &{I} &{R} &&{V} &{I} &{R} &&{V} &{I} &{R} &&{V} &{I} &{R} &&{V} &{I} &{R} \\
\cline{2-2} \cline{4-6} \cline{8-10} \cline{12-14} \cline{16-18} \cline{20-22}
\vspace{-1em} & \\
\cline{2-2} \cline{4-6} \cline{8-10} \cline{12-14} \cline{16-18} \cline{20-22}
\vspace{-1em} & \\
\multirow{6}{*}{\rotatebox{90}{Reported\hspace{0.0em}}} & CAM~\cite{CAM} &  & 42.8 & - & 46.3 &  & 37.1 & 43.7 & 49.4 &  & - & 62.7 & - &  & - & - & - &  & - & - & -\\
 & HaS~\cite{HaS} &  & - & - & - &  & - & - & - &  & - & - & - &  & - & - & - &  & - & - & -\\
 & ACoL~\cite{ACoL} &  & 45.8 & - & - &  & 45.9 & - & - &  & - & - & - &  & - & - & - &  & - & - & -\\
 & SPG~\cite{SPG} &  & - & 48.6 & - &  & - & 46.6 & - &  & - & 64.7 & - &  & - & - & - &  & - & - & -\\
 & ADL~\cite{ADL} &  & 44.9 & 48.7 & - &  & 52.4 & 53.0 & - &  & - & - & - &  & 75.4 & - & - &  & - & - & -\\
 & CutMix~\cite{CutMix} &  & 43.5 & - & 47.3 &  & - & 52.5 & 54.8 &  & - & - & - &  & - & - & - &  & - & - & -\\
\cline{2-2} \cline{4-6} \cline{8-10} \cline{12-14} \cline{16-18} \cline{20-22}
\vspace{-1em} & \\
\multirow{6}{*}{\rotatebox{90}{Reproduced\hspace{0.0em}}} & CAM~\cite{CAM} &  & 45.5 & 48.8 & 51.8 &  & 45.8 & 40.4 & 56.1 &  & 61.1 & 65.3 & 64.2 &  & 71.1 & 62.1 & 73.2 &  & 58.1 & 61.4 & 58.0\\
 & HaS~\cite{HaS} &  & 46.3 & 49.7 & 49.9 &  & 55.6 & 41.1 & 60.7 &  & 61.9 & 65.5 & 63.1 &  & 76.2 & 57.7 & 78.1 &  & 56.9 & 58.5 & 58.2\\
 & ACoL~\cite{ACoL} &  & 45.5 & 49.9 & 47.4 &  & 44.8 & 46.8 & 57.8 &  & 60.3 & 64.6 & 61.6 &  & 72.3 & 59.5 & 72.7 &  & 54.6 & 63.0 & 57.8\\
 & SPG~\cite{SPG} &  & 44.6 & 48.6 & 48.5 &  & 42.9 & 44.9 & 51.5 &  & 61.6 & 65.5 & 63.4 &  & 63.7 & 62.7 & 71.4 &  & 55.9 & 62.4 & 57.7\\
 & ADL~\cite{ADL} &  & 44.4 & 45.0 & 51.5 &  & 39.2 & 35.2 & 41.1 &  & 60.8 & 61.6 & 64.1 &  & 75.6 & 63.3 & 73.5 &  & 58.3 & 62.1 & 54.3\\
 & CutMix~\cite{CutMix} &  & 46.1 & 49.2 & 51.5 &  & 47.0 & 48.3 & 54.5 &  & 63.9 & 62.2 & 65.4 &  & 71.9 & 65.5 & 67.8 &  & 58.2 & 61.7 & 58.6\\
\cline{2-2} \cline{4-6} \cline{8-10} \cline{12-14} \cline{16-18} \cline{20-22}
\end{tabular}
}
{
\small
\setlength{\tabcolsep}{0.4em}
\begin{tabular}{cll}
     & \\
     &&Architecture  \\
     \cline{1-1}\cline{3-3}
     V&&VGG-GAP~\cite{VGG}\\
     I&&InceptionV3~\cite{InceptionV3}\\
     R&&ResNet50~\cite{ResNet} \\
     \cline{1-1}\cline{3-3}
\vspace{-14.6em}
\end{tabular}
}

\caption{\small \textbf{Previously reported vs our results.} The first six rows are reported results in prior WSOL papers. When there are different performance reports for the same method in different papers, we choose the greater performance. The last six rows are our re-implemented results under the new evaluation method (\S4).}
\label{tab:previous_current_results}
\end{table*}

\paragraph{CAM, CVPR'16~\cite{CAM}.}
Class Activation Mapping trains a classifier of fully-convolutional backbone with global average pooling structure, and at test time uses the logit outputs before the global average pooling for the scoring rule $s(X_{ij})$. See Appendix~\S\ref{appendix:cam_as_posterior_approximation} for the interpretation of CAM as a posterior approximator. CAM has learning rate (LR) and the score-map resolution (SR) as hyperparameters. LR is sampled log-uniformly from $[10^{-5},10^{0}]$, where end points correspond roughly to ``no training'' and ``training always diverges'' cases. SR is sampled from $\text{Categorical}\{14,28\}$, two widely used resolutions in prior WSOL methods. All five methods below use CAM technique in the background, and have LR and HR as design choices.

\paragraph{HaS, ICCV'17~\cite{HaS}.}
Hide-and-Seek (HaS) is a data augmentation technique that divides an input image into grid-like patches, and then randomly select patches to be dropped. The hyperparameters of HaS are drop rate (DR) and drop area (DA). Specifically, the size of each patch is decided by DA, and the probability of each patch to be selected for erasing is decided by DR. DA is sampled from a uniform distribution $U[0,1]$, where $0$ corresponds to ``no grid'' and $1$ indicates ``full image as one patch''.

\paragraph{ACoL, CVPR'18~\cite{ACoL}.}
Adversarial Complementary Learning (ACoL) adds one more classification head to backbone networks. From one head, ACoL finds the high-score region using CAM using erasing threshold (ET) and erases it from an internal feature map. The other head learns remaining regions using the erased feature map. We sample ET from a uniform distribution $U[0,1]$, where $0$ means ``erasing whole feature map'' and $1$ means ``do not erase''.

\paragraph{SPG, ECCV'18~\cite{SPG}.}
Self-produced Guidance (SPG) utilizes spatial information about fore- and background using three additional branches (\texttt{SPG-B1,B2,C}). To divide foreground and background from score-map, they introduce two hyperparameters, $\delta_{l}$ and $\delta_{h}$, per each branch. When the score is lower than $\delta_{l}$, the pixel is considered as background, and the pixel is considered as foreground when the score is higher than  $\delta_{h}$. The remaining region (higher than $\delta_{l}$, lower than  $\delta_{h}$) is ignored. We first sample $\delta_{l}$ from $U[0,1]$, and then $\delta_{h}$ is sampled from $U[\delta_{l},1]$.

\paragraph{ADL, CVPR'19~\cite{ADL}.}
Attention-based Dropout Layer (ADL) is a block applied on an internal feature map during training. ADL produces a drop mask by finding the high-score region to be dropped using another scoring rule~\cite{zagoruyko2017paying}. Also, ADL produces an importance map by normalizing the score map and uses it to increase classification power of the backbone. At each iteration, only one component is applied between the drop mask and importance map. The hyperparameters of ADL are drop rate (DR) that indicates how frequently the drop mask is selected and erasing threshold (ET) that means how large regions are dropped. We sample DR and ET from uniform distributions $U[0,1]$.

\paragraph{CutMix, ICCV'19~\cite{CutMix}.}
CutMix is a data augmentation technique, where patches in training images are cut and paste to other images and target labels are mixed likewise. Its hyperparameters consist of the size prior $\alpha$ (used for sampling sizes according to $\sim\!\!\text{Beta}(\alpha,\alpha)$) and the mix rate $r$ (Bernoulli decision for ``CutMix or not''). The size prior is sampled from the positive range $\frac{1}{\text{Unif}(0,2]}-\frac{1}{2}$; then, $\text{Var}(\text{Beta}(\alpha,\alpha))$ follows the uniform distribution between 0 and 0.25 (maximal variance; two Dirac deltas at 0 and 1).

\subsection{Classification results of WSOL methods}
\label{appendix:cls_performance}

The widely-used ``top-1 localization accuracy'' for WSOL represents both the classification \textit{and} localization performances. The metric can be misleading as a localization metric, as the increase in numbers can also be attributed to the improved classification accuracies. Thus, in the main paper, we have suggested using on the ``GT-known'' type of metrics like \maxboxacc and \pxap that measures the localization performances given perfect classification.

Here, we measure the classification accuracies of the models in Table~\ref{tab:main} of the main paper, to complete the analysis. The performances are reported in Table~\ref{tab:main_cls}. There are in general great fluctuations in the classification results (26.8\% to 74.5\% on CUB). This is because we have selected the hyperparameters for each model that maximize the localization performances on \trainfullsup split. In many cases, the best localization performances are achieved at early epochs, before the classifiers are sufficiently trained (see also \S\ref{appendix:learning_curves} and Figure~\ref{fig:wsol_learning_curves}). The result signifies that localization and classification performances may not necessarily correlate and, therefore, localization-only metrics like \maxboxacc and \pxap must be used for model selection and evaluation of WSOL methods.

\definecolor{darkergreen}{RGB}{21, 152, 56}
\definecolor{red2}{RGB}{252, 54, 65}
\definecolor{Gray}{gray}{0.85}
\newcolumntype{g}{>{\columncolor{Gray}}c}

{
\setlength{\tabcolsep}{3pt}
\renewcommand{\arraystretch}{1.1}
\begin{table*}[ht!]
\resizebox{\textwidth}{!}{%
\centering
\small
\begin{tabular}{lc*{3}{c}gc*{3}{c}gc*{3}{c}gcg}
& & \multicolumn{4}{c}{ImageNet (\newmaxboxacc)} & & \multicolumn{4}{c}{CUB (\newmaxboxacc)}  & & \multicolumn{4}{c}{OpenImages (\pxap)} && \multicolumn{1}{c}{Total}\\
Methods  &  & VGG & Inception & ResNet & Mean &  & VGG & Inception & ResNet & Mean &  & VGG & Inception & ResNet & Mean &  & Mean \\
\cline{1-1}\cline{3-6}\cline{8-11}\cline{13-16}\cline{18-18} & \vspace{-1em} \\
CAM~\cite{CAM} &  & 60.0 & 63.4 & 63.7 & 62.4 &  & 63.7 & 56.7 & 63.0 & 61.1 &  & 58.3 & 63.2 & 58.5 & 60.0 &  & 61.2\\
HaS~\cite{HaS} &  & \tableplus{+0.6} & \tableplus{+0.3} & \tableminus{-0.3} & \tableplus{+0.2} &  & \tableplus{+0.0} & \tableminus{-3.3} & \tableplus{+1.7} & \tableminus{-0.5} &  & \tableminus{-0.2} & \tableminus{-5.1} & \tableminus{-2.6} & \tableminus{-2.6} &  & \tableminus{-1.0}\\
ACoL~\cite{ACoL} &  & \tableminus{-2.6} & \tableplus{+0.3} & \tableminus{-1.4} & \tableminus{-1.2} &  & \tableminus{-6.3} & \tableminus{-0.5} & \tableplus{3.5} & \tableminus{-1.1} &  & \tableminus{-4.0} & \tableminus{-6.0} & \tableminus{-1.2} & \tableminus{-3.7} &  & \tableminus{-2.0}\\
SPG~\cite{SPG} &  & \tableminus{-0.1} & \tableminus{-0.1} & \tableminus{-0.4} & \tableminus{-0.2} &  & \tableminus{-7.4} & \tableminus{-0.8} & \tableminus{-2.6} & \tableminus{-3.6} &  & \tableplus{+0.0} & \tableminus{-0.9} & \tableminus{-1.8} & \tableminus{-0.9} &  & \tableminus{-1.6}\\
ADL~\cite{ADL} &  & \tableminus{-0.2} & \tableminus{-2.0} & \tableplus{+0.0} & \tableminus{-0.7} &  & \tableplus{+2.6} & \tableplus{+2.1} & \tableminus{-4.6} & \tableplus{+0.0} &  & \tableplus{+0.4} & \tableminus{-6.4} & \tableminus{-3.3} & \tableminus{-3.1} &  & \tableminus{-1.3}\\
CutMix~\cite{CutMix} &  & \tableminus{-0.6} & \tableplus{+0.5} & \tableminus{-0.4} & \tableminus{-0.2} &  & \tableminus{-1.4} & \tableplus{+0.8} & \tableminus{-0.2} & \tableminus{-0.3} &  & \tableminus{-0.2} & \tableminus{-0.7} & \tableminus{-0.8} & \tableminus{-0.6} &  & \tableminus{-0.3}\\
\cline{1-1}\cline{3-6}\cline{8-11}\cline{13-16}\cline{18-18} & \vspace{-1em} \\
Best WSOL &  & 60.6 & 63.9 & 63.7 & 62.6 &  & 66.3 & 58.8 & 66.4 & 61.1 &  & 58.7 & 63.2 & 58.5 & 60.0 &  & 61.2\\
FSL baseline &  & 60.3 & 65.3 & 66.3 & 64.0 &  & 71.6 & 86.6 & 82.4 & 80.2 &  & 65.9 & 74.1 & 74.4 & 71.5 &  & 71.9\\
Center baseline &  & 52.5 & 52.5 & 52.5 & 52.5 &  & 59.7 & 59.7 & 59.7 & 59.7 &  & 45.8 & 45.8 & 45.8 & 45.8 &  & 52.3\\
\cline{1-1}\cline{3-6}\cline{8-11}\cline{13-16}\cline{18-18} & \vspace{-1em} \\
\end{tabular}
}
\caption{\small \textbf{Re-evaluating WSOL with \newmaxboxacc}. We re-evaluate six recently proposed WSOL methods with \newmaxboxacc. The experimental setting is the same as that of Table~{\color{red}2} in the main paper.}
\label{tab:main_v2}
\vspace{-1em}
\end{table*}
}

\subsection{Reproducing prior WSOL results}
\label{appendix:reproducing}

We also summarize reported results in prior WSOL papers~\cite{CAM,HaS,ACoL,SPG,ADL,CutMix} along with our reproduced results in Table~\ref{tab:previous_current_results}. While our main evaluation metrics are the classification-disentangled GT-known measures (\S\ref{subsec:evaluation_metrics} in main paper), we also report the \textit{top-1 localization} metrics that also measure the classification scores to match the reported numbers. Note that we use oracle $\tau$ for Top-1 localization accuracy. 

We experimentally observe that training epochs and batch sizes influence the localization accuracies ($\sim10$pp in \maxboxacc). However, the training details for each method are not given by the corresponding papers. All methods should share the same training budget for fair comparisons. Therefore, we fix our training epochs to ($10$, $50$, $10$) for (ImageNet, CUB, OpenImages) experiments, and fix the batch size to $32$ for all datasets. 

Following the prior methods~\cite{CAM, HaS, ACoL, SPG, ADL, CutMix}, we train all six methods by fine-tuning ImageNet pre-trained model. Specifically, we apply $10\times$ scale of learning rate to higher-level layers of backbone networks during training. We consider the last two layers for VGG, last three layers for InceptionV3, and last three residual blocks, and fully connected layers for ResNet as the higher-level layers.

Note that with GT-known metrics, our re-implementations produce better performances than the previously reported results (\eg $62.7\rightarrow 65.3$ for CAM on ImageNet with Inception backbone). This is perhaps because \maxboxacc and \pxap are based on the best operating thresholds.

Top-1 localization accuracies are reproduced well in general, except for the ADL: \eg $52.4\rightarrow 39.2$ for ADL on CUB with VGG backbone. This is due to the reduced training epochs compared to the original paper~\cite{ADL}, which results in a decreased classification accuracies.

\subsection{Score calibration and thresholding}
\label{appendix:score_calibration}

The operating threshold $\tau$ for the score map $s$ is an important parameter for WSOL. We show additional plots and visualizations to aid understanding. In Figure~\ref{fig:all_3_by_3_threshold_plots}, we show the \boxacc and \pxprec-\pxrec performances at different operating thresholds with diverse architectures and datasets. It extends the main paper Figure~\ref{fig:cam_threshold_and_pr_curves}. We observe again that the optimal operating thresholds $\tau^\star$ are vastly different across data and architectures for \boxacc. The \maxboxacc in each method are relatively stable across methods especially on ImageNet. OpenImages \pxprec-\pxrec curves do not exhibit big differences amongst WSOL methods, likewise.

In Figure \ref{fig:cam_value_dist}, we visualize score distributions for different WSOL methods. We observe that methods have different distributions of scores; ACoL in particular tends to generate flatter score maps. Comparing dataset, we observe that OpenImages tends to have more peaky score distributions. It is therefore important to find the optimal operating point for each method and dataset for fair comparison.

\begin{figure}
    \centering
    \includegraphics[width=\linewidth]{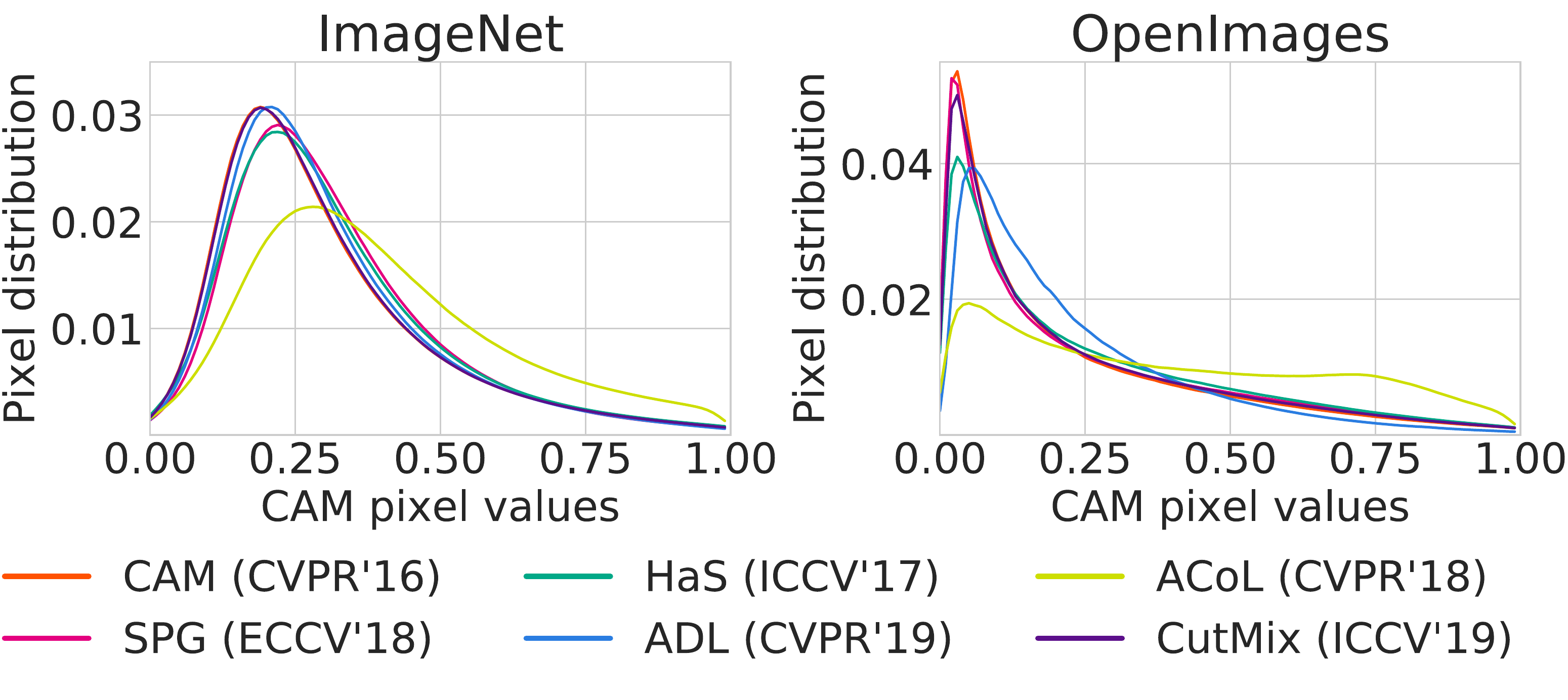}
    \caption{\small \textbf{CAM pixel value distributions.} On ImageNet and OpenImages \testfullsup.}
    \label{fig:cam_value_dist}
\end{figure}

We visualize more score maps in Figure~\ref{fig:score_map_visualization_imagenet},~\ref{fig:score_map_visualization_cub},~\ref{fig:score_map_visualization_openimages}; it extends the main paper Figure~\ref{fig:qualitative_cam_threshold}. We observe qualitatively different score maps in general across methods. However, the optimal IoU values are not as different and hard to predict just by visually looking at the samples. Again, it is important to have an objective way to set the thresholds $\tau$ in each case.

\subsection{Hyperparameter analysis}

We show the performance of all 30 hyperparameter search iterations on three datasets with three backbone architectures in Figure~\ref{fig:all_3_by_3_violin_plots}; it extends the main paper Figure~\ref{fig:wsol_robustness}. We observe similar trends as in the main paper. (1) Performances do vary according to the hyperparameter choice, so the hyperparameter optimization is necessary for the optimal performances. (2) CAM is among the more stable WSOL methods. (3) ACoL and ADL show greater sensitivity to hyperparameters in general. (4) CUB is a difficult benchmark where random choice of hyperparameters is highly likely to lead to performances worse than the center-Gaussian baseline.

\subsection{Learning Curves.}
\label{appendix:learning_curves}

We show the learning curves for CUB, ImageNet, OpenImages in Figure~\ref{fig:wsol_learning_curves}. We observe that the performances converge at epochs ($10$, $50$, $10$) for (ImageNet, CUB, OpenImages); the number of epochs for training models is thus sufficient for all methods. Note that learning rates are decayed by 10 every ($3$, $15$, $3$) epochs for (ImageNet, CUB, OpenImages). For the most cases, performance increases as the training progresses. However, in some cases (\eg CUB, SPG and OpenImages, CAM cases), best performances have already achieved at early iteration. That is, classification training rather hurts localization performance. Analyzing this is an interesting future study.  

\subsection{Evaluating WSOL methods with \newmaxboxacc}
\label{appendix:newmaxboxacc}

We re-evaluate six recently proposed WSOL methods with \newmaxboxacc, and the results are shown in Table~\ref{tab:main_v2}. All training configurations but evaluation metrics are the same as those of the main paper. We evaluate the last checkpoint of each training session. We obtain the same conclusions as with the original metric, \maxboxacc: (1) there has been no significant progress in WSOL performances beyond vanilla CAM~\cite{CAM} and (2) with the same amount of fully-supervised samples, FSL baselines provide better performances than the existing WSOL methods.

\clearpage

\begin{figure*}[ht!]
    \centering
    \begin{subfigure}[b]{.49\linewidth}
        \includegraphics[width=\linewidth]{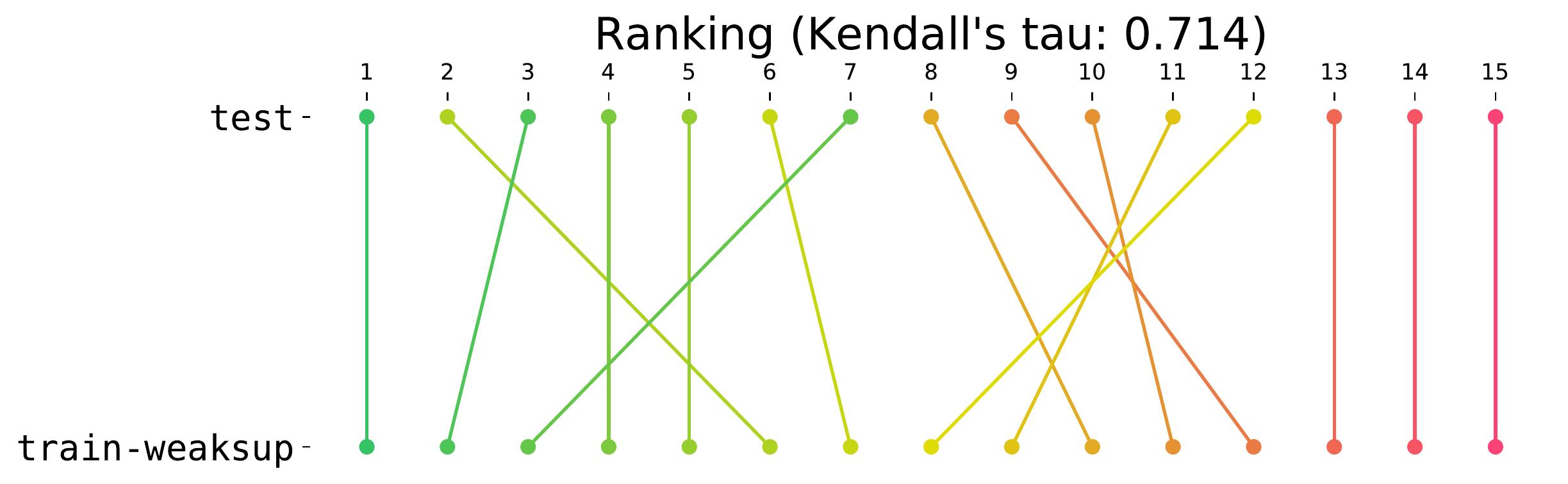}%
        \caption{\small CAM}
        \vspace{2em}
    \end{subfigure}
    \begin{subfigure}[b]{.49\linewidth}
        \includegraphics[width=\linewidth]{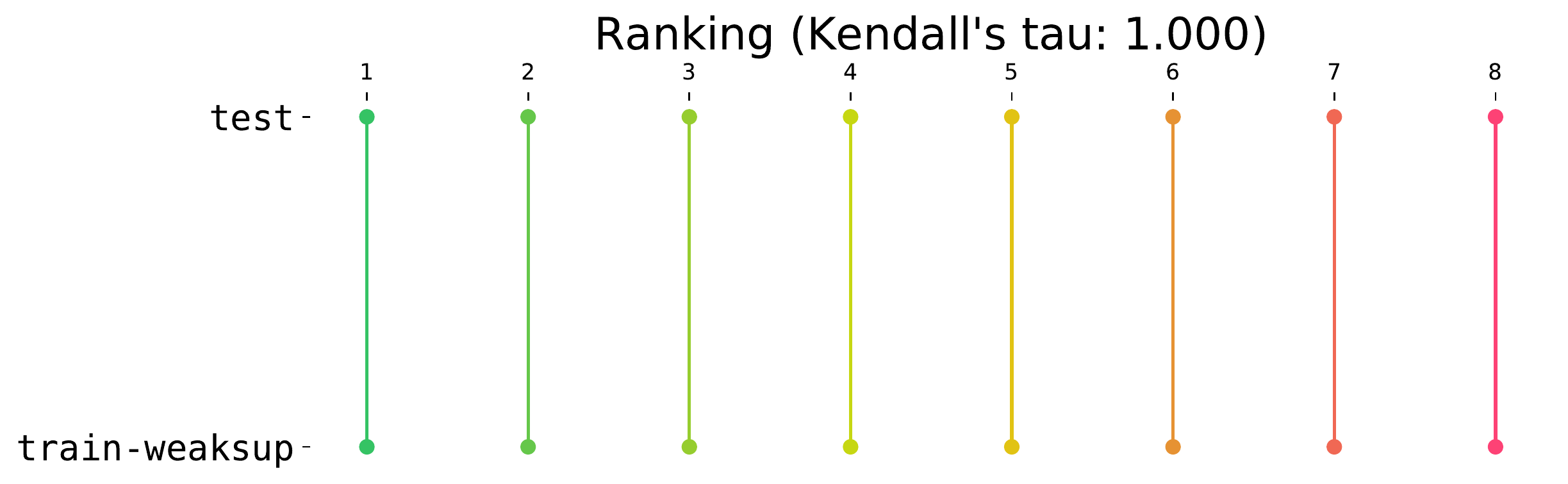}%
        \caption{\small HaS}
        \vspace{2em}
    \end{subfigure}
    \begin{subfigure}[b]{.49\linewidth}
        \includegraphics[width=\linewidth]{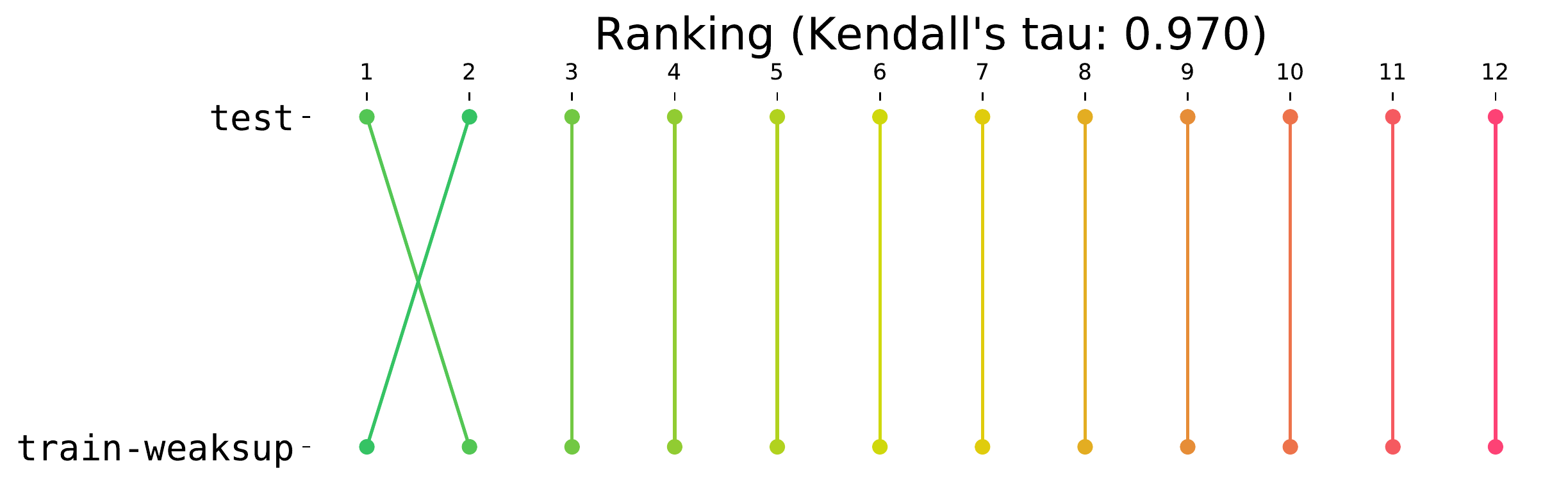}%
        \caption{\small ACoL}
        \vspace{2em}
    \end{subfigure}
    \begin{subfigure}[b]{.49\linewidth}
        \includegraphics[width=\linewidth]{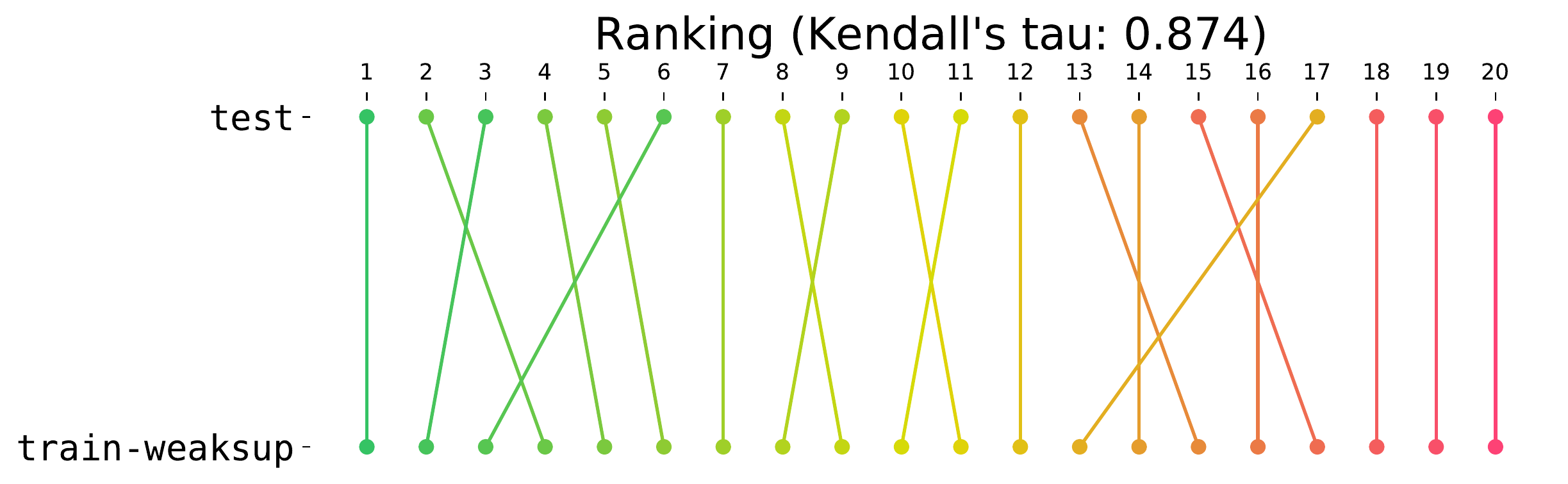}%
        \caption{\small SPG}
        \vspace{2em}
    \end{subfigure}
    \begin{subfigure}[b]{.49\linewidth}
        \includegraphics[width=\linewidth]{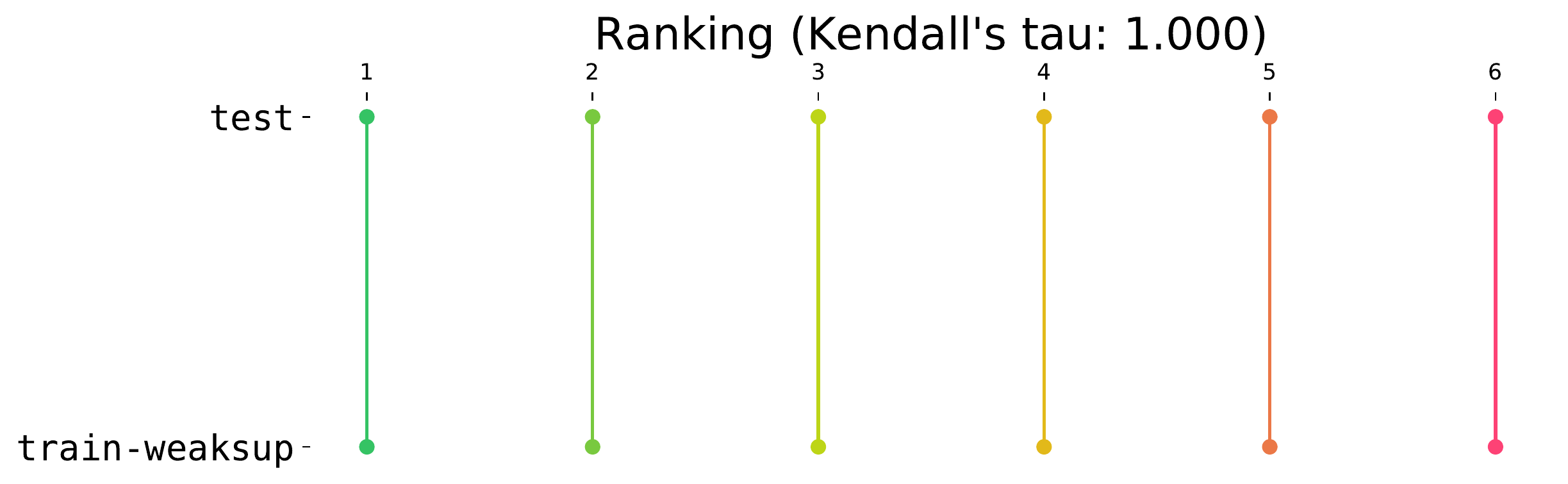}%
        \caption{\small ADL}
        \vspace{2em}
    \end{subfigure}
    \begin{subfigure}[b]{.49\linewidth}
        \includegraphics[width=\linewidth]{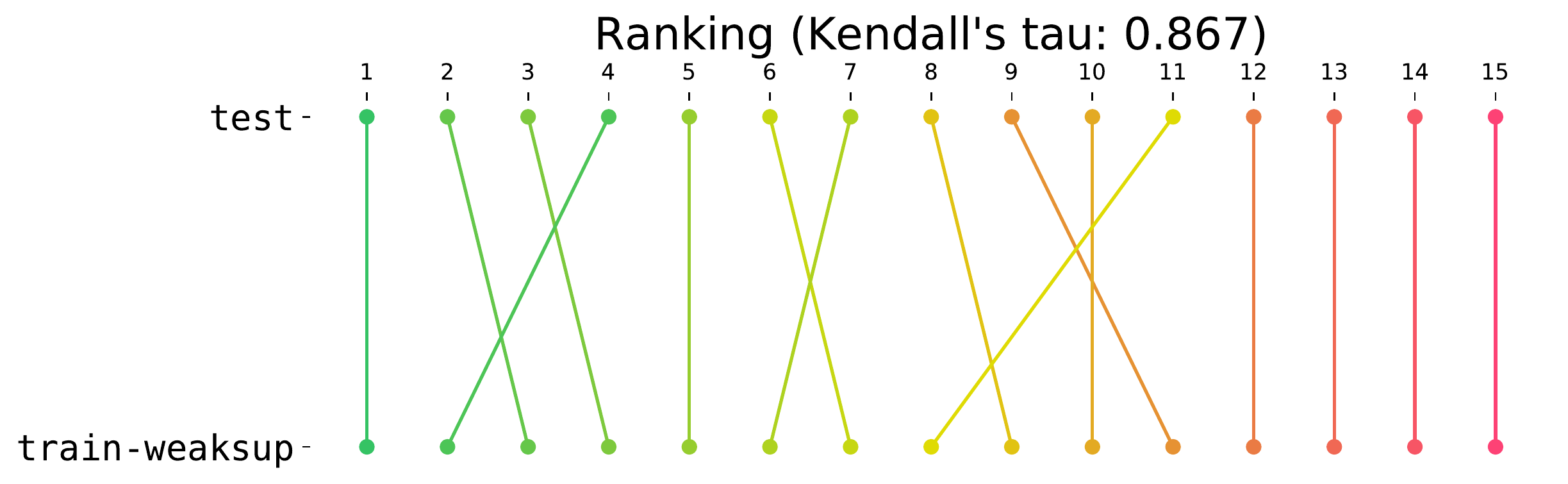}%
        \caption{\small CutMix}
        \vspace{2em}
    \end{subfigure}
    \caption{\small \textbf{Preservation of hyperparameter rankings.} Ranking of hyperparameters is largely preserved between \trainweaksup and \testfullsup. We only show the converged sessions (\S5.4). Results on ResNet50, OpenImages.}
    \label{fig:val-test-ranking_preservation}
\end{figure*}
\begin{figure*}
    \centering
    \begin{subfigure}[b]{\linewidth}
        \includegraphics[width=\linewidth]{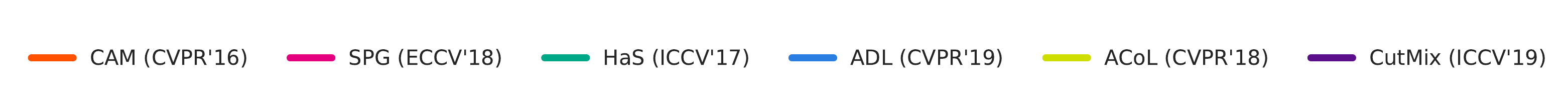}
    \end{subfigure}

    \begin{subfigure}[b]{.32\linewidth}
        \includegraphics[width=\linewidth]{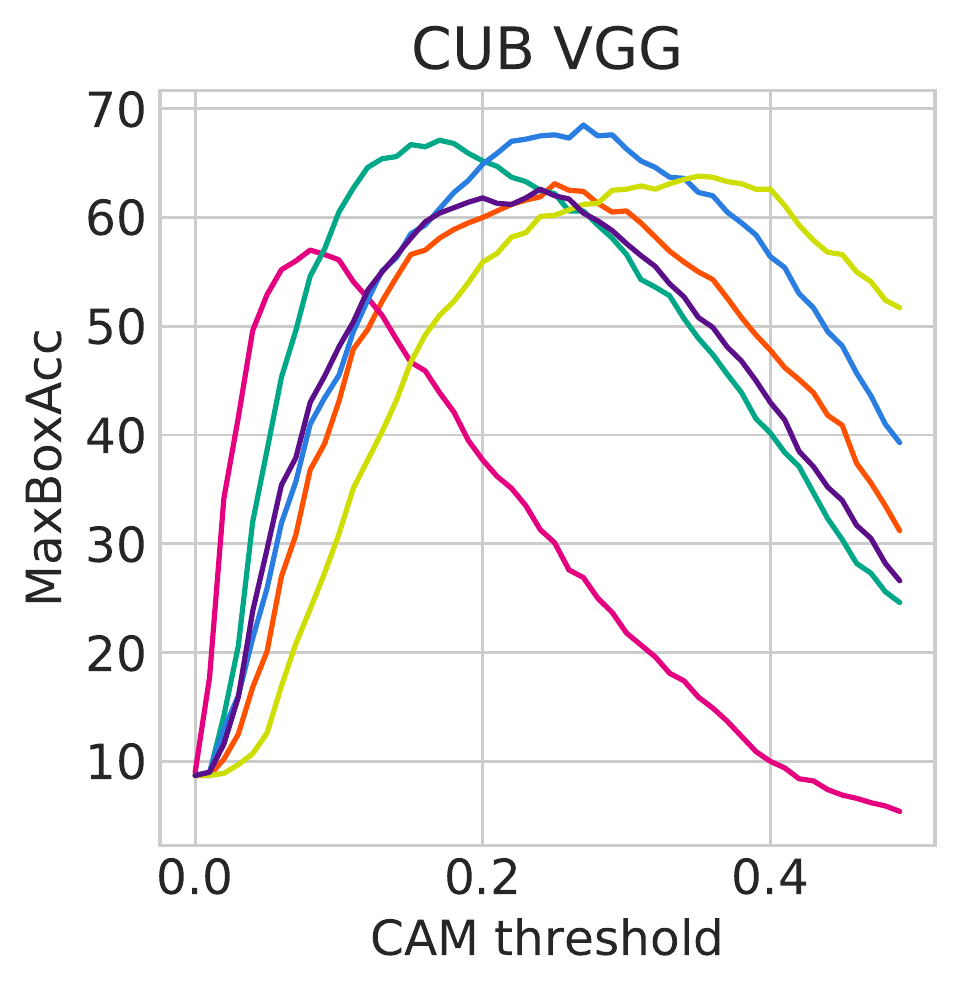}
    \end{subfigure}
    \begin{subfigure}[b]{.32\linewidth}
        \includegraphics[width=\linewidth]{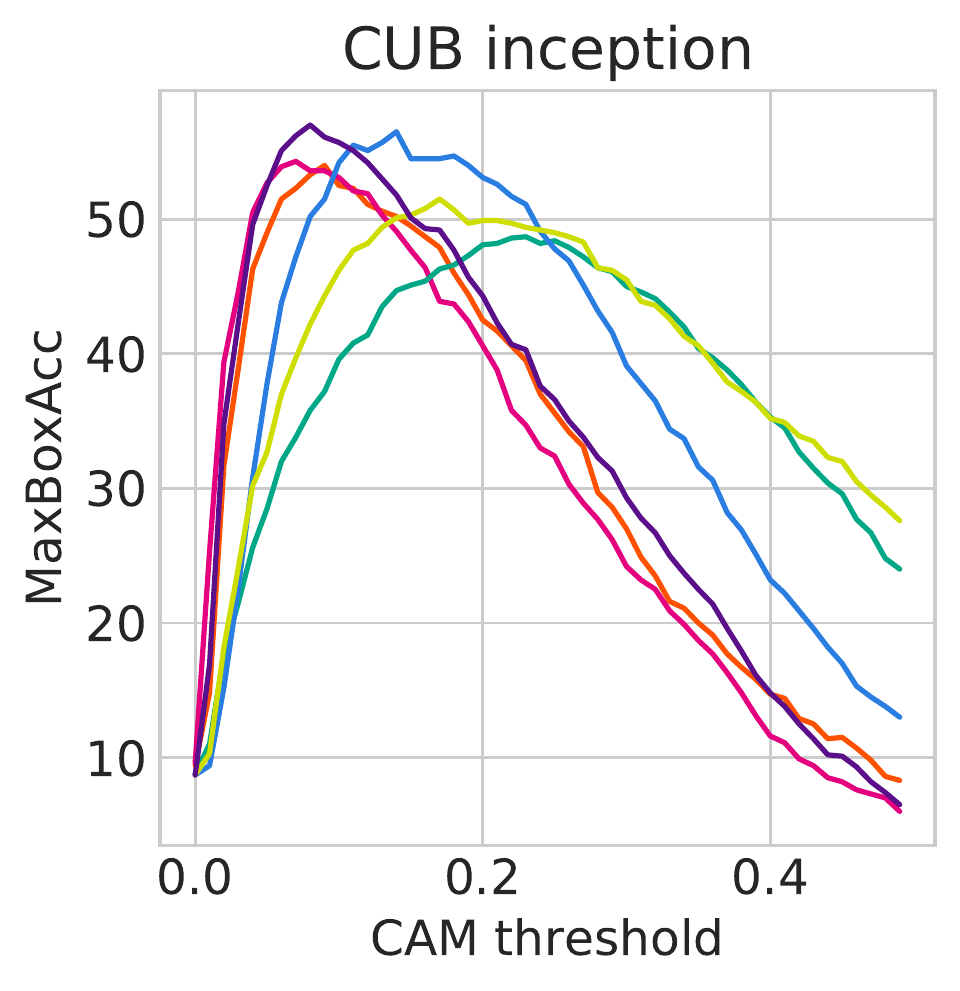}
    \end{subfigure}
    \begin{subfigure}[b]{.32\linewidth}
        \includegraphics[width=\linewidth]{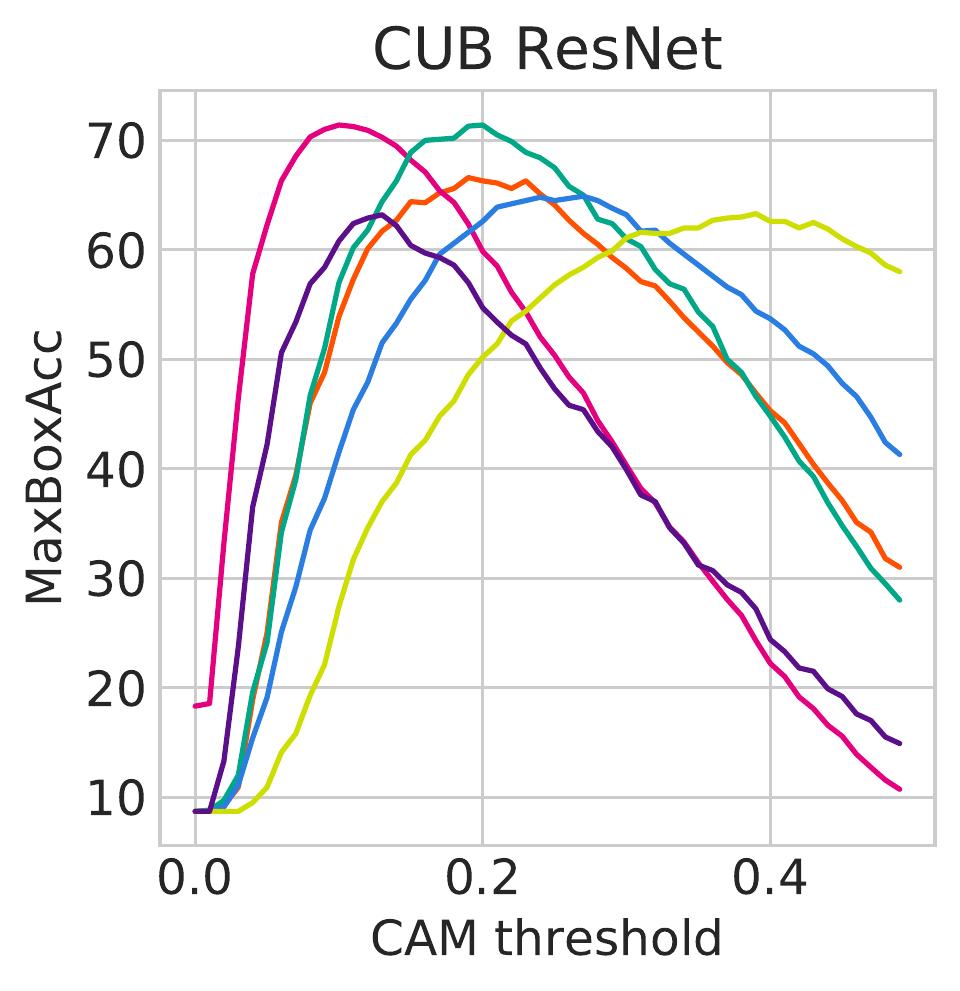}
    \end{subfigure}

    \begin{subfigure}[b]{.32\linewidth}
        \includegraphics[width=\linewidth]{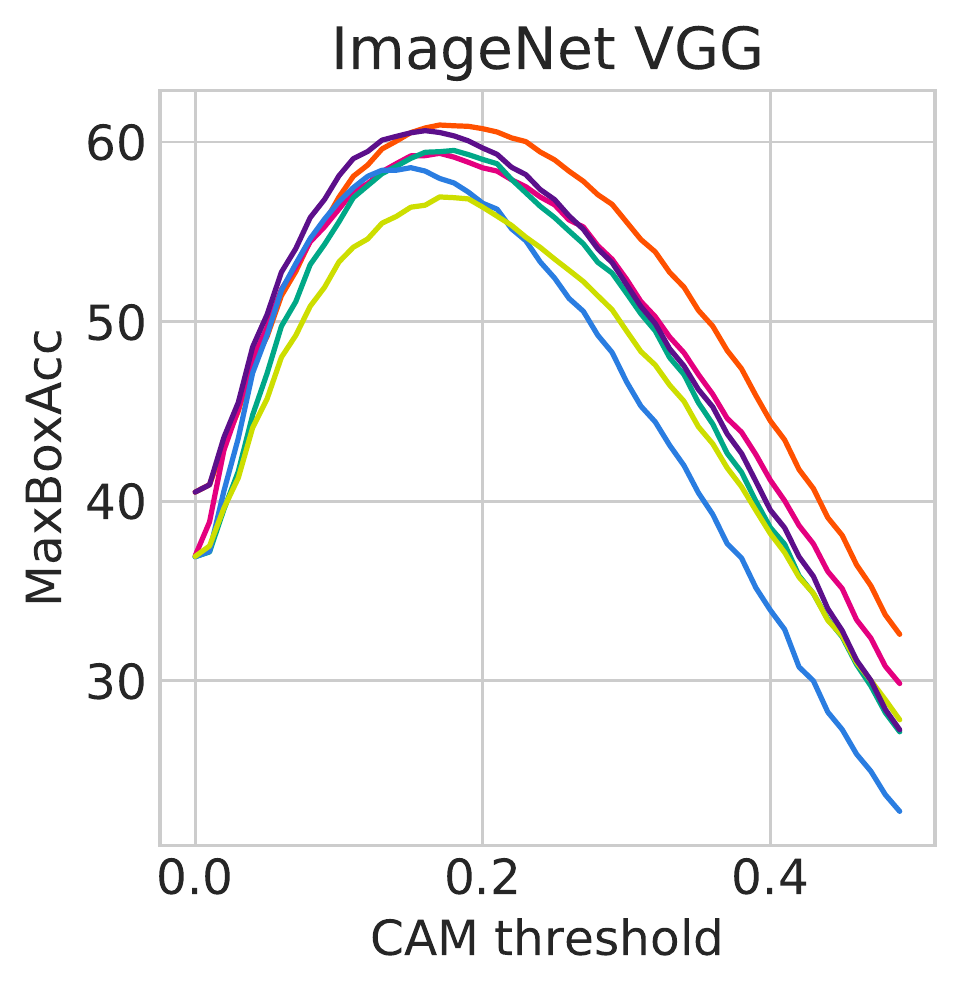}
    \end{subfigure}
    \begin{subfigure}[b]{.32\linewidth}
        \includegraphics[width=\linewidth]{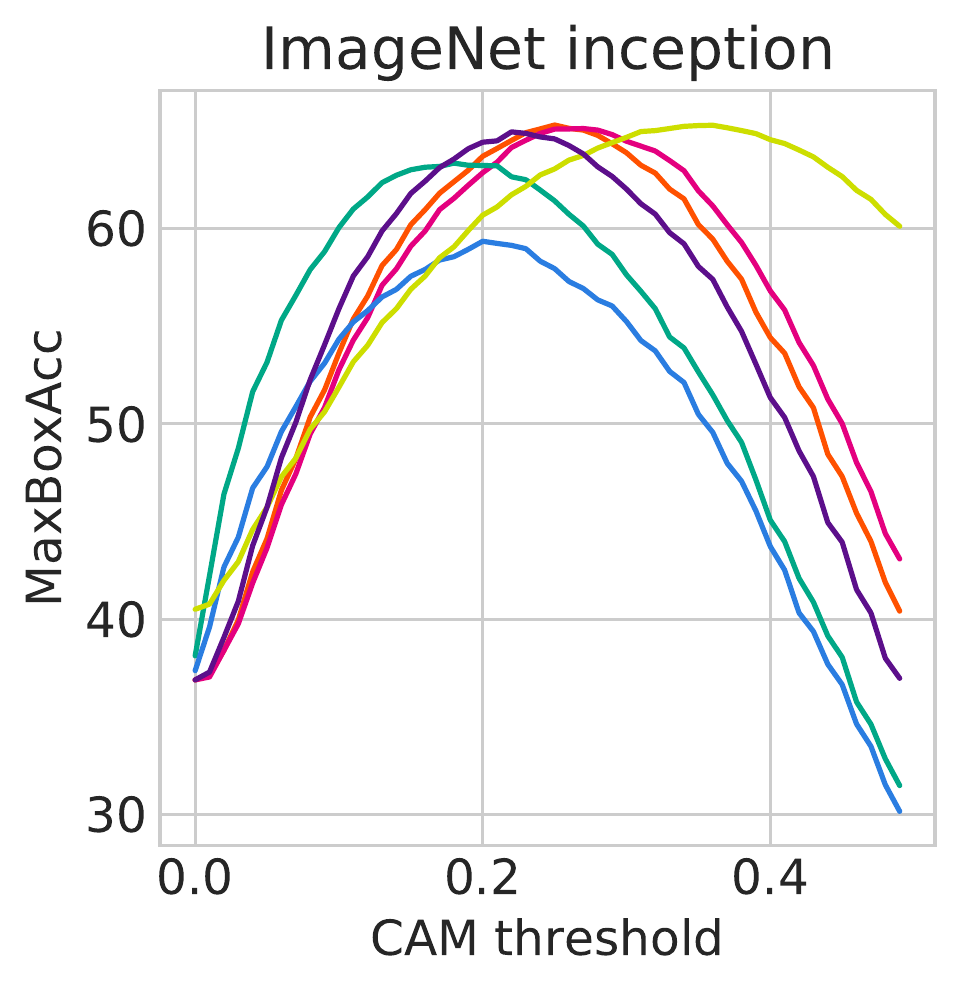}
    \end{subfigure}
    \begin{subfigure}[b]{.32\linewidth}
        \includegraphics[width=\linewidth]{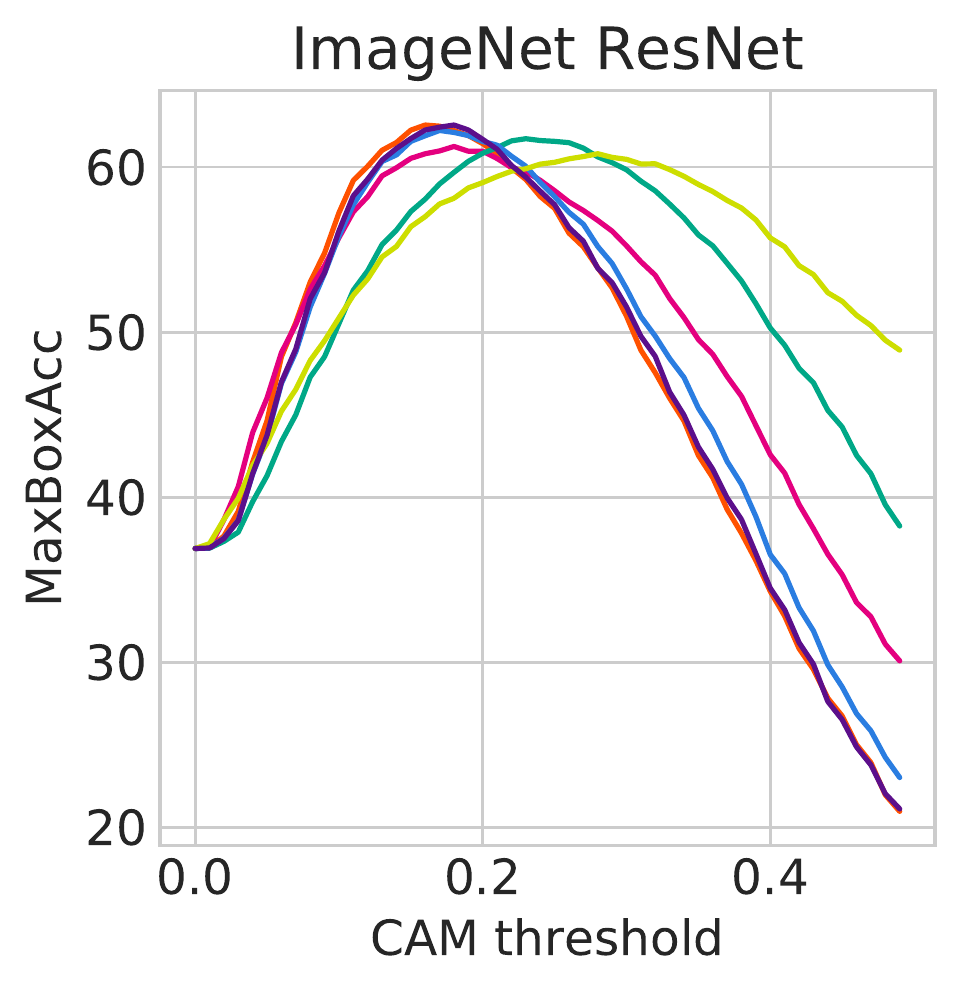}
    \end{subfigure}

    \begin{subfigure}[b]{.32\linewidth}
        \includegraphics[width=\linewidth]{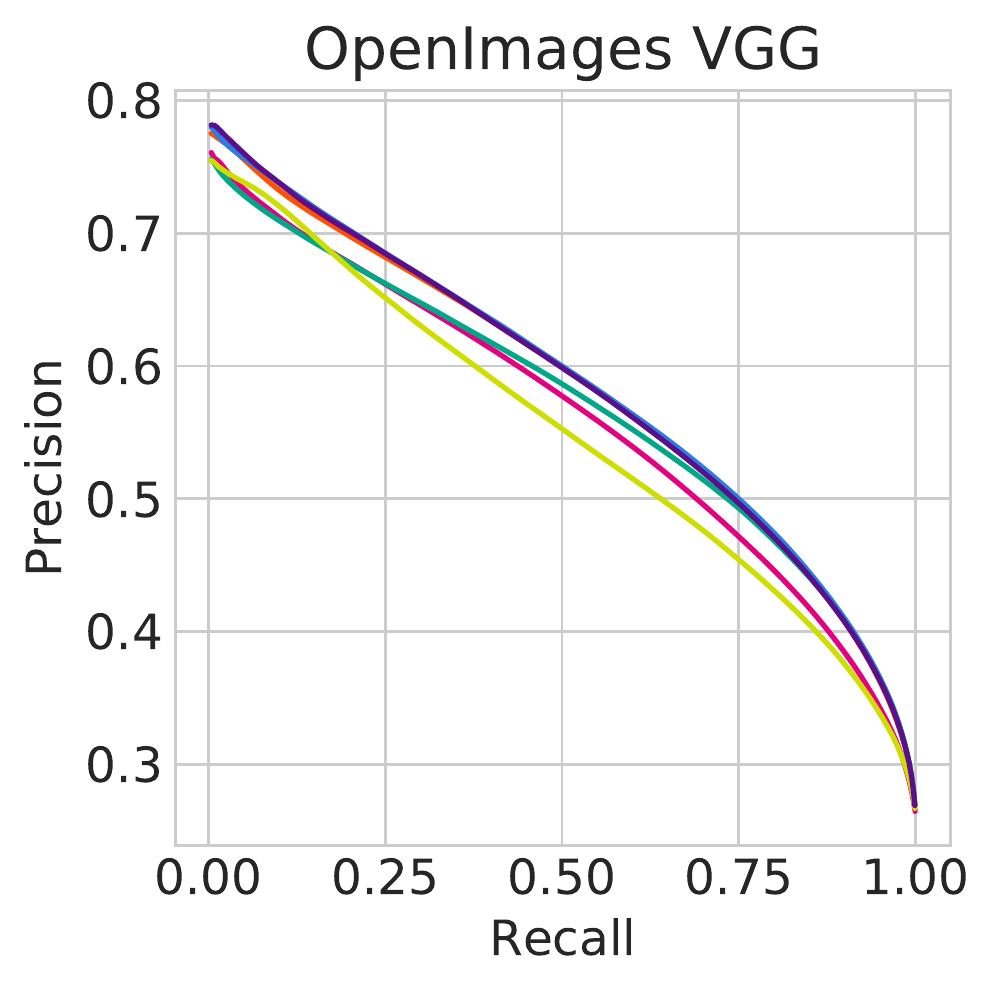}
    \end{subfigure}
    \begin{subfigure}[b]{.32\linewidth}
        \includegraphics[width=\linewidth]{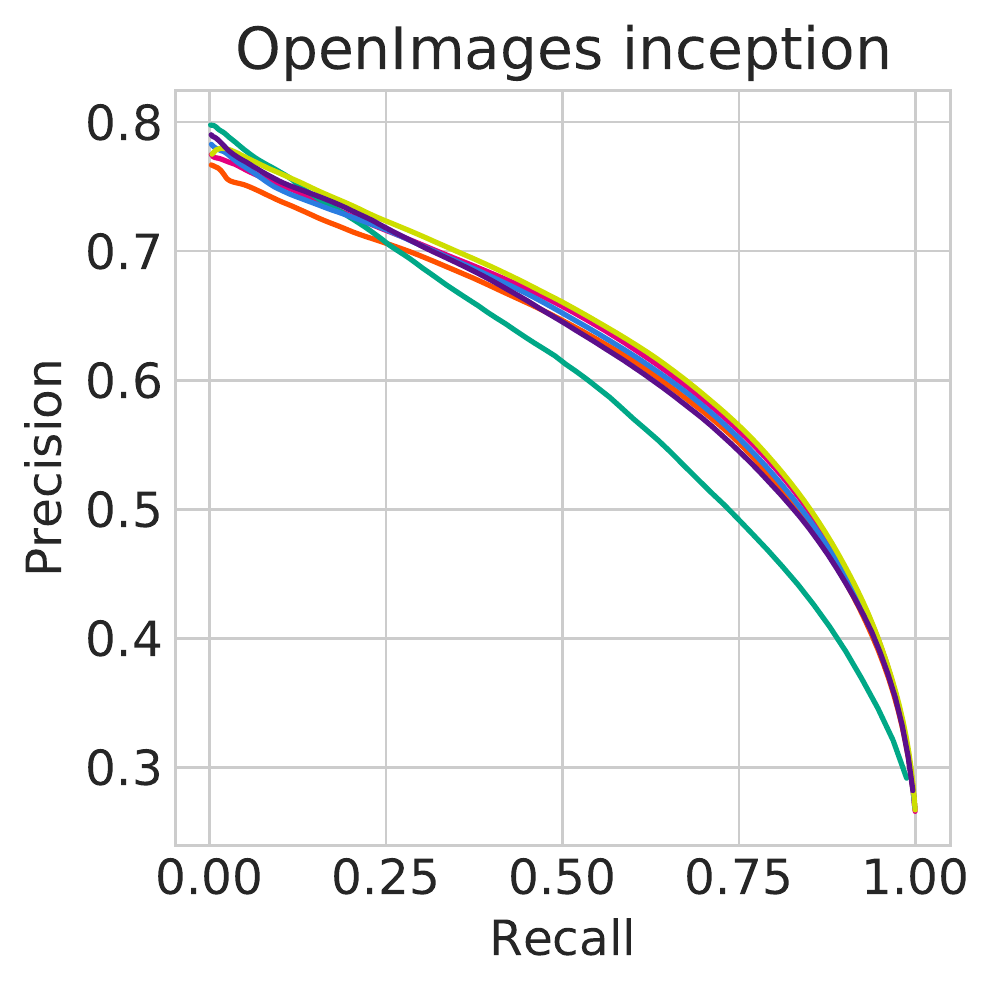}
    \end{subfigure}
    \begin{subfigure}[b]{.32\linewidth}
        \includegraphics[width=\linewidth]{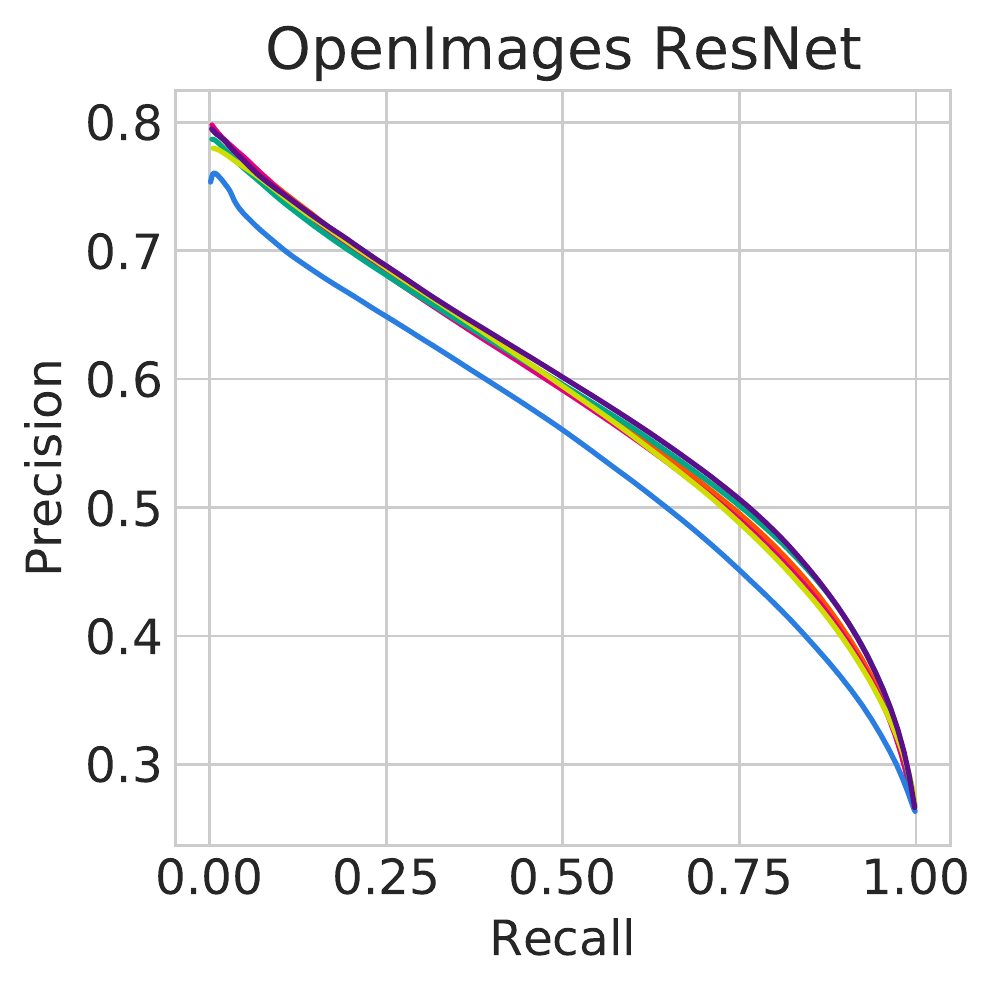}
    \end{subfigure}
    \caption{\small \textbf{Performance by operating threshold $\tau$.} CUB and ImageNet: \boxacc versus $\tau$, OpenImages: \pxprec versus \pxrec. ResNet, VGG, and Inception architecture results are used. This is the extension of Figure~5 in the main paper.}
    \label{fig:all_3_by_3_threshold_plots}
\end{figure*}
\begin{figure*}
    \centering
    \begin{subfigure}[b]{0.85\linewidth}
        \begin{subfigure}[b]{.16\linewidth}
            \includegraphics[width=\linewidth]{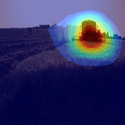}
        \end{subfigure}~\begin{subfigure}[b]{.16\linewidth}
            \includegraphics[width=\linewidth]{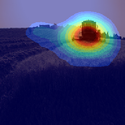}
        \end{subfigure}~\begin{subfigure}[b]{.16\linewidth}
            \includegraphics[width=\linewidth]{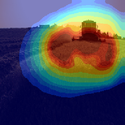}
        \end{subfigure}~\begin{subfigure}[b]{.16\linewidth}
            \includegraphics[width=\linewidth]{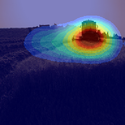}
        \end{subfigure}~\begin{subfigure}[b]{.16\linewidth}
            \includegraphics[width=\linewidth]{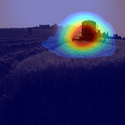}
        \end{subfigure}~\begin{subfigure}[b]{.16\linewidth}
            \includegraphics[width=\linewidth]{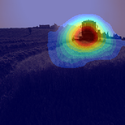}
        \end{subfigure}~\\
        \begin{subfigure}[b]{.16\linewidth}
            \includegraphics[width=\linewidth]{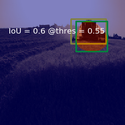}
        \end{subfigure}~\begin{subfigure}[b]{.16\linewidth}
            \includegraphics[width=\linewidth]{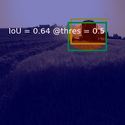}
        \end{subfigure}~\begin{subfigure}[b]{.16\linewidth}
            \includegraphics[width=\linewidth]{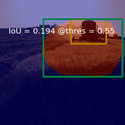}
        \end{subfigure}~\begin{subfigure}[b]{.16\linewidth}
            \includegraphics[width=\linewidth]{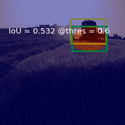}
        \end{subfigure}~\begin{subfigure}[b]{.16\linewidth}
            \includegraphics[width=\linewidth]{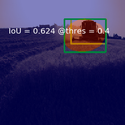}
        \end{subfigure}~\begin{subfigure}[b]{.16\linewidth}
            \includegraphics[width=\linewidth]{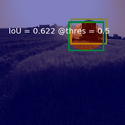}
        \end{subfigure}\caption{Label: Harvester}
    \end{subfigure}
    \vspace{1em}
    
    \begin{subfigure}[b]{0.85\linewidth}
        \begin{subfigure}[b]{.16\linewidth}
            \includegraphics[width=\linewidth]{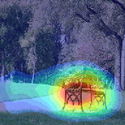}
        \end{subfigure}~\begin{subfigure}[b]{.16\linewidth}
            \includegraphics[width=\linewidth]{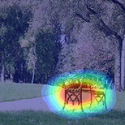}
        \end{subfigure}~\begin{subfigure}[b]{.16\linewidth}
            \includegraphics[width=\linewidth]{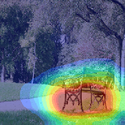}
        \end{subfigure}~\begin{subfigure}[b]{.16\linewidth}
            \includegraphics[width=\linewidth]{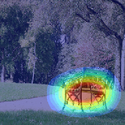}
        \end{subfigure}~\begin{subfigure}[b]{.16\linewidth}
            \includegraphics[width=\linewidth]{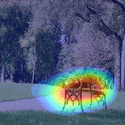}
        \end{subfigure}~\begin{subfigure}[b]{.16\linewidth}
            \includegraphics[width=\linewidth]{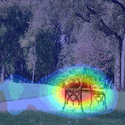}
        \end{subfigure}~\\
        \begin{subfigure}[b]{.16\linewidth}
            \includegraphics[width=\linewidth]{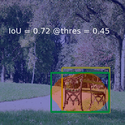}
        \end{subfigure}~\begin{subfigure}[b]{.16\linewidth}
            \includegraphics[width=\linewidth]{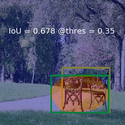}
        \end{subfigure}~\begin{subfigure}[b]{.16\linewidth}
            \includegraphics[width=\linewidth]{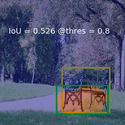}
        \end{subfigure}~\begin{subfigure}[b]{.16\linewidth}
            \includegraphics[width=\linewidth]{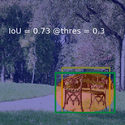}
        \end{subfigure}~\begin{subfigure}[b]{.16\linewidth}
            \includegraphics[width=\linewidth]{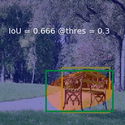}
        \end{subfigure}~\begin{subfigure}[b]{.16\linewidth}
            \includegraphics[width=\linewidth]{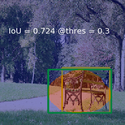}
        \end{subfigure}\caption{Label: Park bench}
    \end{subfigure}
\vspace{1em}

    \begin{subfigure}[b]{0.85\linewidth}
        \begin{subfigure}[b]{.16\linewidth}
            \includegraphics[width=\linewidth]{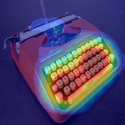}
        \end{subfigure}~\begin{subfigure}[b]{.16\linewidth}
            \includegraphics[width=\linewidth]{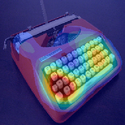}
        \end{subfigure}~\begin{subfigure}[b]{.16\linewidth}
            \includegraphics[width=\linewidth]{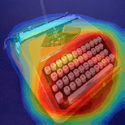}
        \end{subfigure}~\begin{subfigure}[b]{.16\linewidth}
            \includegraphics[width=\linewidth]{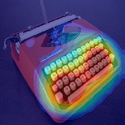}
        \end{subfigure}~\begin{subfigure}[b]{.16\linewidth}
            \includegraphics[width=\linewidth]{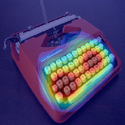}
        \end{subfigure}~\begin{subfigure}[b]{.16\linewidth}
            \includegraphics[width=\linewidth]{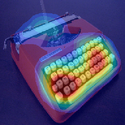}
        \end{subfigure}~\\
        \begin{subfigure}[b]{.16\linewidth}
            \includegraphics[width=\linewidth]{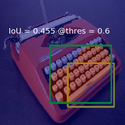}
        \end{subfigure}~\begin{subfigure}[b]{.16\linewidth}
            \includegraphics[width=\linewidth]{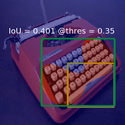}
        \end{subfigure}~\begin{subfigure}[b]{.16\linewidth}
            \includegraphics[width=\linewidth]{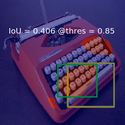}
        \end{subfigure}~\begin{subfigure}[b]{.16\linewidth}
            \includegraphics[width=\linewidth]{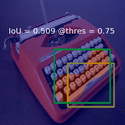}
        \end{subfigure}~\begin{subfigure}[b]{.16\linewidth}
            \includegraphics[width=\linewidth]{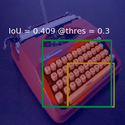}
        \end{subfigure}~\begin{subfigure}[b]{.16\linewidth}
            \includegraphics[width=\linewidth]{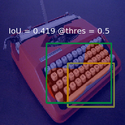}
        \end{subfigure}\caption{Label: Space bar}
    \end{subfigure}
\vspace{1em}

    \begin{subfigure}[b]{0.85\linewidth}
        \begin{subfigure}[b]{.16\linewidth}
            \includegraphics[width=\linewidth]{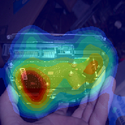}
        \end{subfigure}~\begin{subfigure}[b]{.16\linewidth}
            \includegraphics[width=\linewidth]{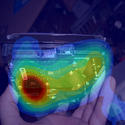}
        \end{subfigure}~\begin{subfigure}[b]{.16\linewidth}
            \includegraphics[width=\linewidth]{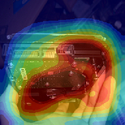}
        \end{subfigure}~\begin{subfigure}[b]{.16\linewidth}
            \includegraphics[width=\linewidth]{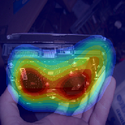}
        \end{subfigure}~\begin{subfigure}[b]{.16\linewidth}
            \includegraphics[width=\linewidth]{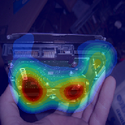}
        \end{subfigure}~\begin{subfigure}[b]{.16\linewidth}
            \includegraphics[width=\linewidth]{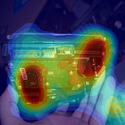}
        \end{subfigure}~\\
        \begin{subfigure}[b]{.16\linewidth}
            \includegraphics[width=\linewidth]{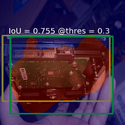}
        \end{subfigure}~\begin{subfigure}[b]{.16\linewidth}
            \includegraphics[width=\linewidth]{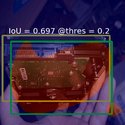}
        \end{subfigure}~\begin{subfigure}[b]{.16\linewidth}
            \includegraphics[width=\linewidth]{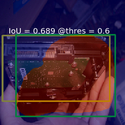}
        \end{subfigure}~\begin{subfigure}[b]{.16\linewidth}
            \includegraphics[width=\linewidth]{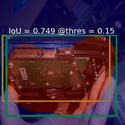}
        \end{subfigure}~\begin{subfigure}[b]{.16\linewidth}
            \includegraphics[width=\linewidth]{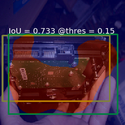}
        \end{subfigure}~\begin{subfigure}[b]{.16\linewidth}
            \includegraphics[width=\linewidth]{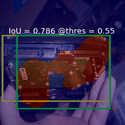}
        \end{subfigure}\caption{Label: Hard disk}
    \end{subfigure}
\vspace{1em}

    \caption{\small \textbf{ImageNet score maps.} Score maps of CAM, HaS, ACoL, SPG, ADL, CutMix from ImageNet.}
    \label{fig:score_map_visualization_imagenet}
\end{figure*}

\begin{figure*}
    \centering
    \begin{subfigure}[b]{0.85\linewidth}
        \begin{subfigure}[b]{.16\linewidth}
            \includegraphics[width=\linewidth]{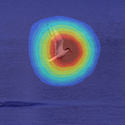}
        \end{subfigure}~\begin{subfigure}[b]{.16\linewidth}
            \includegraphics[width=\linewidth]{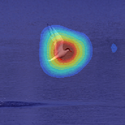}
        \end{subfigure}~\begin{subfigure}[b]{.16\linewidth}
            \includegraphics[width=\linewidth]{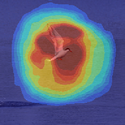}
        \end{subfigure}~\begin{subfigure}[b]{.16\linewidth}
            \includegraphics[width=\linewidth]{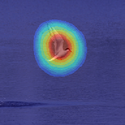}
        \end{subfigure}~\begin{subfigure}[b]{.16\linewidth}
            \includegraphics[width=\linewidth]{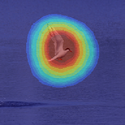}
        \end{subfigure}~\begin{subfigure}[b]{.16\linewidth}
            \includegraphics[width=\linewidth]{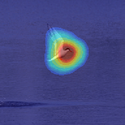}
        \end{subfigure}~\\
        \begin{subfigure}[b]{.16\linewidth}
            \includegraphics[width=\linewidth]{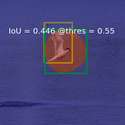}
        \end{subfigure}~\begin{subfigure}[b]{.16\linewidth}
            \includegraphics[width=\linewidth]{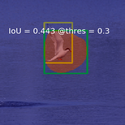}
        \end{subfigure}~\begin{subfigure}[b]{.16\linewidth}
            \includegraphics[width=\linewidth]{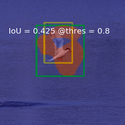}
        \end{subfigure}~\begin{subfigure}[b]{.16\linewidth}
            \includegraphics[width=\linewidth]{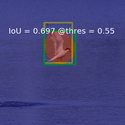}
        \end{subfigure}~\begin{subfigure}[b]{.16\linewidth}
            \includegraphics[width=\linewidth]{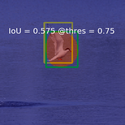}
        \end{subfigure}~\begin{subfigure}[b]{.16\linewidth}
            \includegraphics[width=\linewidth]{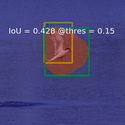}
        \end{subfigure}\caption{Label: Common Tern}
    \end{subfigure}
\vspace{1em}

    \begin{subfigure}[b]{0.85\linewidth}
        \begin{subfigure}[b]{.16\linewidth}
            \includegraphics[width=\linewidth]{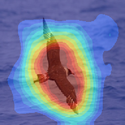}
        \end{subfigure}~\begin{subfigure}[b]{.16\linewidth}
            \includegraphics[width=\linewidth]{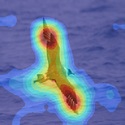}
        \end{subfigure}~\begin{subfigure}[b]{.16\linewidth}
            \includegraphics[width=\linewidth]{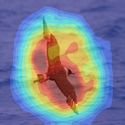}
        \end{subfigure}~\begin{subfigure}[b]{.16\linewidth}
            \includegraphics[width=\linewidth]{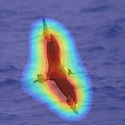}
        \end{subfigure}~\begin{subfigure}[b]{.16\linewidth}
            \includegraphics[width=\linewidth]{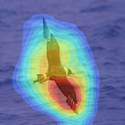}
        \end{subfigure}~\begin{subfigure}[b]{.16\linewidth}
            \includegraphics[width=\linewidth]{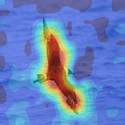}
        \end{subfigure}~\\
        \begin{subfigure}[b]{.16\linewidth}
            \includegraphics[width=\linewidth]{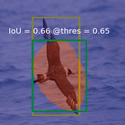}
        \end{subfigure}~\begin{subfigure}[b]{.16\linewidth}
            \includegraphics[width=\linewidth]{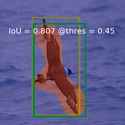}
        \end{subfigure}~\begin{subfigure}[b]{.16\linewidth}
            \includegraphics[width=\linewidth]{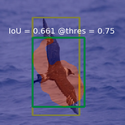}
        \end{subfigure}~\begin{subfigure}[b]{.16\linewidth}
            \includegraphics[width=\linewidth]{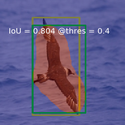}
        \end{subfigure}~\begin{subfigure}[b]{.16\linewidth}
            \includegraphics[width=\linewidth]{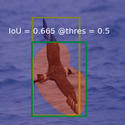}
        \end{subfigure}~\begin{subfigure}[b]{.16\linewidth}
            \includegraphics[width=\linewidth]{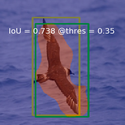}
        \end{subfigure}\caption{Label: Pomarine Jaeger}
    \end{subfigure}
\vspace{1em}

    \begin{subfigure}[b]{0.85\linewidth}
        \begin{subfigure}[b]{.16\linewidth}
            \includegraphics[width=\linewidth]{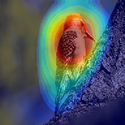}
        \end{subfigure}~\begin{subfigure}[b]{.16\linewidth}
            \includegraphics[width=\linewidth]{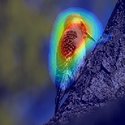}
        \end{subfigure}~\begin{subfigure}[b]{.16\linewidth}
            \includegraphics[width=\linewidth]{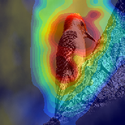}
        \end{subfigure}~\begin{subfigure}[b]{.16\linewidth}
            \includegraphics[width=\linewidth]{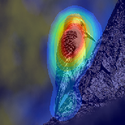}
        \end{subfigure}~\begin{subfigure}[b]{.16\linewidth}
            \includegraphics[width=\linewidth]{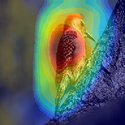}
        \end{subfigure}~\begin{subfigure}[b]{.16\linewidth}
            \includegraphics[width=\linewidth]{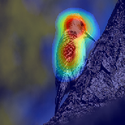}
        \end{subfigure}~\\
        \begin{subfigure}[b]{.16\linewidth}
            \includegraphics[width=\linewidth]{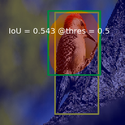}
        \end{subfigure}~\begin{subfigure}[b]{.16\linewidth}
            \includegraphics[width=\linewidth]{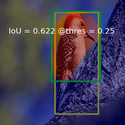}
        \end{subfigure}~\begin{subfigure}[b]{.16\linewidth}
            \includegraphics[width=\linewidth]{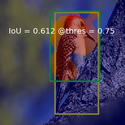}
        \end{subfigure}~\begin{subfigure}[b]{.16\linewidth}
            \includegraphics[width=\linewidth]{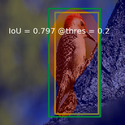}
        \end{subfigure}~\begin{subfigure}[b]{.16\linewidth}
            \includegraphics[width=\linewidth]{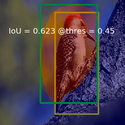}
        \end{subfigure}~\begin{subfigure}[b]{.16\linewidth}
            \includegraphics[width=\linewidth]{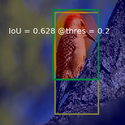}
        \end{subfigure}\caption{Label: Red bellied Woodpecker}
    \end{subfigure}
\vspace{1em}

    \begin{subfigure}[b]{0.85\linewidth}
        \begin{subfigure}[b]{.16\linewidth}
            \includegraphics[width=\linewidth]{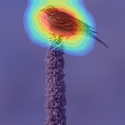}
        \end{subfigure}~\begin{subfigure}[b]{.16\linewidth}
            \includegraphics[width=\linewidth]{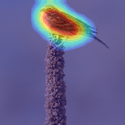}
        \end{subfigure}~\begin{subfigure}[b]{.16\linewidth}
            \includegraphics[width=\linewidth]{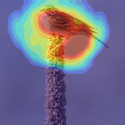}
        \end{subfigure}~\begin{subfigure}[b]{.16\linewidth}
            \includegraphics[width=\linewidth]{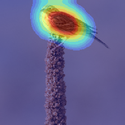}
        \end{subfigure}~\begin{subfigure}[b]{.16\linewidth}
            \includegraphics[width=\linewidth]{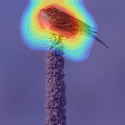}
        \end{subfigure}~\begin{subfigure}[b]{.16\linewidth}
            \includegraphics[width=\linewidth]{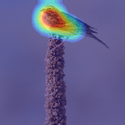}
        \end{subfigure}~\\
        \begin{subfigure}[b]{.16\linewidth}
            \includegraphics[width=\linewidth]{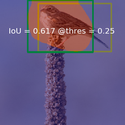}
        \end{subfigure}~\begin{subfigure}[b]{.16\linewidth}
            \includegraphics[width=\linewidth]{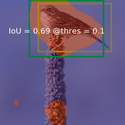}
        \end{subfigure}~\begin{subfigure}[b]{.16\linewidth}
            \includegraphics[width=\linewidth]{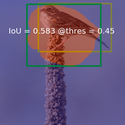}
        \end{subfigure}~\begin{subfigure}[b]{.16\linewidth}
            \includegraphics[width=\linewidth]{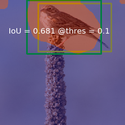}
        \end{subfigure}~\begin{subfigure}[b]{.16\linewidth}
            \includegraphics[width=\linewidth]{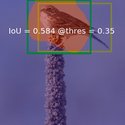}
        \end{subfigure}~\begin{subfigure}[b]{.16\linewidth}
            \includegraphics[width=\linewidth]{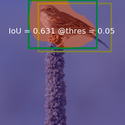}
        \end{subfigure}\caption{Label: Vesper Sparrow}
    \end{subfigure}
\vspace{1em}

    \caption{\small \textbf{CUB score maps.} Score maps of CAM, HaS, ACoL, SPG, ADL, CutMix from CUB.}
    \label{fig:score_map_visualization_cub}
\end{figure*}

\begin{figure*}
    \centering
    \begin{subfigure}[b]{0.85\linewidth}
        \begin{subfigure}[b]{.16\linewidth}
            \includegraphics[width=\linewidth]{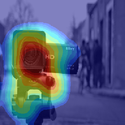}
        \end{subfigure}~\begin{subfigure}[b]{.16\linewidth}
            \includegraphics[width=\linewidth]{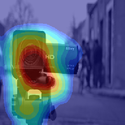}
        \end{subfigure}~\begin{subfigure}[b]{.16\linewidth}
            \includegraphics[width=\linewidth]{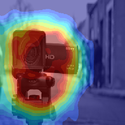}
        \end{subfigure}~\begin{subfigure}[b]{.16\linewidth}
            \includegraphics[width=\linewidth]{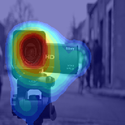}
        \end{subfigure}~\begin{subfigure}[b]{.16\linewidth}
            \includegraphics[width=\linewidth]{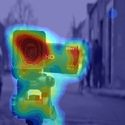}
        \end{subfigure}~\begin{subfigure}[b]{.16\linewidth}
            \includegraphics[width=\linewidth]{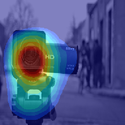}
        \end{subfigure}~\\
        \begin{subfigure}[b]{.16\linewidth}
            \includegraphics[width=\linewidth]{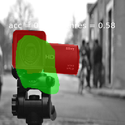}
        \end{subfigure}~\begin{subfigure}[b]{.16\linewidth}
            \includegraphics[width=\linewidth]{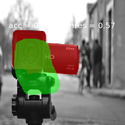}
        \end{subfigure}~\begin{subfigure}[b]{.16\linewidth}
            \includegraphics[width=\linewidth]{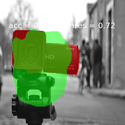}
        \end{subfigure}~\begin{subfigure}[b]{.16\linewidth}
            \includegraphics[width=\linewidth]{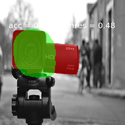}
        \end{subfigure}~\begin{subfigure}[b]{.16\linewidth}
            \includegraphics[width=\linewidth]{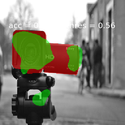}
        \end{subfigure}~\begin{subfigure}[b]{.16\linewidth}
            \includegraphics[width=\linewidth]{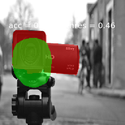}
        \end{subfigure}\caption{Label: Camera}
    \end{subfigure}
\vspace{1em}
    
    \begin{subfigure}[b]{0.85\linewidth}
        \begin{subfigure}[b]{.16\linewidth}
            \includegraphics[width=\linewidth]{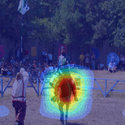}
        \end{subfigure}~\begin{subfigure}[b]{.16\linewidth}
            \includegraphics[width=\linewidth]{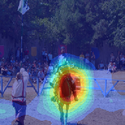}
        \end{subfigure}~\begin{subfigure}[b]{.16\linewidth}
            \includegraphics[width=\linewidth]{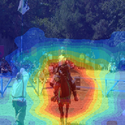}
        \end{subfigure}~\begin{subfigure}[b]{.16\linewidth}
            \includegraphics[width=\linewidth]{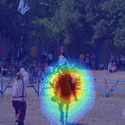}
        \end{subfigure}~\begin{subfigure}[b]{.16\linewidth}
            \includegraphics[width=\linewidth]{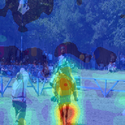}
        \end{subfigure}~\begin{subfigure}[b]{.16\linewidth}
            \includegraphics[width=\linewidth]{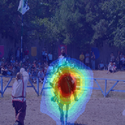}
        \end{subfigure}~\\
        \begin{subfigure}[b]{.16\linewidth}
            \includegraphics[width=\linewidth]{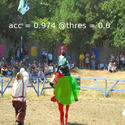}
        \end{subfigure}~\begin{subfigure}[b]{.16\linewidth}
            \includegraphics[width=\linewidth]{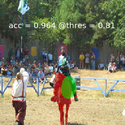}
        \end{subfigure}~\begin{subfigure}[b]{.16\linewidth}
            \includegraphics[width=\linewidth]{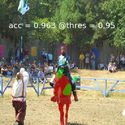}
        \end{subfigure}~\begin{subfigure}[b]{.16\linewidth}
            \includegraphics[width=\linewidth]{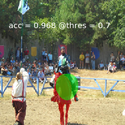}
        \end{subfigure}~\begin{subfigure}[b]{.16\linewidth}
            \includegraphics[width=\linewidth]{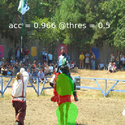}
        \end{subfigure}~\begin{subfigure}[b]{.16\linewidth}
            \includegraphics[width=\linewidth]{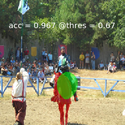}
        \end{subfigure}\caption{Label: Horse}
    \end{subfigure}
\vspace{1em}

    \begin{subfigure}[b]{0.85\linewidth}
        \begin{subfigure}[b]{.16\linewidth}
            \includegraphics[width=\linewidth]{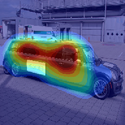}
        \end{subfigure}~\begin{subfigure}[b]{.16\linewidth}
            \includegraphics[width=\linewidth]{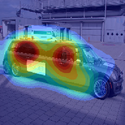}
        \end{subfigure}~\begin{subfigure}[b]{.16\linewidth}
            \includegraphics[width=\linewidth]{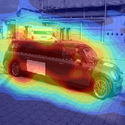}
        \end{subfigure}~\begin{subfigure}[b]{.16\linewidth}
            \includegraphics[width=\linewidth]{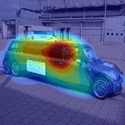}
        \end{subfigure}~\begin{subfigure}[b]{.16\linewidth}
            \includegraphics[width=\linewidth]{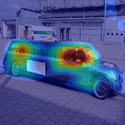}
        \end{subfigure}~\begin{subfigure}[b]{.16\linewidth}
            \includegraphics[width=\linewidth]{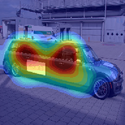}
        \end{subfigure}~\\
        \begin{subfigure}[b]{.16\linewidth}
            \includegraphics[width=\linewidth]{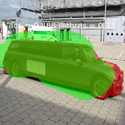}
        \end{subfigure}~\begin{subfigure}[b]{.16\linewidth}
            \includegraphics[width=\linewidth]{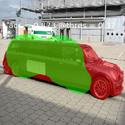}
        \end{subfigure}~\begin{subfigure}[b]{.16\linewidth}
            \includegraphics[width=\linewidth]{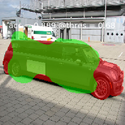}
        \end{subfigure}~\begin{subfigure}[b]{.16\linewidth}
            \includegraphics[width=\linewidth]{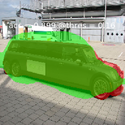}
        \end{subfigure}~\begin{subfigure}[b]{.16\linewidth}
            \includegraphics[width=\linewidth]{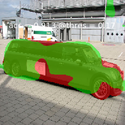}
        \end{subfigure}~\begin{subfigure}[b]{.16\linewidth}
            \includegraphics[width=\linewidth]{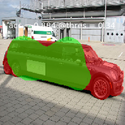}
        \end{subfigure}\caption{Label: Limousine}
    \end{subfigure}
\vspace{1em}
    
    \begin{subfigure}[b]{0.85\linewidth}
        \begin{subfigure}[b]{.16\linewidth}
            \includegraphics[width=\linewidth]{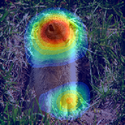}
        \end{subfigure}~\begin{subfigure}[b]{.16\linewidth}
            \includegraphics[width=\linewidth]{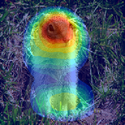}
        \end{subfigure}~\begin{subfigure}[b]{.16\linewidth}
            \includegraphics[width=\linewidth]{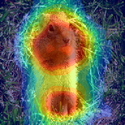}
        \end{subfigure}~\begin{subfigure}[b]{.16\linewidth}
            \includegraphics[width=\linewidth]{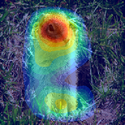}
        \end{subfigure}~\begin{subfigure}[b]{.16\linewidth}
            \includegraphics[width=\linewidth]{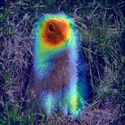}
        \end{subfigure}~\begin{subfigure}[b]{.16\linewidth}
            \includegraphics[width=\linewidth]{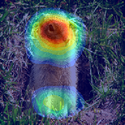}
        \end{subfigure}~\\
        \begin{subfigure}[b]{.16\linewidth}
            \includegraphics[width=\linewidth]{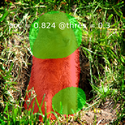}
        \end{subfigure}~\begin{subfigure}[b]{.16\linewidth}
            \includegraphics[width=\linewidth]{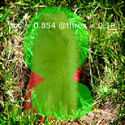}
        \end{subfigure}~\begin{subfigure}[b]{.16\linewidth}
            \includegraphics[width=\linewidth]{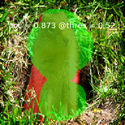}
        \end{subfigure}~\begin{subfigure}[b]{.16\linewidth}
            \includegraphics[width=\linewidth]{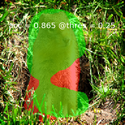}
        \end{subfigure}~\begin{subfigure}[b]{.16\linewidth}
            \includegraphics[width=\linewidth]{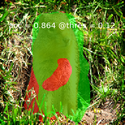}
        \end{subfigure}~\begin{subfigure}[b]{.16\linewidth}
            \includegraphics[width=\linewidth]{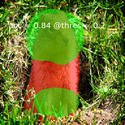}
        \end{subfigure}\caption{Label: Squirrel}
    \end{subfigure}
\vspace{1em}

    \caption{\small \textbf{OpenImages score maps.} Score maps of CAM, HaS, ACoL, SPG, ADL, CutMix from OpenImages.}
    \label{fig:score_map_visualization_openimages}
\end{figure*}

\newcommand\appendixviolinwidth{.33}

\begin{figure*}
    \centering

    \begin{subfigure}[b]{\appendixviolinwidth\linewidth}
        \includegraphics[width=\linewidth]{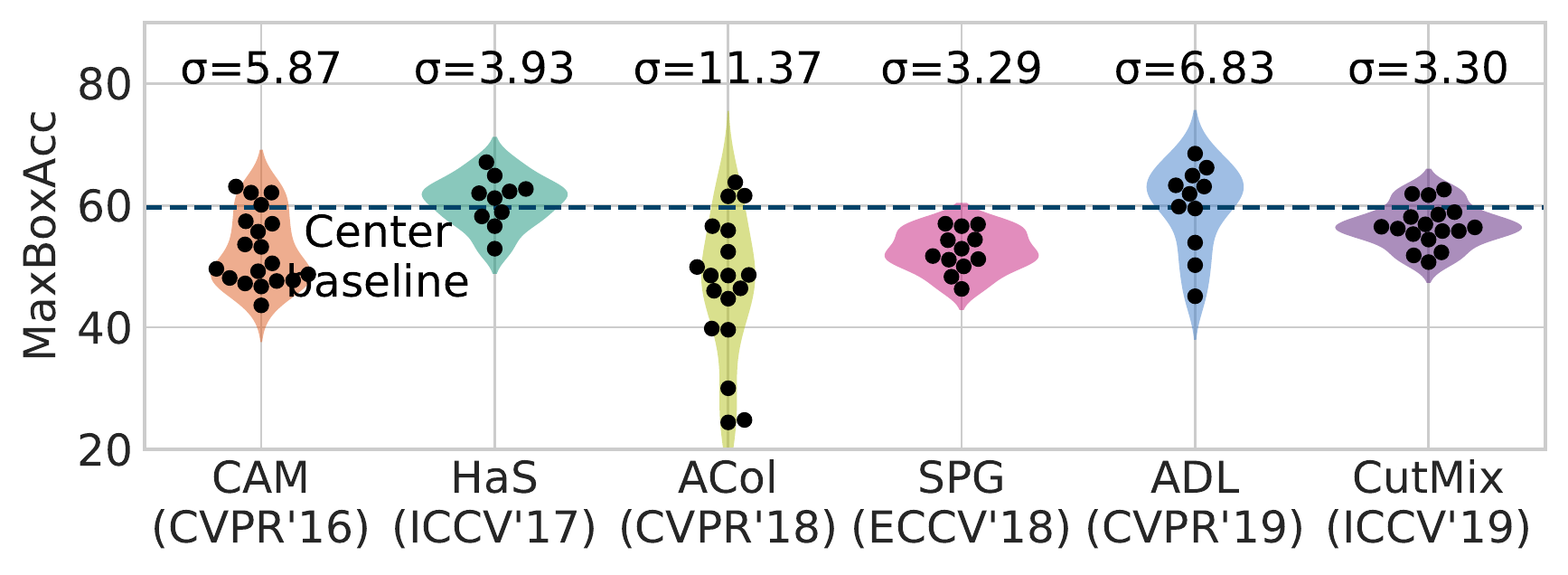}
        \caption{CUB, VGG}
    \end{subfigure}
    \begin{subfigure}[b]{\appendixviolinwidth\linewidth}
        \includegraphics[width=\linewidth]{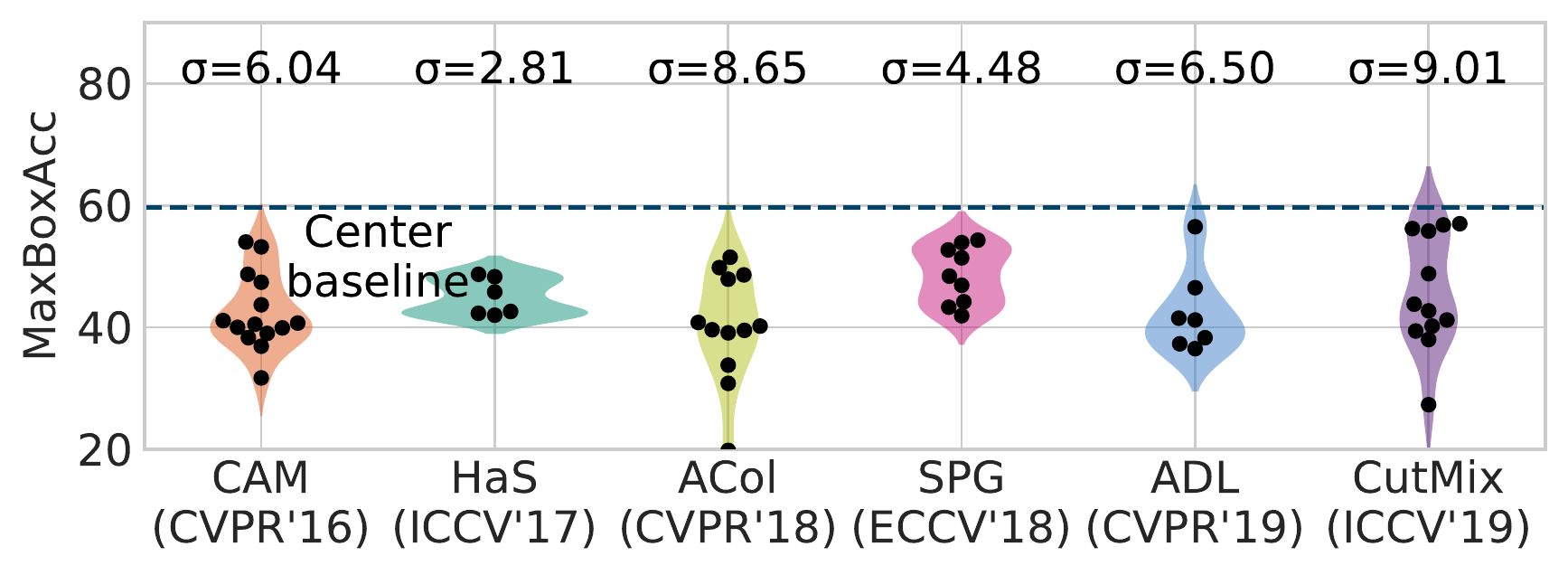}
        \caption{CUB, Inception}
    \end{subfigure}
    \begin{subfigure}[b]{\appendixviolinwidth\linewidth}
        \includegraphics[width=\linewidth]{figures/hyperparam_robustness_CUB_ResNet_reduced2.pdf}
        \caption{CUB, ResNet}
    \end{subfigure}
\vspace{1em}
    
    \begin{subfigure}[b]{\appendixviolinwidth\linewidth}
    \includegraphics[width=\linewidth]{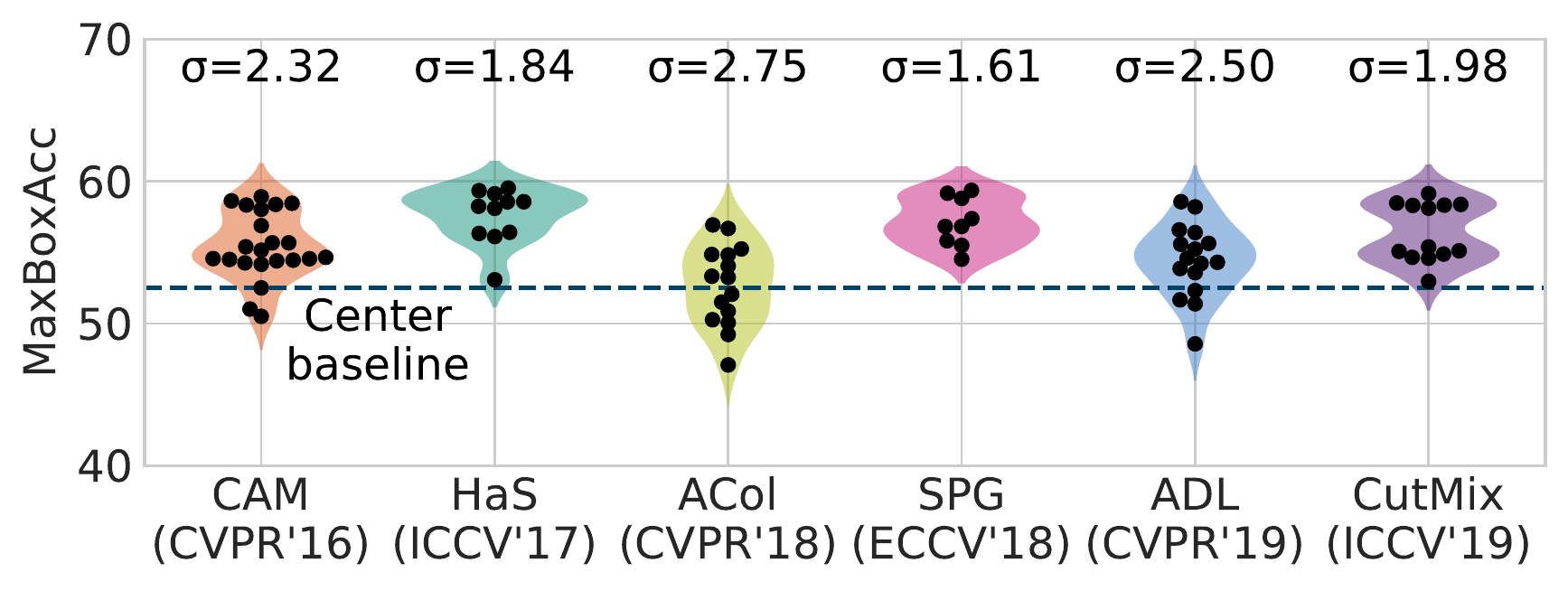}
    \caption{ImageNet, VGG}
    \end{subfigure}
    \begin{subfigure}[b]{\appendixviolinwidth\linewidth}
        \includegraphics[width=\linewidth]{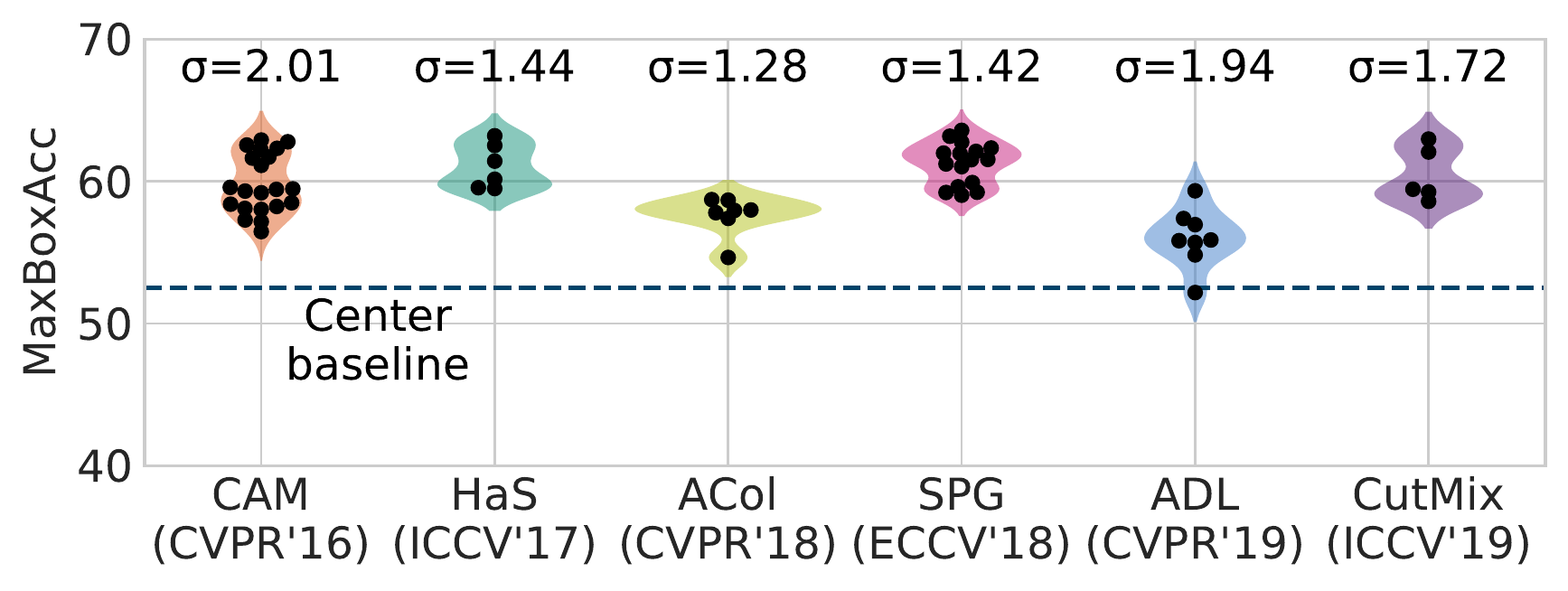}
        \caption{ImageNet, Inception}
    \end{subfigure}
    \begin{subfigure}[b]{\appendixviolinwidth\linewidth}
        \includegraphics[width=\linewidth]{figures/hyperparam_robustness_ImageNet_ResNet_reduced2.pdf}
        \caption{ImageNet, ResNet}
    \end{subfigure}
\vspace{1em}
    
    \begin{subfigure}[b]{\appendixviolinwidth\linewidth}
        \includegraphics[width=\linewidth]{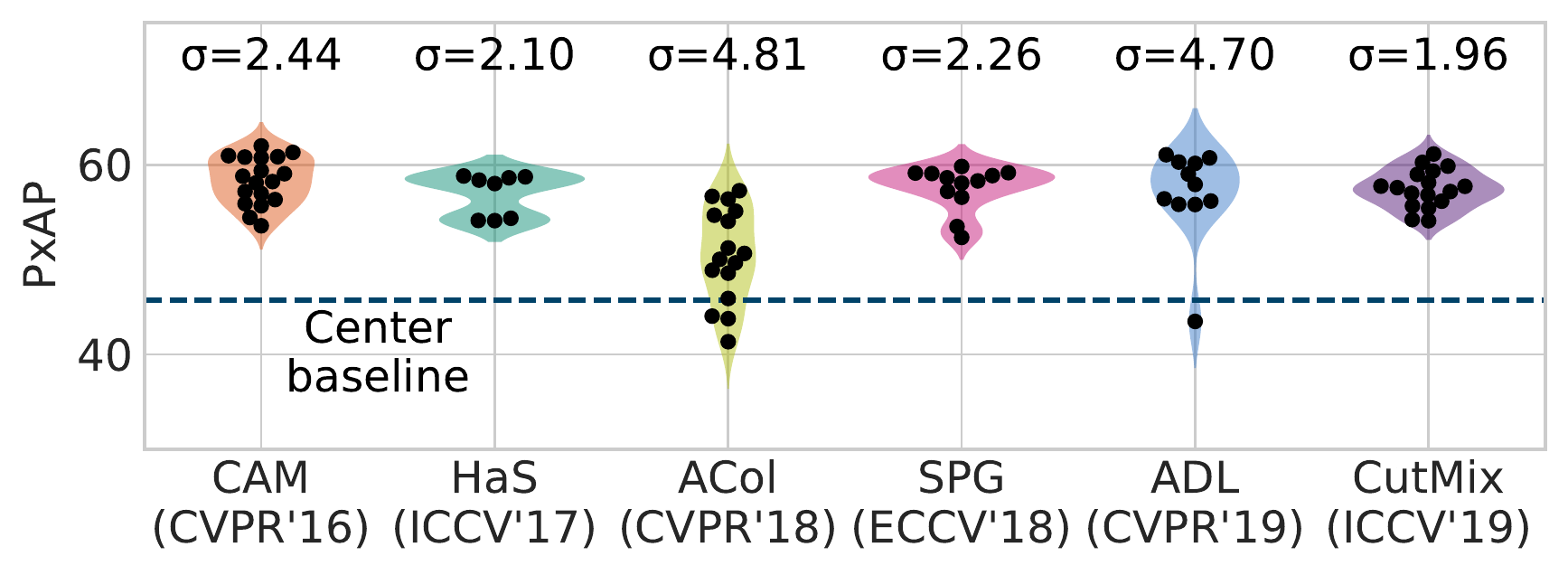}
        \caption{OpenImages, VGG}
    \end{subfigure}
    \begin{subfigure}[b]{\appendixviolinwidth\linewidth}
        \includegraphics[width=\linewidth]{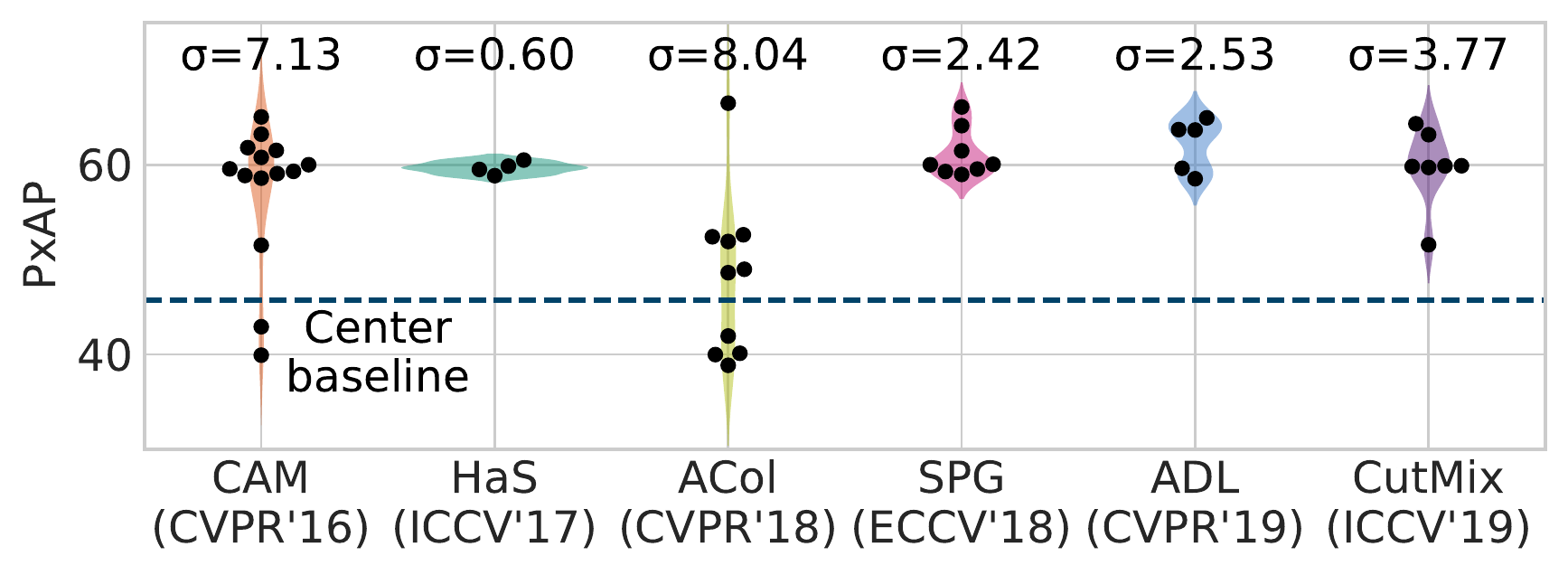}
        \caption{OpenImages, Inception}
    \end{subfigure}
    \begin{subfigure}[b]{\appendixviolinwidth\linewidth}
        \includegraphics[width=\linewidth]{figures/hyperparam_robustness_OpenImages_ResNet_reduced2.pdf}
        \caption{OpenImages, ResNet}
    \end{subfigure}
\vspace{1em}
    
    \begin{subfigure}[b]{.5\linewidth}
        \includegraphics[width=\linewidth]{figures/training_failure_ratio_gray.pdf}
        \caption{Ratio of training failures}
    \end{subfigure}
\vspace{1em}

    \caption{\small\textbf{All results of the 30 hyperparameter trials.} CUB, ImageNet, OpenImages performances of all 30 randomly chosen hyperparameter combinations for each method. This is the extension of Figure~6 in the main content.}
    \label{fig:all_3_by_3_violin_plots}
\end{figure*}
\begin{figure*}
    \centering
    \begin{subfigure}[b]{\linewidth}
        \includegraphics[width=\linewidth]{figures/appendix/legend.pdf}
    \end{subfigure}\\
    
     \begin{subfigure}[b]{.32\linewidth}
        \includegraphics[width=\linewidth]{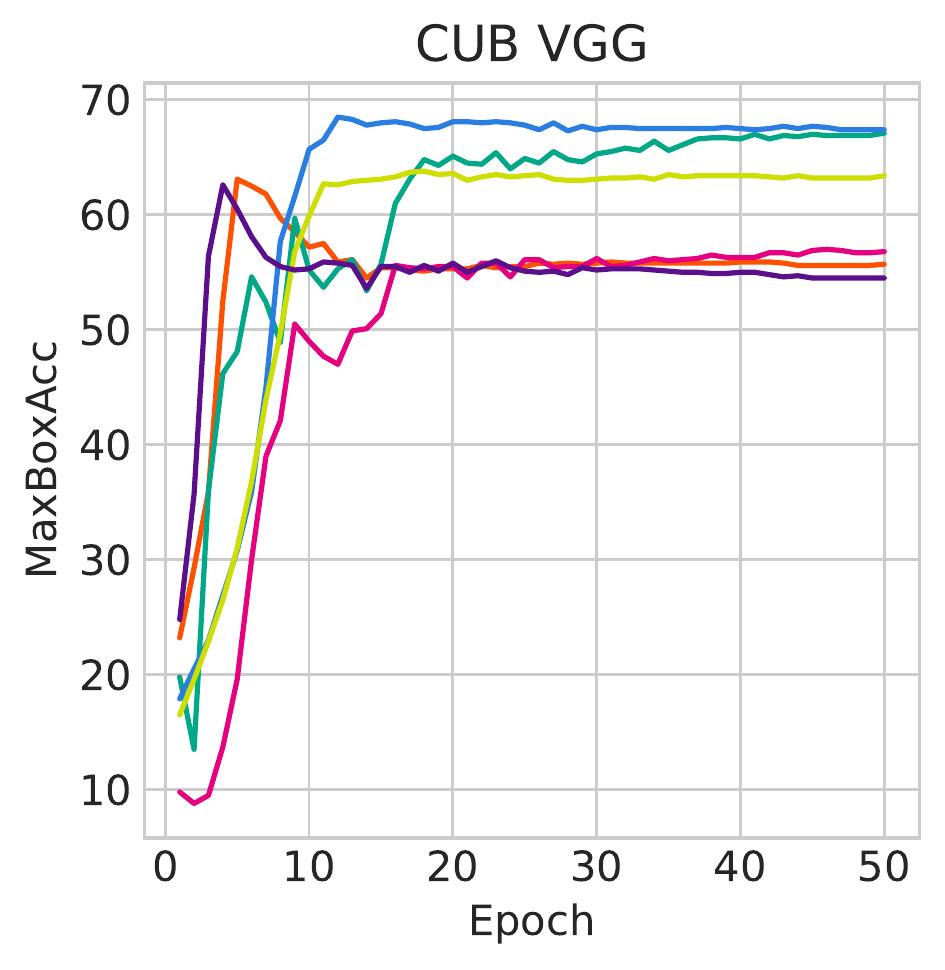}
    \end{subfigure}
    \begin{subfigure}[b]{.32\linewidth}
        \includegraphics[width=\linewidth]{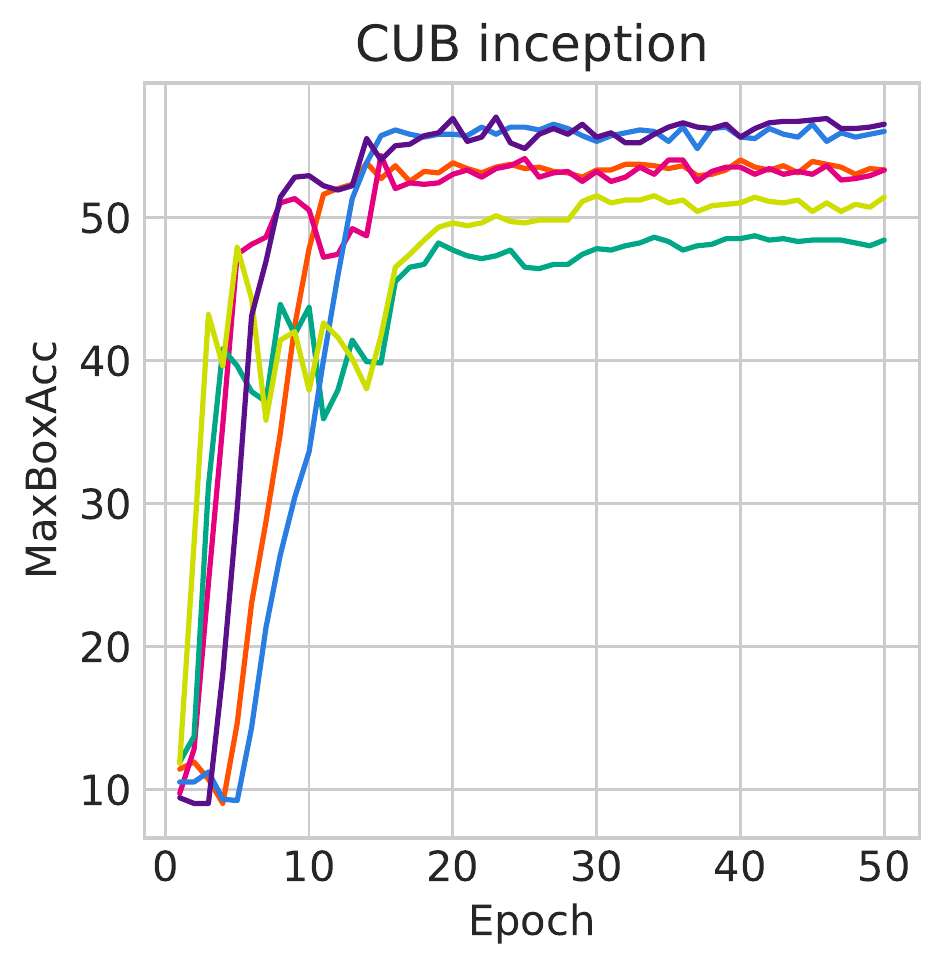}
    \end{subfigure}
    \begin{subfigure}[b]{.32\linewidth}
        \includegraphics[width=\linewidth]{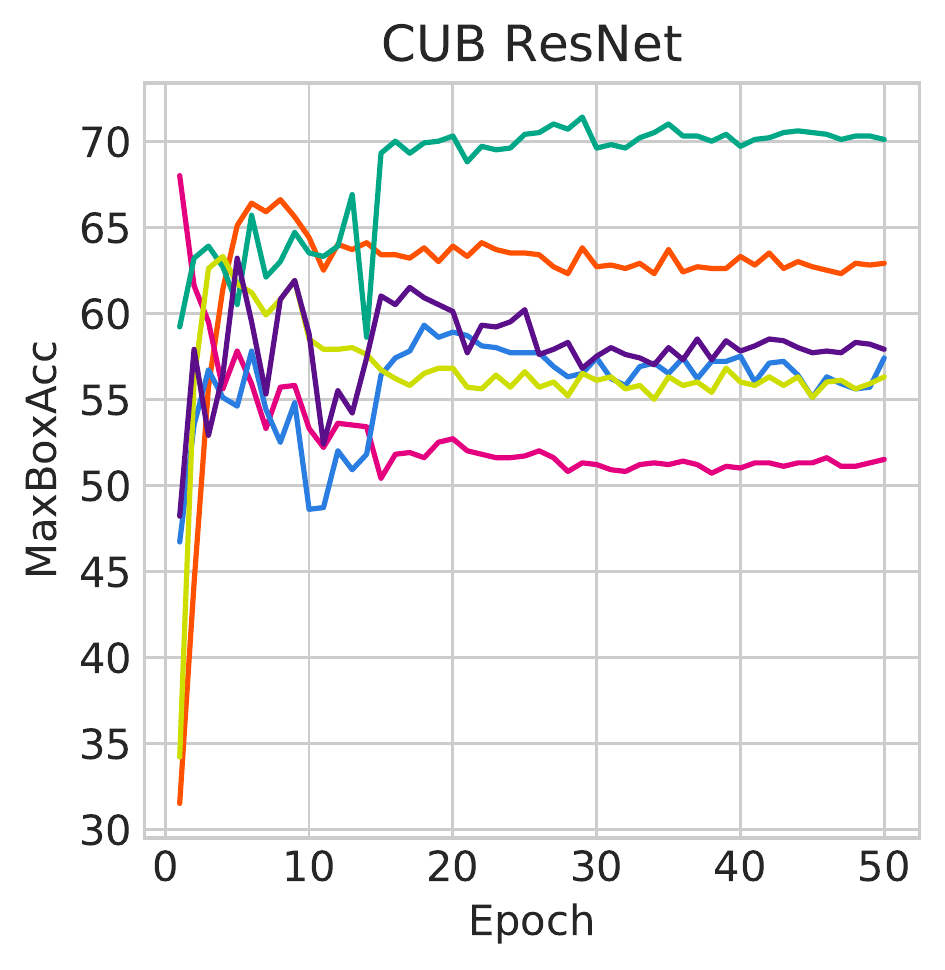}
    \end{subfigure}
    
    \begin{subfigure}[b]{.32\linewidth}
        \includegraphics[width=\linewidth]{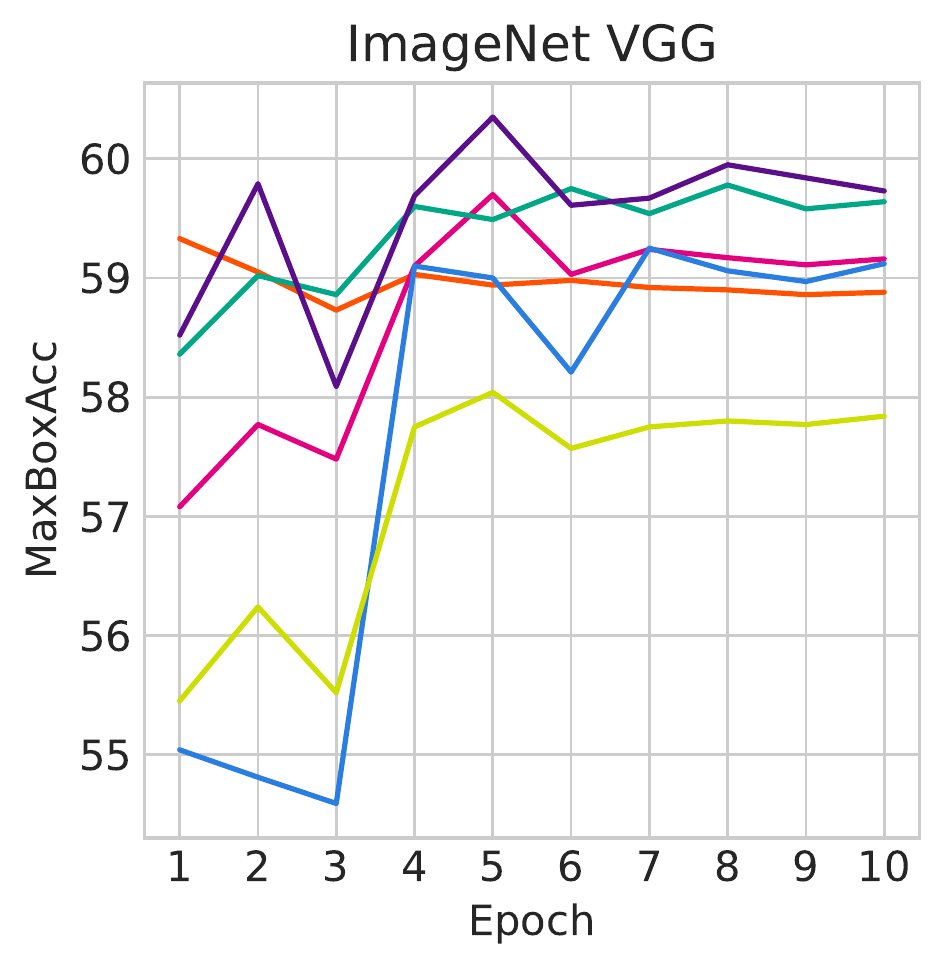}
    \end{subfigure}
    \begin{subfigure}[b]{.32\linewidth}
        \includegraphics[width=\linewidth]{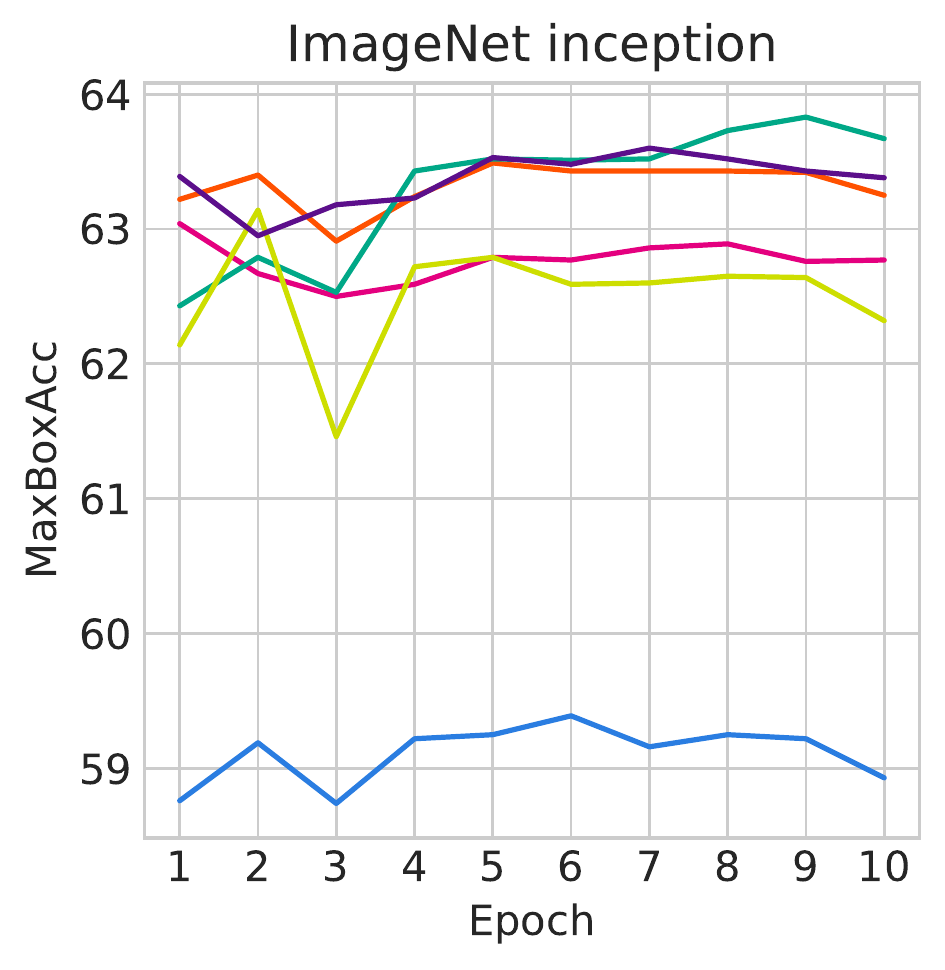}
    \end{subfigure}
    \begin{subfigure}[b]{.32\linewidth}
        \includegraphics[width=\linewidth]{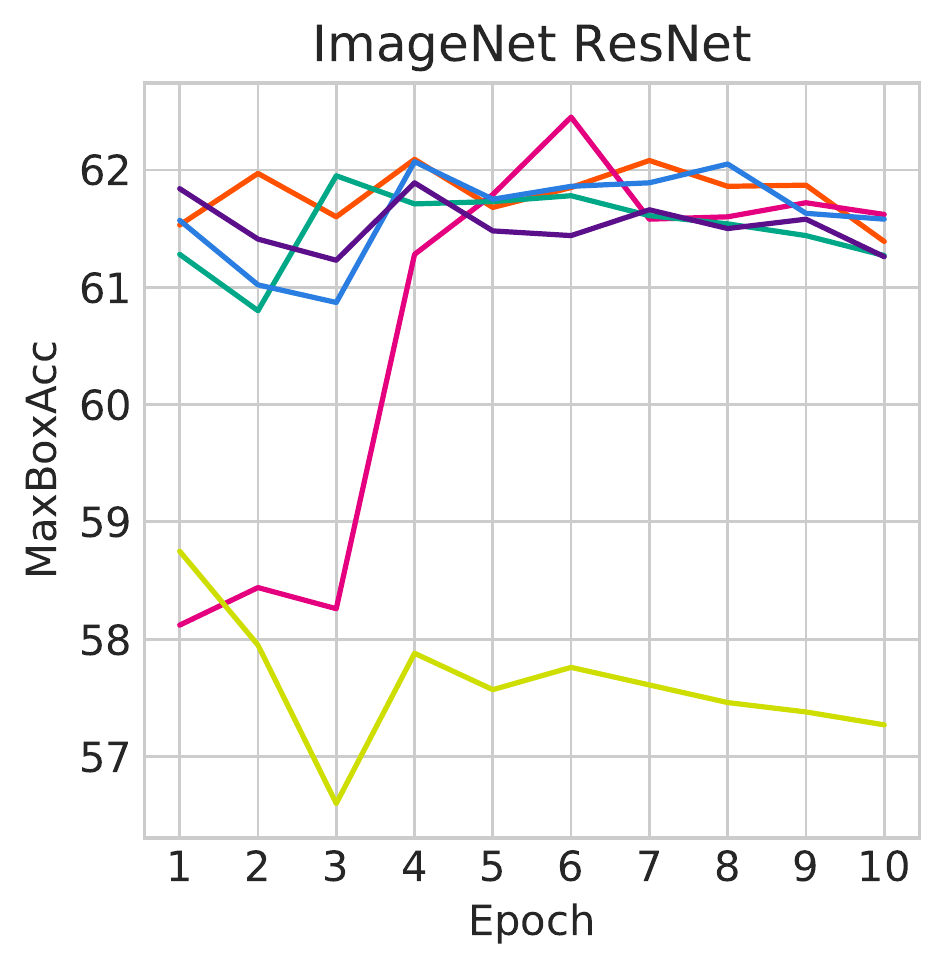}
    \end{subfigure}

    \begin{subfigure}[b]{.32\linewidth}
        \includegraphics[width=\linewidth]{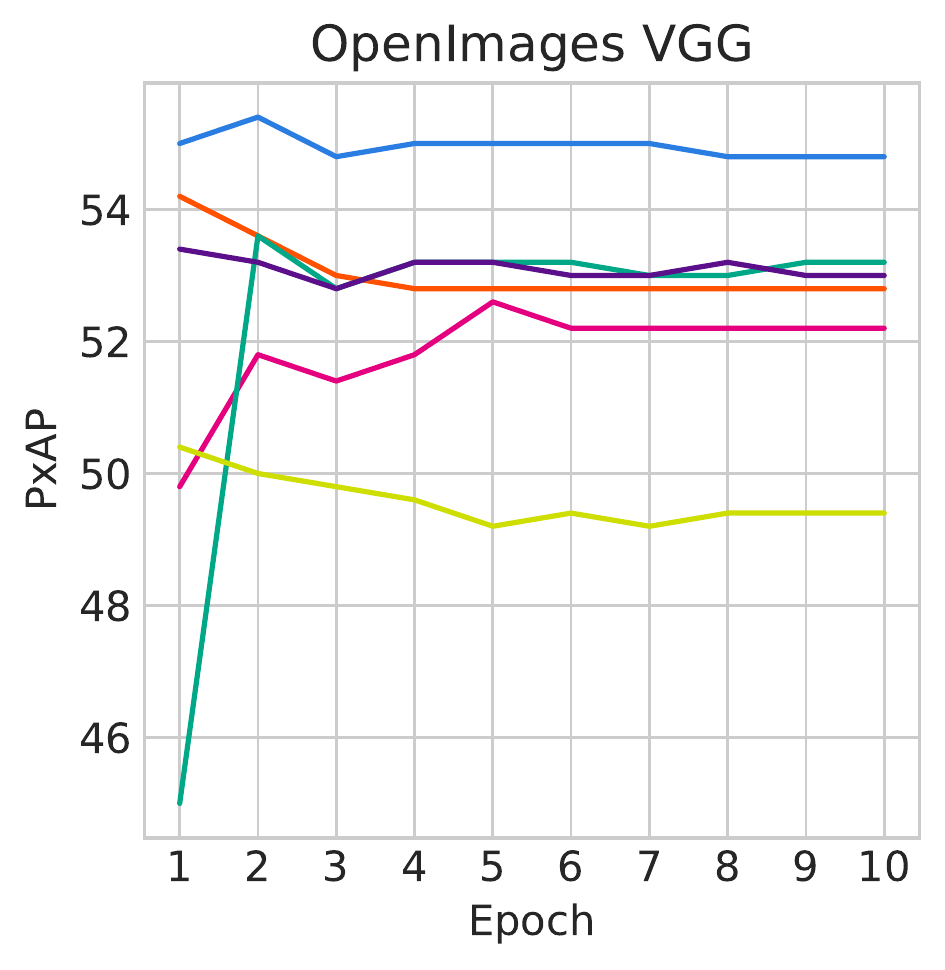}
    \end{subfigure}
    \begin{subfigure}[b]{.32\linewidth}
        \includegraphics[width=\linewidth]{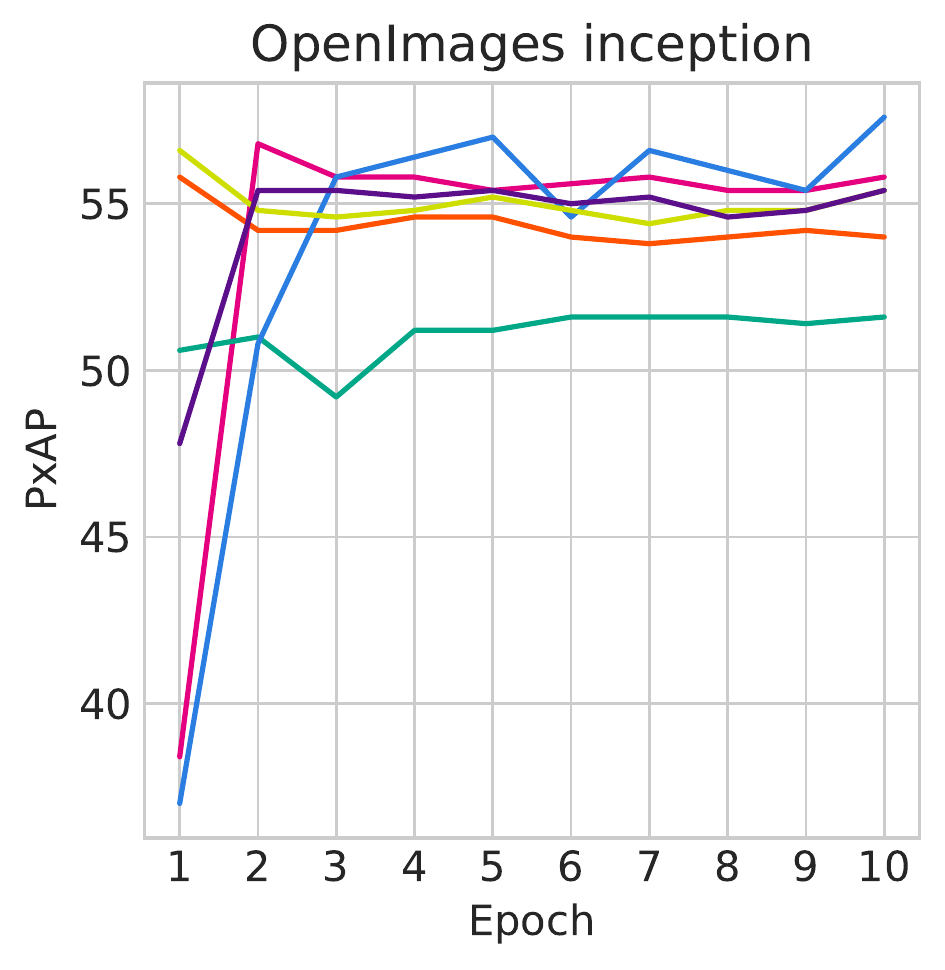}
    \end{subfigure}
    \begin{subfigure}[b]{.32\linewidth}
        \includegraphics[width=\linewidth]{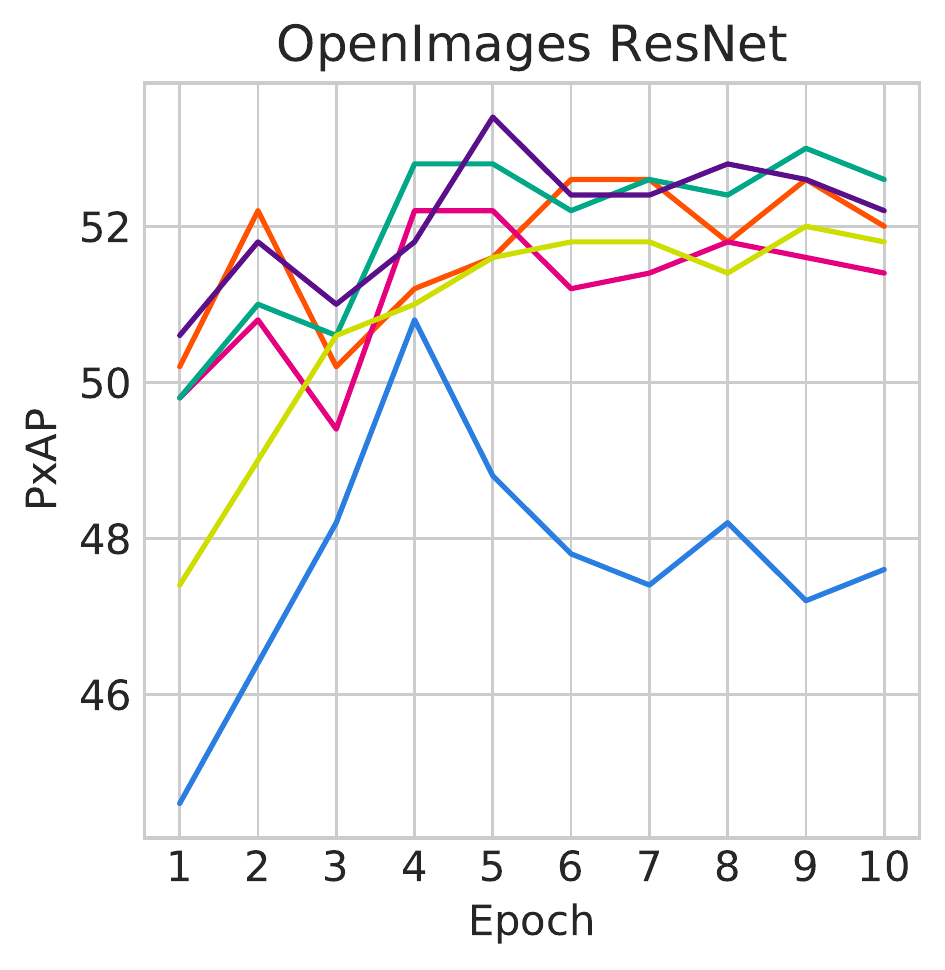}
    \end{subfigure}
    
    \caption{\small \textbf{Learning curves.} Results on three datasets with VGG, Inception and ResNet architectures.}
    \label{fig:wsol_learning_curves}
\end{figure*}

\end{document}